\documentclass{article}

\usepackage{latexsym}
\usepackage{amsmath, amsfonts, amssymb, amsthm}
\usepackage{epsfig}
\usepackage{stmaryrd}
\usepackage{multicol}
\usepackage{pdfsync}
\usepackage{dsfont}
\usepackage{algorithm,algpseudocode}
\algnewcommand\algorithmicto{\textbf{to}}
\algnewcommand\algorithmicinit{\textbf{Initialise:}}
\algrenewtext{For}[3]{\algorithmicfor\ #1 $\gets$ #2 \algorithmicto\ #3}

\usepackage{comment}
\usepackage{caption}
\usepackage{psfrag}
\usepackage{subfig}
\usepackage{xcolor}
\usepackage{graphics}
\usepackage[pdfencoding=auto,psdextra]{hyperref}
\usepackage[indent]{parskip}
\usepackage{float}

\usepackage{mathtools}

\DeclarePairedDelimiter\floor{\lfloor}{\rfloor}

\newcommand{\la}{\left \langle}
\newcommand{\ra}{\right\rangle}
\newcommand{\DS}[1]{\displaystyle{#1}}

\newcommand{\abs}[1]{\left| #1 \right|}
\newcommand{\norm}[1]{\left\lVert #1 \right\rVert}
\newcommand{\bracket}[1]{\left( #1 \right)}
\newcommand{\sqbracket}[1]{\left[ #1 \right]}
\newcommand{\set}[1]{\left\{ #1 \right\}}
\newcommand{\cA}{\mathcal{A}}

\newcommand{\sfC}{\mathsf{C}}

\newcommand{\E}{\mathbb{E}}

\newcommand{\fe}{\mathfrak{e}}

\newcommand{\cH}{\mathcal{H}}
\newcommand{\sH}{\mathsf{H}}
\newcommand{\K}{\mathcal{K}}
\renewcommand{\L}{\mathcal{L}}

\renewcommand{\O}{\mathcal{O}}

\newcommand{\cP}{\mathcal{P}}
\newcommand{\R}{\mathbb{R}}

\newcommand{\su}{\mathsf{u}}

\newcommand{\sw}{\mathsf{w}}
\newcommand{\X}{\mathcal{X}}
\newcommand{\sx}{\mathsf{x}}
\newcommand{\sy}{\mathsf{y}}

\newcommand{\sz}{\mathsf{z}}

\newcommand{\Wass}{\mathsf{Wass}}

\newtheorem{theorem}{Theorem}[section]
\newtheorem{corollary}[theorem]{Corollary}
\newtheorem{lemma}[theorem]{Lemma}
\newtheorem{proposition}[theorem]{Proposition}
\theoremstyle{definition}
\newtheorem{definition}[theorem]{Definition}

\newtheorem{assumption}[theorem]{Assumption}

\theoremstyle{remark}
\newtheorem{remark}[theorem]{Remark}
\newtheorem{example}[theorem]{Example}

\numberwithin{equation}{section}

\newcommand\numberthis{\addtocounter{equation}{1}\tag{\theequation}}

\title{Kernel Limit for a Class of Recurrent Neural Networks Trained on Ergodic Data Sequences}
\author{Samuel Chun-Hei Lam\footnote{Mathematical Institute, University of Oxford, E-mail: \url{Samuel.Lam@math.ox.ac.uk}.}\thanks{Samuel Lam's fellowship is supported by the EPSRC Centre for Doctoral Training in Mathematics of Random Systems: Analysis, Modelling and Simulation (EP/S023925/1).}, Justin Sirignano\footnote{Mathematical Institute, University of Oxford, E-mail: \url{Justin.Sirignano@maths.ox.ac.uk}.}, and Konstantinos Spiliopoulos\footnote{Department of Mathematics \& Statistics, Boston University, E-mail: \url{kspiliop@math.bu.edu}.}
\thanks{This article is part of the project ``DMS-EPSRC: Asymptotic Analysis of Online Training Algorithms in Machine Learning: Recurrent, Graphical, and Deep Neural Networks" (NSF DMS-2311500).}
\thanks{Konstantinos Spiliopoulos was also partially supported by NSF DMS-2107856 and Simons Foundation Award 672441.}}

\begin{document}
\maketitle

\begin{abstract}
Mathematical methods are developed to characterize the asymptotics of recurrent neural networks (RNN) as the number of hidden units, data samples in the sequence, hidden state updates, and training steps simultaneously grow to infinity. In the case of an RNN with a simplified weight matrix, we prove the convergence of the RNN to the solution of an infinite-dimensional ODE coupled with the fixed point of a random algebraic equation. The analysis requires addressing several challenges which are unique to RNNs. In typical mean-field applications (e.g., feedforward
neural networks), discrete updates are of magnitude $\mathcal{O}(1/N)$ and the number of updates is $\mathcal{O}(N)$. Therefore, the system can be represented as an Euler approximation of an appropriate ODE/PDE, which it will converge to as $N \rightarrow \infty$. However, the RNN hidden layer updates are $\mathcal{O}(1)$. Therefore, RNNs cannot be represented as a discretization of an ODE/PDE and standard mean-field techniques cannot be applied. Instead, we develop a fixed point analysis for the evolution of the RNN memory states, with convergence estimates in terms of the number of update steps and the number of hidden units. The RNN hidden layer is studied as a function in a Sobolev space, whose evolution is governed by the data sequence (a Markov chain), the parameter updates, and its dependence on the RNN hidden layer at the previous time step. Due to the strong correlation between updates, a Poisson equation must be used to bound the fluctuations of the RNN around its limit equation. These mathematical methods give rise to the neural tangent kernel (NTK) limits for RNNs trained on data sequences as the number of data samples and size of the neural network grow to infinity.
\end{abstract}

\tableofcontents

\section{Introduction} \label{S:ProposedResearch}

Recurrent neural networks (RNN) are widely used to model sequential data. Examples include natural language processing (NLP) and speech recognition \cite{DeepVoice, SpeechRecognition3}. The key architectural feature of an RNN is a hidden layer
which is updated at each time step of the sequence. This hidden layer -- sometimes referred to as a ``memory layer" -- is a nonlinear representation of the history of the data sequence. Using its hidden layer, the
RNN can -- in principle -- learn functions which map the path of a sequence (of arbitrary length) to fixed-dimensional vector predictions. The RNN's hidden layer therefore provides a parsimonious, nonlinear representation of the data in the sequence up until the current time. An RNN is trained by minimizing an appropriate loss function over a high-dimensional set of parameters using a gradient-descent-type algorithm.

The mathematical theory for RNNs is limited. In this article, we study the asymptotics of a single-layer RNN as the number of hidden units, training steps, and data samples in the sequence tend to infinity.  In the case of an RNN with a simplified weight matrix, we prove the convergence of the RNN to the solution of an infinite-dimensional ordinary differential equation (ODE) coupled with the fixed point of a random algebraic equation.

For feed forward neural networks (NNs) with i.i.d. data samples, limits have been proven as the number of hidden units, training steps, and data samples tend to infinity. The dynamics of the output of the trained network converges to either an ODE (the Neural Tangent Kernel NTK limit) \cite{jacot2018neural} or a  PDE (the mean-field limit) \cite{ChizatBach2018,Montanari, Rotskoff_VandenEijnden2018,NeuralNetworkLLN} depending upon the normalization used for the neural network output. For the NTK case, the equation for the limit neural network can be studied to prove global convergence of the neural network to the global minimizer of the objective function. Proving limit results for RNNs is substantially more challenging. The data sequence is not i.i.d., which complicates the analysis of the evolution of the trained neural network. Furthermore, the RNNs cannot be studied using standard mean-field or weak convergence analysis (e.g., as is true for feedforward neural networks). We explain in more detail below.

Consider a classic recurrent neural network (the standard Elman network \cite{elman}) with one hidden layer that takes in the input sequence $X = (X_k)_{k\geq 0}$ and outputs a prediction $(\hat{Y}^N_k)_{k\geq 0}$ for the target data $(Y^N_k)_{k\geq 0}$. The RNN predictions are given by the model outputs $\hat{Y}^N_k = g^N_{k}(X;\theta)$. The RNN depends on the parameters $\theta = (C,W,B)$, which must be trained on data. In particular, for all $k \geq 0$, the RNN hidden layer $S_k^N$ and predictions $\hat{Y}_k^N$ are updated as:
\begin{align}
S^{i,N}_{k+1}(X;\theta) &:= \sigma \bracket{(W^i)^\top X_k + \frac{1}{N^{\beta_1}} \sum_{j=1}^N B^{ij} S^{j,N}_{k}(X;\theta)}, \quad S^{i,N}_{0}(X;\theta) = 0 \label{eq:StandardRNNmemory}\\
g^N_{k}(X;\theta) &:= \frac{1}{N^{\beta_2}} \sum_{i=1}^N C^i S^{i,N}_{k+1}(X;\theta),
\label{eq:StandardRNN}
\end{align}
where
\begin{itemize}
    \item $N$ is the number of hidden units in the memory states $S_k^{i,N}(X;\theta)$,
    \item $\beta_1 = 1$ and $\beta_2$ determine the scaling used to normalise the outputs of the network,
    \item $C \in \R^N$, with $C^i$ representing the $i$-th component of $C$,
    \item $W \in \R^{N\times d}$, with $W^i$ representing the $i$-th column of $W$ as a column vector, and
    \item $B \in \R^{N\times N}$, with $B^{ij}$ is the $(i,j)$-entries of $B$.
\end{itemize}

 The data samples $X_k$, which are elements of a data sequence, are \emph{not} i.i.d. (unlike for feedforward neural networks). In our mathematical analysis, we will make the simplifying assumption that all columns of $B$ are equal, i.e. for all $j$ we have $B^{ij}=B^j$ for some $B^j$. Generalizing this assumption and developing a mean-field limit for the full-weight matrix $B^{ij}$, without any simplifying assumptions, is an interesting topic for future research. Similarly, deriving a mean-field limit for $\frac{1}{2} \leq \beta_1 < 1 $ would be interesting since, for $\beta_1 = 1$, the $W^i$ parameters primarily drive the limit dynamics while it is expected for $\frac{1}{2} \leq \beta_1 < 1 $ that the $B^{ij}$ will also become important drivers of the limit dynamics.

The memory state $S_k^{N,i}(X;\theta)$ is a non-linear representation of the history of the data sequence $(X_j)_{j=0}^{k-1}$. Using this non-linear representation -- which is learned from the data by training the parameters $\theta$ -- the RNN generates a prediction $\hat Y_k^N$ for the target data $Y_k$. Notice that if we fix $X_k = x$ and $B^{ij} = 0$, \eqref{eq:StandardRNN} becomes a standard feedforward network (i.e., the network does not dynamically evolve over time $k$ and the network output is a static prediction). Limits for gradient-descent-trained feedforward networks as the number of hidden units $N \rightarrow \infty$ can been established when $ \frac{1}{2} \leq \beta_2 < 1$ (the NTK limit \cite{jacot2018neural} and \cite{NENN,dNENN}) or $\beta_2 = 1$ (the mean-field limit \cite{ChizatBach2018, Montanari, Rotskoff_VandenEijnden2018, NeuralNetworkLLN}).

A ``typical" limit ODE from mean-field analysis will not occur in for the RNNs and standard mean-field techniques (see, for example, \cite{DeepNN}) cannot be directly applied. As an illustrative example, standard mean-field techniques would be applicable to a residual-type neural network with the following updates:
\begin{eqnarray}
S^{i,N}_{k+1}(X;\theta) &:=& {\color{blue} S^{i,N}_{k}(X;\theta)} + {\color{red} \frac{1}{N}} \sigma \bracket{(W^i)^\top X_k + \frac{1}{N^{\beta_1}} \sum_{j=1}^N B^{ij} S^{j,N}_{k}(X;\theta)}, \label{eq:StandardMeanFieldProblem} \\
S^{i,N}_{0}(X;\theta) &=& 0. \nonumber
\end{eqnarray}
\eqref{eq:StandardMeanFieldProblem} is an Euler approximation of an ODE with step size $\mathcal{O}(1/N)$. \eqref{eq:StandardMeanFieldProblem} is a standard mean-field framework and it can be shown as $N \rightarrow \infty$ that \eqref{eq:StandardMeanFieldProblem} converges to an appropriate infinite-dimensional ODE. However, the $\textcolor{blue}{S^{i,N}_{k}(X;\theta)}$  and $\textcolor{red}{1/N}$ do not appear in the RNN \eqref{eq:StandardRNN}. This changes the analysis: the RNN hidden layer \eqref{eq:StandardRNN} is not an Euler approximation to an ODE.

Although \eqref{eq:StandardRNN} is not a standard mean-field equation, we can observe mean-field behaviour for the distribution of the hidden layer. This is illustrated by Figure \ref{fig:3} where, for varying $N$, we simulated paths of the hidden layer $(S^{i,N}_k(X;\theta))_{i=1}^N$, based upon a common $\theta$ and independent paths of the input sequence $X$. The empirical distributions of the hidden units in the memory layer at a large, fixed time step $k$, is displayed as $N \to +\infty$. The empirical distributions converge as $N \rightarrow \infty$. Figure \ref{fig:3} suggests that a mean-field limit as $N \rightarrow \infty$ does exist. Further details of the simulation are provided in Appendix \ref{S:numerics}.

\begin{figure}[t]
    \centering
    \includegraphics[width=0.75\textwidth]{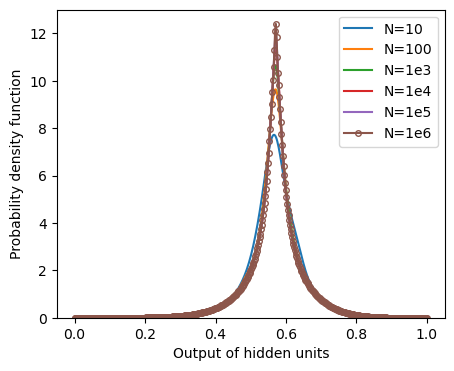}
    \caption{Each curve represents the \textit{overall} empirical distributions of the untrained hidden units in the memory states (the \textit{hidden memory units}) from all simulation instances $\ell = 1, \ldots, 100$ for $N = 10^2, ... 10^6$ and time step $k \approx 50000$.}
    \label{fig:3}
\end{figure}

Furthermore, numerical simulations suggest that the hidden layer is ergodic as the time steps $k \rightarrow \infty$. Figure \ref{fig:1_truncated} displays the time-averaged first and second moments of the hidden layer. The formal definition of the time averages is provided in Appendix \ref{S:numerics}. Figures \ref{fig:3} and \ref{fig:1_truncated} together motivate an analysis of the RNN \eqref{eq:StandardRNN} as both $k, N \rightarrow \infty$.

\begin{figure}[t]
\centering
\includegraphics[width=\textwidth]{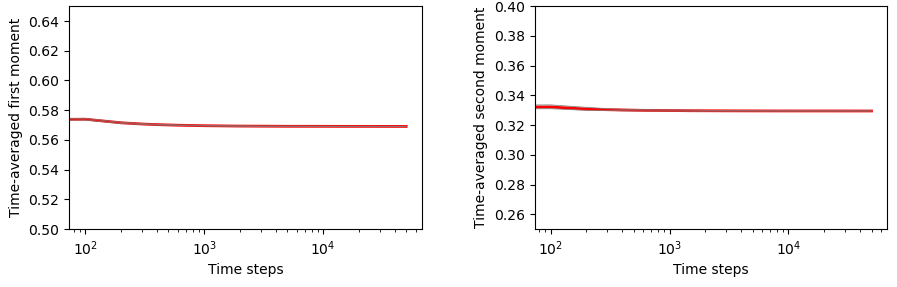}
\caption{The plots of the time-averaged first and second moments of the hidden units for a sufficiently large $N$ (chosen to be $10^6$) and $p = 1,2$. The $x$-axis represents the number of time steps. We summarise the minimum/maximum of the simulated first and second moments of the time-averages for independent input sequences $X$ using a (seemingly invisible) grey band. The red line represents the mean of the time-averaged moments for all input sequences $X$, thus providing a Monte-Carlo estimate for the moments of the \textit{random} fixed point. The fact that the realisations of the time averages all converge as $k \to \infty$ illustrates the ergodicity of the sequence $S^{i,N}_k(X;\theta)$.}
\label{fig:1_truncated}
\end{figure}

Our mean-field analysis studies an appropriate fixed point for the untrained hidden units in the memory states, and use it to study the evolution of the RNN. The case for a constant input sequence $X_k = x$ (with a finite number of hidden units) has been studied in \cite{recurrent_BP_1, recurrent_BP_2, recurrent_BP_3} for developing a more efficient gradient-descent-type algorithm. However, the fixed point analysis is more complex when the input sequence $X_k$ is non-constant and random. As a result, we would expect the RNN to have a \textit{random} fixed point and that the \textit{distribution} of the untrained hidden unit in the memory states should converge \emph{weakly}. The random fixed point should also depend on the distribution of the data sequence $X_k$.

We are able to prove that such a random fixed point exists if the data sequence $X_k$ is ergodic. Subsection \ref{SS:dynamics_of_sample_memories} shows that the ergodicity of $X$ actually leads to the ergodicity of $S^{i,N}_k(X;\theta)$.

The fixed point analysis is further complicated since the parameters $\theta$ will be simultaneously trained using the truncated backpropagation through time (tBPTT) algorithm as the hidden layer is updated at each time $k$; see Section \ref{S:ModelSetup}. Both the hidden layers $S_k^N(X;\theta_k)$ and the parameters $\theta_k$ (which govern the transition function for the hidden layer $S_k^N$) will jointly evolve in time. Therefore, the dynamics of the RNN will be changing over time. Fortunately, the changes due to the parameter updates will be of magnitude $\mathcal{O}(1/N)$. Therefore, the evolution of the output layer of the network ($g^N_{k}(X;\theta) = \frac{1}{N^{\beta_2}} \sum_{i=1}^N C^i S^{i,N}_{k+1}(X;\theta)$) can be represented as an Euler approximation of an appropriate infinite-dimensional ODE whose dynamics are a function of the RNN hidden layer's random fixed point, which it will converge to as $N \rightarrow \infty$. We emphasise that the evolution of the RNN hidden layer itself (i.e., $S^{i,N}_{k+1}(X;\theta)$) cannot be represented as a discretization of an ODE/PDE since the RNN hidden layer updates are $\mathcal{O}(1)$.

The algebraic equation which the fixed point satisfies specifically arises due to the memory and feedback connection in the recurrent structure of the RNN. Unlike in a feedforward network (which simply evaluates a function on i.i.d. data samples), the RNN will use the state of the hidden layer $S_k^N$ at the last time step $k$ (and the new data $X_k$ which arrives at time $k$) to update the state of the hidden layer to its new value $S_{k+1}^N$ at time $k+1$. Thus, there is a random map from $S_k, X_k \rightarrow S_{k+1}$ (where the randomness is due to the data $X_k$). We are able to prove that (the distribution of) $S_k$ converges to the solution of an appropriate fixed point equation (an algebraic equation) as $k \rightarrow \infty$. Since the map from $S_k, X_k \rightarrow S_{k+1}$ depends also on the data $X_k$, the ergodicity assumption for the underlying data sequence $X_k$ is required to study this map as time $k \rightarrow \infty$ and for a fixed point to exist. These mathematical features do not appear in the analysis of a standard fully-connected feedforward neural network architecture. In a standard feedforward neural network architecture, data samples are i.i.d. (i.e., they are not from a data sequence) and the hidden layer of the neural network for data sample $k+1$ is not affected by the hidden layer for data sample $k$ (the neural network evaluations and their hidden layers are completely independent for different data samples). Therefore, RNNs introduce non-trivial challenges due to the strong dependency of the hidden layer across the sequence of data samples at times $1, 2, \ldots, k, k+1, \ldots$ (i.e., the ``memory" of the RNN).

We present our main result on the limit dynamics of the RNN output in Subsection \ref{SS:ModelDynamics}. In Subsection \ref{SS:LimitProblemAnalysis} we prove that the limiting average loss is monotonically decreasing during training.

The RNN hidden layer is studied as a function in the appropriate space, whose evolution is governed by the data sequence $X_k$, the parameter updates, and its dependence on the RNN hidden layer at the previous time step. Due to the correlation between updates (including both the hidden layer and data samples), a Poisson equation must be used to bound the fluctuations of the RNN around its limit equation.

The rest of the paper is organized as follows. In Section \ref{S:ModelSetup} we present the model that we study and our main assumptions. The main results are discussed in Section \ref{S:main_results}, which involves convergence in distribution of the RNN memory layer and RNN network output to an ODE. The proof is spanned across two sections: Section \ref{S:proof_of_dynamics_of_memories} studies the RNN memory/hidden layer, and \ref{S:weak_convergence} studies the network outputs.

The paper contains a few appendices. Appendix \ref{S:recursive_inequality} presents a recursive inequality that is used in various places throughout the paper. Appendix \ref{S:construction_of_clipping_function} provides details on how to construct a clipping function to be used in the training of RNN. Appendix \ref{S:increments_of_parameters} establishes a-priori bounds of increments of parameters during training, which would be extensively used in the main proof. Finally, Appendix \ref{S:numerics} provides details on the numerical simulation.

\section{Assumptions, Data, and Model Architecture}
\label{S:ModelSetup}

\subsection{Data generation}
Our paper focuses on the problem of recovering the map from an input data sequence to an output data sequence. We assume the input data sequence $X = (X_k)_{k\geq 0}, X_k \in \R^d$, depends on the hidden state $Z = (Z_k)_{k\geq 0}, \, Z_k \in \R^{d_z}$, such that the joint process $(X_k, Z_k)$ is a time-homogeneous Markov chain with transition kernel $P$ that satisfies
\begin{equation}
\mathbb{P}((X_{k+1}, Z_{k+1}) \in \cdot \mid (X_k, Z_k) = (x,z)) = P((x,z), \cdot) \label{eq:dynamics_eq1}
\end{equation}
The output data sequence $Y = (Y_k)_{k\geq 0}, Y_k \in \R^{d_y}$ shall be governed by the following equation:
\begin{equation}
Y_k = f(X_k, Z_k, \eta_k), \label{eq:dynamics_eq2}
\end{equation}
where $(\eta_k)_{k\geq 0}, \eta \in \R^{d_\eta}$ are independent, identically distributed (iid) noises independent with $(X_k, Z_k)_{k\geq 0}$, and $f:\R^{d+d_z+d_\eta} \to \R^{d_y}$ is a map. We emphasise that only the inputs $(X_k)_{k\geq 0}$ and the outputs $(Y_k)_{k\geq 0}$ are observed, and it is not guaranteed whether the observation $(X_k, Y_k)_{k\geq 0}$ is actually a Markov process.

We assume that the joint process $(X_k, Z_k)_{k\geq 0}$ is nicely behaved. To specify this we need the notion of a (2-)Wasserstein distance between probability distributions. Recall that for $p\geq 1$, the $p$-Wasserstein space could be defined on a complete (Banach) and separable normed space $(\X,\norm{\cdot})$: is Banach (i.e. complete) and separable (see, for example, \cite{Villanioldandnew}):
\begin{definition}[Wasserstein Metric]
Let $\rho_1, \rho_2$ are measures in the $p$-Wasserstein space of $(\mathcal{X}, \norm{\cdot})$:
\begin{equation}
    \rho_1, \rho_2 \in \mathcal{P}^{\norm{\cdot}}_p(\mathcal{X}) = \set{\rho \,:\, \int_\mathcal{X} \norm{u}^p_\mathcal{X} \, \rho(du) < +\infty}\nonumber
\end{equation}
A measure $\rho$ on $\mathcal{X} \times \mathcal{X}$ is a \textit{coupling} between $\rho_1$ and $\rho_2$ if
\begin{equation}
    \rho_1(\cdot) = \rho(\cdot \times \X) \quad \rho_2(\cdot) = \rho(\X \times \cdot).\nonumber
\end{equation}
The $p$-Wasserstein distance between $\rho_1$ and $\rho_2$ with respect to norm $d$ is defined as $\Wass^{\norm{\cdot}}_p(\rho_1, \rho_2)$ that satisfies
\begin{equation}
    \sqbracket{\Wass^{\|\cdot\|}_p(\rho_1, \rho_2)}^p = \inf_{\gamma \in \Gamma(\rho_1,\rho_2)} \int_{\X \times \X} \norm{u-v}^p_\X \, \gamma(du \, dv), \label{eq:Wasserstein_def}
\end{equation}
where $\Gamma(\rho_1, \rho_2)$ is the set of all couplings between $\rho_1$ and $\rho_2$.
\end{definition}

The norm $\norm{\cdot}$ might be omitted in writing $\Wass^{\norm{\cdot}}_p(\rho_1, \rho_2)$ and $\cP^{\norm{\cdot}}_p(\X)$ if no confusion is created as a result.

\begin{remark} \label{rmk:optimal_coupling}
Note that the infinmum in \eqref{eq:Wasserstein_def} is a minimum \cite{Villanioldandnew}, i.e. there is an optimal coupling $\gamma^* \in \Gamma(\rho_1, \rho_2)$ such that
\begin{equation}
    \sqbracket{\Wass^{\|\cdot \|}_p(\rho_1, \rho_2)}^p = \int_{\X \times \X} \norm{u-v}_\X^p \, \gamma^*(du \, dv). \label{eq:Wasserstein_optimal}
\end{equation}
This fact will be used to prove bounds for the Wasserstein metric between successive terms of sequence of distribution $\mu_k$.
\end{remark}

With the notion of Wasserstein distance, we make the following assumptions for the transition kernel $P$ and the function $f$.

\begin{assumption}[On governing dynamics and background noises of the data sequences] \label{as:data_generation} \phantom{Blah\\}
\begin{enumerate}
\item The transition kernel $P$ is contractive in the following sense: there exists $L < 1$ such that for any $(x,z), (\tilde{x}, \tilde{z}) \in \R^{d+d_z}$,
\begin{equation}
    \frac{\Wass_2(P((x,z), \cdot), P((\tilde{x}, \tilde{z})), \cdot)}{|(x,z) - (\tilde{x},\tilde{z})|} \leq L.
\end{equation}
\item The sequence $|X_k, Z_k| \leq C_x$, where $C_x <\infty$.
\item The noises $\eta_k$ satisfies $\E|\eta_k|^2 = C_\eta < +\infty.$
\item The function $f$ is bounded and Lipschitz in all arguments.
\end{enumerate}
\end{assumption}

Assumptions \ref{as:data_generation}(1-2) are satisfied by the following examples.
\begin{example}
Let $g:\R^{d+d_z} \to \R^{d+d_z}$ be a bounded, $L$-Lipschitz function with $L < 1$, and $\epsilon_k \in \R^{d+d_z}$ (for $k\geq 1$) be bounded, independent, identically distributed (i.i.d.). Then, the sequence $(X_k, Z_k)$ defined recursively as
\begin{align*}
(X_{k+1}, Z_{k+1}) = g(X_k, Z_k) + \epsilon_{k+1}, \quad k\geq 0,
\end{align*}
with $(X_0, Z_0)$ initialised independently of the sequence $(\epsilon_{k})_{k\geq 1}$, satisfies Assumptions \ref{as:data_generation}(1-2).
\end{example}

Here is a related example that involves an i.i.d. sequence $(\epsilon_k)$ but does not require its boundedness.
\begin{example}
Let $g:\R^{d+d_z+1} \to \R^{d+d_z}$ be a bounded function such that $g(\cdot, \cdot, \epsilon)$ is a $L$-Lipschitz function with $L < 1$ for any $\epsilon$. Assume that $\epsilon_k \in \R$ are i.i.d. Then the sequence $(X_k, Z_k)$, defined recursively as
\begin{equation*}
(X_{k+1}, Z_{k+1}) = g(X_k, Z_k, \epsilon_{k+1}), \quad k \geq 0,
\end{equation*}
with $(X_0, Z_0)$ initialised independently of the sequence $(\epsilon_{k})_{k\geq 1}$, satisfies Assumptions \ref{as:data_generation}(1-2).
\end{example}

As we shall see in Theorem \ref{prop:conv_of_mu_k} and its proof in Subsection \ref{SS:dynamics_of_memory_units}, these assumptions are sufficient to ensure that the sequence of RNN memory layers is ergodic.

\subsection{Recurrent Neural Network}
A recurrent neural network is trained with the tBPTT algorithm to approximate the function $g$. Specifically, we will study the standard recurrent network in equation (\ref{eq:StandardRNN}) but with the following simplifying assumption for its weight matrix:

\begin{assumption}[Simplifying assumption for memory layer weight matrix]\label{as:memoryWeights}
We assume that all columns of $B$ are equal, that is, for all $j$ we have $B^{ij}=B^j$ for some $B^j$.
\end{assumption}

\begin{assumption}[Regularity of activation function] \label{as:activation_function}
We assume that $\sigma \in C^2_{b}(\R)$ (i.e., twice-continuously differentiable with bounded derivatives) and define:
\begin{equation*}
C_\sigma = \sup_z |\sigma(z)|, \quad C_{\sigma'} = \sup_z |\sigma'(z)|, \quad C_{\sigma''} = \sup_z |\sigma''(z)|.
\end{equation*}
\end{assumption}

\begin{assumption}[Initialization of parameters] \label{as:ergodicity_conditions}
We assume that
\begin{itemize}
    \item the initial parameters $\theta_0 = (C^i_0, W^i_0, B^i_0)$ are independent from the \\ sequence $(X_k, Z_k, \eta_k)_{k\geq 0}$,
    \item the initial parameters $(C^i_0, W^i_0, B^i_0)_{i=1}^{N}$ are independent and identically distributed, such that $(C^i_0, W^i_0, B^i_0) \sim \lambda_0 := \lambda(dc,dw,db)$, with $\lambda(dc,dw,db)$ is absolutely continuous (i.e., admitting a density) with respect to Lebesgue measure,
    \item for each $i$, the random variables $C^i_0$, $W^i_0$ and $B^i_0$ are mutually independent,
    \item there exist constants $\mu_{c^2}, \mu_b, \mu_{b^2}, \mu_{w^2}, K_c, K_b \in (0,+\infty)$ such that
    \begin{align*}
        |C^i_0| &\leq K_c, \quad \E[C^i_0] = 0, \quad \E[C^i_0]^2 = \mu_{c^2} \\
        |B^i_0| &\leq K_b, \quad \E[B^i_0] = \mu_b, \quad \E[B^i_0]^2 = \mu_{b^2}, \quad \E|W^i_0|^2 = \mu_{w^2}, \\
        \mu_b^2 &\leq \mu_{b^2} \leq K_b^2, \quad \mu_{c^2} \leq K^2_c, \quad 2 K_b^2 C_{\sigma'}^{2}<1,
    \end{align*}
    and
    \begin{align}
        q_0 := \sqrt{\max\{L^{2}+2\mu_{b}^{2}\mu_{w^{2}}C_{\sigma'}^{2},2\mu_{b}^{2}C_{\sigma'}^{2}\}} < 1.
    \end{align}
\end{itemize}
\end{assumption}

Assumption \ref{as:ergodicity_conditions} on the relation of the upper bounds and constants $L,\mu_{b}, \mu_{w^{2}}C_{\sigma'}, K_{b}$ is instrumental in our proof of the ergodicity of the memory states in Subsection \ref{SS:dynamics_of_memory_units}. It is used to obtain a contraction mapping, which in turns guarantees the existence of a limit. At an intuitive level, the relation between the different upper bounds of the involved parameters shows that the memory effect should not be too strong. Having said that, if one can establish the existence of a limit in a different way, then Assumption \ref{as:ergodicity_conditions} may not be needed.  In fact, we conjecture from our simulation results that the RNN exhibits the desired mean-field behaviour under possibly weaker assumptions. Despite our best efforts, we did not manage to find a viable alternative theoretical path.

At this point we mention that a careful investigation of the bounds demonstrates that the condition $|C^i_0| \leq K_c$ can be replaced by the moment boundedness condition $\E|C^i_0|^{8} \leq K_c$. We chose to present the arguments with the more restrictive  condition $|C^i_0| \leq K_c$ for the purposes of a more concise presentation and exposition.

\begin{example}
An example of an activation function satisfying our assumptions is the standard sigmoid function, defined as
\begin{equation}
    z \in \R \mapsto \sigma(z) = \frac{1}{1+e^{-z}} \implies \quad\sup_{z\in\R}|\sigma(z)|\leq 1.\nonumber
\end{equation}

One could show (see e.g. \cite{Goodfellow}) that $\sigma'(z) \in [0,1/4]$ and $|\sigma''(z)| \leq 1/4$ for any $z \in \R$. As a result, the sigmoid function satisfies our assumption when $L<1/2$ and $K_b,\mu_{b},\mu_{w^{2}}<1$. Alternatively, if the activation function is such that $C_{\sigma'}$ is larger, then Assumption \ref{as:ergodicity_conditions} suggests that the variables $B^{i}_{0}$ and $W^{i}_{0}$ should be initialized with  possibly smaller upper bounds so that Assumption \ref{as:ergodicity_conditions} holds.
\end{example}

In our notation later in the paper, the dependence of the RNN hidden layer $S_{k}^{i,N}(X; \theta)$ on $X$ may be dropped in the later sections.

Assumptions \ref{as:data_generation}, \ref{as:memoryWeights}, \ref{as:activation_function} and \ref{as:ergodicity_conditions} are assumed to hold throughout the paper.
\subsection{Training the RNN parameters}
The parameters $\theta$ are in practice trained with an online SGD algorithm seeking to minimize the objective function
\begin{eqnarray}
\L(\theta) = \frac{1}{T} \sum_{k=1}^T \L_k(\theta), \quad \text{where }\L_k(\theta) = \frac{1}{2} (g^N_k(X;\theta) - Y_k)^2.\nonumber
\label{eq:TrueObj}
\end{eqnarray}

The computational cost of evaluating the RNN network up to time step $T$ and performing one full evaluation of the gradient by back-propagation through time (BPTT) grows as $\O(T)$. This means that a single gradient descent iteration to train the RNN has a computational cost of $\O(T)$, which becomes computationally intractable for large $T$. We instead update the parameters through the \textit{online stochastic gradient descent with truncated back-propagation through time} (online SGD with tBPTT) \cite{WilliamsPeng1990} \cite{Sutskever2012}. A detailed explanation is given in the next section, but the main idea is to truncate the  computational graph up to $\tau$ time steps for $\tau \ll T$ when computing the outputs of the RNN and gradients with respect to the parameters. For simplicity and without loss of generality, we restrict our discussion to the case when $\tau = 1$.

\paragraph{Assumptions for the wide-network limit}
In our paper we set $\beta_1 = 1$. Depending on the choice of $\beta_2$, we will get different limits. Prior analysis for feedforward NNs (including our prior work),  \cite{NeuralNetworkLLN, RLasymptotics, PDEclosure, NENN, dNENN} demonstrate that when $N\to+\infty$, the evolution of the output of the feedforward NN converges to a limit equation:
\begin{enumerate}
    \item for $\beta_2 = 1$, the limit equation is expected to be a PDE,
    \item for $\beta_2 = \beta \in (1/2,1)$, the limit equation is expected to be an infinite-dimensional ODE, and
    \item for $\beta_2 = 1/2$, the limit equation is expected to be a \emph{random} infinite-dimensional ODE.
\end{enumerate}
Our paper will focus on the case when $\beta_2 = \beta \in (1/2,1)$. Some analyses have been done for the case when $\beta_2 = 1$ in \cite{mean_field_agazzi}, with the RNN trained \textit{offline} by continuous gradient descent after observing a fixed number of steps of the sequence $(X_k)_{k\geq 0}$. We emphasize that in our paper the RNN is trained \textit{online} and as the number of time steps $\rightarrow \infty$ (instead of prior literature which only considers a fixed finite number of time steps for the RNN and data sequence), where we update the parameters every time we observe a new time step of the data sequences $(X_k, Y_k)$. We study the asymptotics of the training of RNN as both the training time (and hence number of observations made for the input and output data sequences) and the width of the hidden layer $N$ $\rightarrow \infty$.

In order to derive a well-defined typical behaviour of the neural network when training our RNN in the limit as $N \to +\infty$, we choose the following learning rate:
\begin{equation}
   \alpha^N := \frac{\alpha}{N^{2-2\beta}},\label{Eq:LR}
\end{equation}
for some constant $\alpha > 0$.

\paragraph{Clipping the neural network output} It is challenging to prove a uniform bound for the neural network output $\hat{Y}^N_k$. To resolve this issue, we will clip the gradients used in the parameter updates. Gradient clipping is a standard method in deep learning and training RNNs \cite{Goodfellow, zhang2019gradient, exploding_gradient_2}. Note that the RNN output itself is not clipped. Furthermore, the gradient clipping will actually disappear in the final approximation argument as the number of hidden units $\rightarrow \infty$. Once the gradients in the parameter updates are clipped, the output $\hat{Y}^N_k$ can be proven to be bounded. We use the following clipping function:
\begin{definition}[Smooth clipping function \cite{CohenJiangSirignano}] \label{def:smooth_clipping_function}
A family of functions $(\psi^N)_{N \in \mathbb{N}}$, $\psi^N \in C^2_b(\R)$ is a family of \textit{smooth clipping function} with parameter $\gamma$ if the following are satisfied:
\begin{enumerate}
    \item $\abs{\psi^N (x)}$ is bounded by $2N^\gamma$,
    \item $\psi^N(x) = x$ for $x \in [-N^\gamma, N^\gamma]$,
    \item $\DS{\abs{\frac{d}{dx} \psi^N(x)} \leq 1}$.
\end{enumerate}
\end{definition}
The definition implies that $|\psi^N(x)| \leq |x|$ by the fundamental theorem of calculus. The parameter $\gamma$ should be small to ensure that $|\psi^N(\hat{Y}^N_k)|$ does not blow up too quickly as $N\to\infty$. In particular,
\begin{assumption}[Choice of $\gamma$] \label{as:choice_of_gamma}
We choose $\gamma \in (0, (1-\beta)/2)$, such that $\beta + 2\gamma < 1$.
\end{assumption}

As $\beta \in (1/2, 1)$, we must have $\gamma \in (0, 1/4)$. A construction of such a clipping function is provided in Appendix \ref{S:construction_of_clipping_function} following the arguments in \cite{NestruevJet2020Smao}. The online SGD algorithm with a tBPTT gradient estimate and clipped output is fully stated as Algorithm \ref{alg:onlineSGDwithBPTT}.

In this paper, we will make use of the notation ``$a \lesssim b$" to represent that there is a constant $\sfC > 0$ such that $a \leq \sfC b$. The constant $\sfC$ should be independent of $N$. We shall denote by $C_{T}<\infty$ constants that may depend on $T<\infty$. Using this notation, we next obtain the following bound (uniform in $i$) for how much the parameters change from their initial conditions during training.

\begin{algorithm}[t]
\caption{Online SGD with tBPTT ($\tau=1$)}\label{alg:onlineSGDwithBPTT}
\begin{algorithmic}[1]
\Procedure{SGDtBPTT}{$N, \lambda, T$} \Comment{network size, initial parameters distribution, running time}
    \State \textbf{Initalise:} initial parameters $\theta=(C_0,W_0,B_0) \sim \lambda$, initial memory layer $\forall i, \hat{S}^{i,N}_0 = 0$, step $k = 0$
    \While{$k \leq NT$}
    \ForAll{$i \in \{1,2,...,N\}$} \Comment{Truncated forward propagation}
        \State $\DS{\hat{S}^{i,N}_{k+1} \gets \sigma\bracket{(W^i_k)^\top X_k + \frac{1}{N} \sum_{j=1}^N B^j_k \hat{S}^{j,N}_k}}$ \Comment{Updating memory}
    \EndFor
    \State $\DS{\hat{Y}^N_k \gets \frac{1}{N^\beta} \sum_{i=1}^N C^i_k \hat{S}^{i,N}_{k+1}}$ \Comment{Updating output}
    \State $\DS{\hat{L}_k = \frac{1}{2}(\hat{Y}^N_k - Y_k)^2}$ \Comment{Computing loss}
    \ForAll{$i \in \{1,2,...,N\}$} \Comment{Truncated backward propagation on $\hat{L}_k$}
        \State $\DS{\Delta\hat{S}^{i,N}_{k+1} \gets \sigma'\bracket{(W^i_k)^\top X_k + \frac{1}{N} \sum_{j=1}^N B^j_k \hat{S}^{i,N}_k}}$
        \State $\DS{C^i_{k+1} = C^i_k - \frac{\alpha}{N^{2-\beta}}(\psi^N(\hat{Y}^N_k) - Y_k) \hat{S}^{i,N}_{k+1}}$
        \State $\DS{W^i_{k+1} = W^i_k - \frac{\alpha C^i_k}{N^{2-\beta}}(\psi^N(\hat{Y}^N_k) - Y_k) \Delta\hat{S}^{i,N}_{k+1}} X_k$
        \State $\DS{B^i_{k+1} = B^i_k - \frac{\alpha}{N^{3-\beta}}(\psi^N(\hat{Y}^N_k) - Y_k) \sum_{\ell=1}^N C^\ell_k  \hat{S}^{\ell,N}_k \Delta\hat{S}^{\ell,N}_{k+1}}$
    \EndFor
    \EndWhile
\EndProcedure
\end{algorithmic}
\end{algorithm}

\begin{lemma} \label{lem:evolution_of_parameters}
Fix $T>0$. If we choose $\gamma \in (0, (1-\beta)/2)$ as in assumption \ref{as:choice_of_gamma}, then for all $k$ with $k/N \leq T$, there exists $C_T > 0$ such that
\begin{equation} \label{eq:evolution_of_parameters}
|C_k^i - C_0^i| + \|W_k^i - W_0^i \| + |B_k^i - B_0^i| \leq \frac{C_T}{N^{1-\beta-\gamma}}.
\end{equation}
\end{lemma}

\begin{proof}
See Appendix \ref{SS:proof_of_evolution_of_parameters_with_clipping}.
\end{proof}

\begin{remark} \label{rmk:less_sharp_evolution_of_parameters}
One can also show, with the absence of clipping function $\psi^N$, that there exists a constant $C_T > 0$ such that for all $k \leq \lfloor NT \rfloor$
\begin{equation} \label{eq:increment_of_parameters}
|C^i_{k+1} - C^i_k| + \|W^i_{k+1} - W^i_k \| + |B^i_{k+1} - B^i_k| \leq \frac{C_T}{N},
\end{equation}
and therefore
\begin{equation}
|C^i_k - C^i_0| + \|W^i_k - W^i_0 \| + |B^i_k - B^i_0| \leq TC_T.
\end{equation}
The proof of this remark is also very similar - see Appendix \ref{SS:proof_of_evolution_of_parameters}. However, this does not reflect our intuition that the trained parameters are getting closer to where they are initialised as $N\to\infty$. The clipping function here is crucial in justifying the linearisation of the sample outputs with respect to the parameters in our analysis.
\end{remark}

Define the empirical measure of $(C^i_k, W^i_k, B^i_k)$ as $\lambda^N_k$, i.e.
\begin{equation}
    \lambda^N_k = \frac{1}{N} \sum_{i=1}^N \delta_{C^i_k, W^i_k, B^i_k}.\nonumber
\end{equation}

Lemma \ref{lem:evolution_of_parameters} indicates that the parameters at step $k$ should be close to the initial parameters on average, so $\lambda^N_k$ should be well-approximated by $\lambda^N$ in some sense. We therefore expect that for $k \leq \floor{TN}$, $\lambda_k^N$ converges to $\lambda$ (distribution of initialisation) in a sense to be specified. However, establishing such convergence using traditional mean-field techniques for proving convergence to a limit ODE (see e.g. \cite{RLasymptotics}) is difficult. To that end, the evolution of
\begin{equation*}
\la f, \lambda^N_k \ra = \int_{\R^{1+d+1}} f(c,w,b) \, \lambda^N_k(dc,dw,db) = \frac{1}{N} \sum_{i=1}^N f(C^i_k, W^i_k, B^i_k),
\end{equation*}
that is, the inner-product of the empirical distribution with a smooth test function $f \in C^\infty(\R^{1+d+1})$ does not look like a discretization of a differential equation. A new mathematical approach is therefore required to analyse the infinite-width RNN.

\section{Main Results}
\label{S:main_results}

There are three parts to the main results. In Section \ref{SS:dynamics_of_sample_memories}, we establish the convergence of the RNN memory layers. In Section \ref{SS:ModelDynamics}, we state the limiting ODE of the network outputs $\hat{Y}^N_k$ when the RNN is trained under tBPTT. Finally, in \ref{SS:LimitProblemAnalysis}, we show that tBPTT monotonically decreases the objective function.

\subsection{Dynamics of the RNN Hidden/Memory Layers} \label{SS:dynamics_of_sample_memories}
Let us define the function
\begin{equation}
\varsigma_{x,u}(w) := \sigma(w^\top x+u),
\end{equation}
With this, the memory layer $S^{i,N}_k(X;\theta_0)$ as defined in \eqref{eq:StandardRNNmemory} could be written as the following:
\begin{align}
S^{i,N}_{k+1}(X; \theta_0) &= \varsigma_{X_k, u^N_k}(W^i_0) = \sigma((W^i_0)^\top X_k + u^N_k), \\
u^N_k &= \frac{1}{N} \sum_{j=1}^N B^j_0 \, \varsigma_{X_{k-1}, u^N_{k-1}}(W^j_0) = \frac{1}{N} \sum_{j=1}^N B^j_0 \sigma((W^j_0)^\top X_{k-1} + u^N_{k-1}), \quad u^N_0 = 0,
\end{align}
The trained memory layer $\hat{S}^{i,N}_k$, as defined in Algorithm \ref{alg:onlineSGDwithBPTT}, could also be written as the following:
\begin{align}
\hat{S}^{i,N}_{k+1} &= \varsigma_{X_k, v^N_k}(W^i_k) = \sigma((W^i_k)^\top X_k + v^N_k), \\
v^N_k &= \frac{1}{N} \sum_{j=1}^N B^j_k \, \varsigma_{X_{k-1}, v^N_{k-1}}(W^j_{k-1}) = \frac{1}{N} \sum_{j=1}^N B^j_k \sigma((W^j_{k-1})^\top X_{k-1} + v^N_{k-1}), \quad v^N_0 = 0.
\end{align}

For notational purposes, let us also define as $\mathcal{W}=\R^{d}$ to be the state space of the random variables $W^{i}$. Next, we define the random function $h_k(w)$ that represents the mean-field limit of the RNN memory layer. The random function shall be determined by the mean-field term $u_k$ by the following:
\begin{align}
h_{k+1}(w) &= \varsigma_{X_k,u_k}(w) =  \sigma(w^\top X_k + u_k)\\
u_{k+1} &= \mu_b \int_{\mathcal{W}} \sigma(\sw^\top X_k + u_k) \, \lambda(d\sw), \quad u_0 = 0.
\end{align}
Finally, we define the following joint processes
\begin{equation*}
V^N_{k} = (X_k, Z_k, Y_k, v^N_k), \quad H^N_k = (X_k, Z_k, Y_k, u^N_k), \quad H_k = (X_k, Z_k, Y_k, u_k).
\end{equation*}

With this, we can now state a weak law of large numbers for the trained RNN memory layer (and also the memory layer at its initialisation):

\begin{lemma}[Dynamics of RNN Memory Layer] \label{lem:diff_vNk_hk}
For any $k \leq NT$, there is a constant $C_T > 0$ (depending on $T$) such that:
\begin{equation} \label{eq:diff_vNk_hk_repeated}
\|V^N_k - H^N_k\| = |v^N_k - u^N_k| \leq \frac{C_T}{N^{1-\beta-\gamma}},
\end{equation}
and for any $k \geq 0$, there is constant $C > 0$ such that
\begin{equation} \label{eq:diff_hNk_hk}
\mathbb{E}\|H^N_k - H_k\|^2 = \mathbb{E}|u^N_k - u_k|^2 \leq \frac{C}{N}.
\end{equation}
\end{lemma}

Equation \eqref{eq:diff_vNk_hk_repeated} is due to the fact that the parameters $(B^i_k, W^i_k)_{k\geq 0}$ are not too far from their initialisations $(B^i_0, W^i_0)_{k\geq 0}$, as shown in Lemma \ref{lem:evolution_of_parameters}. Equation \eqref{eq:diff_hNk_hk} is a direct result of the $L^2$ weak law of large numbers and the recursive inequality \ref{lem:recursive_inequality}.

\begin{proof}
See Subsections \ref{SS:3_1} and \ref{SS:3_2}.
\end{proof}

It remains to study the mean-field process $H_k$. Since $(X_k, Z_k)$ is ergodic, we can show that $H_k$ converges in distribution as $k\to\infty$ regardless of the distribution of $(X_0, Z_0, Y_0)$ (recall that $u_0$ is defined to be $0$). Formally,
\begin{theorem} \label{prop:conv_of_mu_k}
There exists a unique $\mu \in \mathcal{P}_2(\R^{d+d_z+d_\eta+1})$ such that for any initialisation $H_0 \sim \mu_0 \in \mathcal{P}_2(\R^{d+d_z+d_\eta+1})$, we have
\begin{equation}
    \Wass_2(\mu_k, \mu) \overset{k\to\infty}\to 0,
\end{equation}
where $\mu_k$ be the distribution of the sequence $(H_k)_{k\geq 0}$ (which depends on the distribution of $H_0$).
\end{theorem}

\begin{proof}
See Subsection \ref{SS:dynamics_of_memory_units}.
\end{proof}

\subsection{Dynamics of the RNN Outputs}
\label{SS:ModelDynamics}

To study the dynamics of the sequence $\hat{Y}^N_k$, we first study a pre-limit operator $g^N_k$, which acts on the space of bounded and second-time differentiable functions $C^2_b(\R^d)$:
\begin{equation*}
\cH = \{h : \max(\|h\|_\infty, \|h'\|_\infty, \|h''\|_\infty) \leq C\}.
\end{equation*}
The pre-limit operator is defined as:
\begin{equation}
g^N_k: h \in \cH \mapsto g^N_k(h) := \frac{1}{N^\beta} \sum_{i=1}^{N}C^i_k h(W^i_k).
\end{equation}

The $g^N_k(\cdot)$ is a sequence of random linear functionals in $\cH$ that evolves with the time step $k$. We further note that $\hat{S}^{i,N}_{k+1} = \varsigma_{X_k, v^N_k}(W^i_k)$, so
\begin{equation}
    \hat{Y}^N_k = g^N_k(\varsigma_{X_k, v^N_k}).\nonumber
\end{equation}

Lemma \ref{lem:evolution_of_parameters} asserts that the parameters do not move too far from their initialisations. Therefore, the evolution of $g^N_k$ when evaluated at a fixed function could be approximated by a Taylor expansion. This, in turn, results in an ODE approximation of the time-rescaled evolution of the operator $g^N_{\lfloor Nt \rfloor}$. Formally:
\begin{theorem} \label{thm:weak_convergence_thm}
Let $T<\infty$ be given and $t\leq T$. Define the infinite-dimensional equation for $g_t \in (C^2_b(\R^d))^*$, the dual space of $C^2_b(\R^d)$, such that for the test function $h \in \cH$:
\begin{align*}
g_t(h) &= - \alpha \int_0^t \sqbracket{\int_{\X} (g_s(\varsigma_{\sx, \su}) - \sy) \K_{\lambda}(h, \varsigma_{\sx,\su}) \, \mu(d\sH)} \, ds, \quad g_0(h) = 0, \\
\K_{\lambda}(h, \varsigma) &:= \int_{\mathcal{W}} [h(\sw) \varsigma(\sw) + \mu_{c^2} \nabla h(\sw)^\top \nabla \varsigma(\sw)] \, \lambda(d\sw). \numberthis
\label{eq:KernelLimitEqn}
\end{align*}
where $\sH = (\sx,\sz,\sy,\su) \in \X$, with $\X=\R^{d+d_{z}+d_{y}+1}$, and $\mu$ is the stationary distribution of Markov chain $(H_k)_{k\geq 0}$ obtained in Theorem \ref{prop:conv_of_mu_k}. Then, for $\gamma \in (0, (1-\beta)/2)$ sufficiently small, as $N\to\infty$ there exists a constant $C_T > 0$ such that
\begin{equation}
\sup_{t \in [0,T]} \E|g^N_t(h) - g_t(h)| \leq \frac{C_T}{N^\epsilon} \|h\|_{C^2},\nonumber
\end{equation}
where $\epsilon = (1-\beta-2\gamma) \wedge \gamma \wedge (\beta-1/2) > 0$.
\end{theorem}

A trivial corollary is the following
\begin{corollary}
For $h \in C^2_b$, the deterministic process $t \mapsto \E[g^N_t(h)]$ converges uniformly to $g_t$ on $t \in [0,T]$. In particular,
\begin{equation*}
|\E[g^N_t(h)] - g_t(h)| \leq \frac{C_T}{N^\epsilon} \|h\|_{C^2}.
\end{equation*}
\end{corollary}

The ODE can be derived by applying a Taylor expansion to the increments of $g^N_k(h)$, resulting in the emergence of a Neural Tangent Kernel (NTK). The NTK consists of two terms, one due to the update of $(C^i)$'s, and one due to the update of $(W^i)$'s. Note that the NTK $\mathcal{K}_\lambda$ is defined on an infinite-dimensional space $C^2_b(\R^d)$. Due to the ergodicity of the data, the ODE depends on the data sequence $(X_k)$ through the stationary measure as provided in Theorem \ref{prop:conv_of_mu_k}. The actual proof of weak convergence is presented in Section \ref{S:weak_convergence}.

\subsection{Limit RNN minimises the average loss} \label{SS:LimitProblemAnalysis}
Let us analyse the deterministic limiting ODE. Again we let $\sH = (\sx,\sz,\sy,\su)$ and $\tilde{\sH} = (\tilde{\sx},\tilde{\sz},\tilde{\sy},\tilde{\su})$. We therefore have the following:
\begin{align}
\int_{\X} g_t(\varsigma_{\tilde{\sx},\tilde{\su}}) \, \mu(d\tilde{\sH})
&= \int \, \sqbracket{-\alpha \int_0^t \int_{\X} (g_s(\varsigma_{\sx,\su}) - \sy) \mathcal{K}_{\lambda}(\varsigma_{\tilde{\sx},\tilde{\su}}, \varsigma_{\sx,\su}) \, \mu(d\sH) \, ds} \, \mu(d\tilde{\sH})\nonumber \\
&\overset{\text{(Fubini)}}= - \alpha \int_0^t \int_{\X}\int_{\X} (g_s(\varsigma_{\sx,\su}) - \sy) \mathcal{K}_{\lambda}(\varsigma_{\sx,\su}, \varsigma_{\tilde{\sx},\tilde{\su}}) \, \mu(d\sH) \, \mu(d\tilde{\sH}) \, ds, \label{Eq:IntergatedFunctMinizize}
\end{align}
which yields
\begin{equation*}
\frac{d}{dt} \int_{\X} g_t(\varsigma_{\tilde{\sx},\tilde{\su}}) \, \mu(d\tilde{\sH}) = -\alpha \int_{\X}\int_{\X} (g_t(\varsigma_{\sx,\su}) - \sy) \mathcal{K}_{\lambda}(\varsigma_{\sx,\su}, \varsigma_{\tilde{\sx},\tilde{\su}}) \, \mu(d\sH) \, \mu(d\tilde{\sH}).
\end{equation*}
We can further compute that
\begin{align*}
&\phantom{=}\frac{d}{dt} \sqbracket{\frac{1}{2} \int_{\X} (g_t(\varsigma_{\tilde{\sx},\tilde{\su}}) - \tilde{\sy})^2 \, \mu(d\tilde{\sH})} \\
&= \int_{\X} \sqbracket{(g_t(\varsigma_{\tilde{\sx},\tilde{\su}}) - \tilde{\sy}) \frac{dg_t(\varsigma_{\sx,\su})}{dt}}\, \mu(d\tilde{\sH}) \\
&= -\alpha \int_{\X}\int_{\X} \mathcal{K}_{\lambda}(\varsigma_{\tilde{\sx},\tilde{\su}}, \varsigma_{\sx,\su}) (g_t(\varsigma_{\tilde{\sx}, \tilde{\su}}) - \tilde{\sy}) (g_t(\varsigma_{\sx,\su}) - \sy) \, \mu(d\sH) \, \mu(d\tilde{\sH}). \numberthis
\end{align*}
We can now prove that the loss (i.e., the objective function) is monotonically decreasing during training.
\begin{proposition}[Minimisation of averaged square loss]
\label{PropositionDescentDirection}
\begin{equation}
\frac{d}{dt} \sqbracket{\frac{1}{2} \int_\mathcal{X}  (g_t(\varsigma_{\tilde{\sx},\tilde{\su}}) - \tilde{\sy})^2  \, \mu(d\tilde{\sH})} \leq 0.
\end{equation}
\end{proposition}

\begin{proof}
Recall
\begin{equation*}
\K_{\lambda}(h, \varsigma) = \int_{\mathcal{W}} \left[h(\sw) \varsigma(\sw) + \mu_{c^2} \sum_{j=1}^d \frac{\partial h}{\partial w_j}(\sw) \frac{\partial \varsigma}{\partial w_j}(\sw) \right] \, \lambda(d\sw).
\end{equation*}
In particular, we have
\begin{equation*}
\K_{\lambda}(\varsigma_{\sx,\su}, \varsigma_{\tilde{\sx},\tilde{\su}}) = \int_{\mathcal{W}} \left[\varsigma_{\sx,\su}(\sw) \varsigma_{\tilde{\sx},\tilde{\su}}(\sw) + \mu_{c^2} \sigma'(\sw^\top \sx + \su) \sigma'(\sw^\top \tilde{\sx} + \tilde{\su}) \sx^\top \tilde{\sx} \right] \, \lambda(d\sw).
\end{equation*}
Define $\sx_j$ as the $j$-th entry for the vector $\sx$, we have
\begin{align*}
&\phantom{=}\frac{d}{dt} \sqbracket{\frac{1}{2} \int_{\X} (g_t(\varsigma_{\sx,\su}) - \tilde{\sy})^2 \, \mu(d\tilde{\sH})} \\
&= -\alpha \int_{\X}\int_{\X}\int_{\mathcal{W}} \left[\varsigma_{\sx,\su}(\sw) \varsigma_{\tilde{\sx},\tilde{\su}}(\sw) + \mu_{c^2} \sigma'(\sw^\top \sx + \su) \sigma'(\sw^\top \tilde{\sx} + \tilde{\su}) \sx^\top \tilde{\sx} \right] \\
&\phantom{=}\times (g_t(\varsigma_{\tilde{\sx}, \tilde{\su}}) - \tilde{\sy}) (g_t(\varsigma_{\sx,\su}) - \sy) \, \lambda(d\sw)\, \mu(d\sH) \, \mu(d\tilde{\sH}) \\
&= -\alpha \int_{\mathcal{W}}\Bigg[\left(\int_{\X} (g_t(\varsigma_{\sx,\su}) - \sy) \, \varsigma_{\sx,\su}(\sw) \,  \mu(d\sH)\right)^2
\\
&\phantom{=}+ \mu_{c^2} \sum_{j=1}^d \left(\int_{\X} (g_t(\varsigma_{\sx,\su})) - \sy) \, \sigma'(\sw^\top \sx + \su) \, \sx_j \,  \mu(d\sH)\right)^2 \Bigg]\lambda(d\sw) \, \\
&\leq 0.
\end{align*}
\end{proof}

Therefore, the function $t \mapsto \frac{1}{2} \int_\mathcal{X}  (g_t(\varsigma_{\tilde{x}, \tilde{u}}) - \tilde{\sy})^2 \, \mu(d\tilde{\sH})$ is non-increasing. This is a useful theoretical guarantee that emerges from the limit analysis. The pre-limit training algorithm (tBPPT) truncates the chain rule -- see the algorithm (\ref{alg:onlineSGDwithBPTT}) -- and therefore it is not guaranteed that parameter updates are in a descent direction for the loss function. That is, in principle, the loss (for the long-run distribution of the data sequence) may actually increase when the parameters are updated. Proposition \ref{PropositionDescentDirection} proves that, as $N,k \rightarrow \infty$, the RNN model
will be updated in a descent direction for the loss function.

\section{Proof of Dynamics of RNN Memory Layer}
\label{S:proof_of_dynamics_of_memories}

\subsection{Reduction to Initialisations}
\label{SS:3_1}

Let $E^{N,1}_k = v^N_k - u^N_k$. Then $E^{N,1}_0 = 0$. For $k\geq 1$, we have:
\begin{align*}
E^{N,1}_k
&= \frac1N \sum_{j=1}^N \left[B^j_k \sigma((W^j_{k-1})^\top X_{k-1} + v^N_{k-1}) - B^j_0 \sigma((W^j_0)^\top X_{k-1} + u^N_{k-1})) \right] \\
&= \frac{1}{N} \sum_{j=1}^N \Big[(B^j_k - B^j_0) \sigma((W^j_{k-1})^\top X_{k-1} + v^N_{k-1}) \\
&\phantom{=}+ \frac{1}{N}\sum_{j=1}^N B^i_0 \big(\sigma((W^j_{k-1})^\top X_{k-1} + v^N_{k-1}) - \sigma((W^j_0)^\top X_{k-1} + u^N_{k-1}) \big) \Big].
\end{align*}
The first term could be bounded by
\begin{align*}
&\phantom{=}\abs{\frac{1}{N} \sum_{j=1}^N \big[(B^j_k - B^j_0) \sigma((W^j_{k-1})^\top X_{k-1} + v^N_{k-1}) \big]} \\
&\leq \frac1N \sum_{j=1}^N \abs{B^j_k - B^j_0} \abs{\sigma((W^j_{k-1})^\top X_{k-1} + v^N_{k-1})} \\
&\lesssim \frac{1}{N} \sum_{j=1}^N \abs{B^j_k - B^j_0} \\
&\overset{\text{(Lemma \ref{lem:evolution_of_parameters})}}\leq \frac{C_T}{N^{1-\beta-\gamma}}.
\end{align*}
The second term could also be bounded as followed:
\begin{align*}
&\phantom{=}\left|\frac{1}{N}\sum_{j=1}^N B^j_0 \big(\sigma((W^j_{k-1})^\top X_{k-1} + v^N_{k-1}) - \sigma((W^j_0)^\top X_{k-1} + u^N_{k-1}) \big) \right|  \\
&\leq \frac{C_{\sigma'}}{N}\sum_{j=1}^N B^j_0\big[|(W^j_{k-1} - W^j_0)^\top X_k| + |v^N_{k-1} - u^N_{k-1}| \big] \\
&\leq \frac{C_{\sigma'} K_b C_x}{N} \sum_{j=1}^N \|W^j_{k-1} - W^j_0\| + C_{\sigma'} K_b E^{N,1}_{k-1} \\
&\overset{\text{(Lemma \ref{lem:evolution_of_parameters})}}\leq \frac{C_T}{N^{1-\beta-\gamma}} + C_{\sigma'} K_b E^{N,1}_{k-1}.
\end{align*}
Therefore, we have
\begin{equation}
|E^{N,1}_k| \leq \frac{C_T}{N^{1-\beta-\gamma}} + C_{\sigma'} K_b |E^{N,1}_{k-1}|.\label{eq:evolution_ENk1_new}
\end{equation}
Since $C_{\sigma'}K_b < 1$ by Assumption \ref{as:ergodicity_conditions}, we conclude from Lemma \ref{lem:recursive_inequality} that for all $k \geq 0$ we have
\begin{equation}
|E^{N,1}_k| \leq  \frac{C_T}{N^{1-\beta-\gamma}}. \nonumber
\end{equation}

\subsection{Weak Law of Large Numbers}
\label{SS:3_2}

Recall that
\begin{align*}
u^N_k &= \frac{1}{N} \sum_{j=1}^N B^j_0 \sigma((W^j_0)^\top X_{k-1} + u^N_{k-1}), \quad u^N_0 = 0, \\
u_k &= \mu_b \int_{\mathcal{W}} \sigma(\sw^\top X_k + u_k) \, \lambda(d\sw), \quad u_0 = 0.
\end{align*}
Since $u^N_k$ is an empirical average, it should converge to $u_k$ by an argument of weak law of large numbers. To make this formal, let us define $E^{N,2}_k = u^N_k - u_k$. We have $E^{N,2}_0 = 0$, and for $k \geq 1$:
\begin{align*}
E^{N,2}_k
&=
\frac{1}{N} \sum_{j=1}^N B^j_0 \sigma((W^j_0)^\top X_{k-1} + u^N_{k-1}) - \mu_b \int_{\mathcal{W}} \sigma(\sw^\top X_{k-1} + u_{k-1}) \, \lambda(d\sw) \\
&= \frac{1}{N} \sum_{j=1}^N \left[B^j_0 \Big(\sigma((W^j_0)^\top X_{k-1} + u^N_{k-1}) - \sigma((W^j_0)^\top X_{k-1} + u_{k-1}) \Big)\right] \\
&\phantom{=}+ \frac{1}{N} \sum_{j=1}^N (B^j_0 - \mu_b) \sigma((W^j_0)^\top X_{k-1} + u_{k-1}) \\
&\phantom{=}+ \frac{\mu_b}{N} \sum_{j=1}^N \sigma((W^j_0)^\top X_{k-1} + u_{k-1}) - \mu_b \int_{\mathcal{W}} \sigma(\sw^\top X_{k-1} + u_{k-1}) \, \lambda(d\sw). \numberthis \label{eq:EN2_breakdown}
\end{align*}
By Cauchy-Schwarz inequality, for any $a,b,c \in \R$:
\begin{equation}
(a+b+c)^2 \leq 2a^2 + 2(b+c)^2 \leq 2a^2+4b^2+4c^2. \label{eq:Cauchy_Schwarz_three_terms}
\end{equation}
Squaring both sides of \eqref{eq:EN2_breakdown}, apply  \eqref{eq:Cauchy_Schwarz_three_terms} and take expectation both sides, we have
\begin{align}
&\phantom{=} \E\sqbracket{E^{N,2}_k}^2 \nonumber \\
&\leq 2\E\Bigg[\frac{1}{N} \sum_{j=1}^N \Big[B^j_0 \big(\sigma((W^j_0)^\top X_{k-1} + u^N_{k-1}) - \sigma((W^j_0)^\top X_{k-1} + u_{k-1}) \big)\Big] \Bigg]^2 \nonumber \\
&\phantom{=}+ 4\E\Bigg[\frac{1}{N} \sum_{j=1}^N (B^j_0 - \mu_b) \sigma((W^j_0)^\top X_{k-1} + u_{k-1}) \Bigg]^2 \label{eq:EN2_squared_breakdown} \\
&\phantom{=}+ 4\E\Bigg[\frac{\mu_b}{N} \sum_{j=1}^N \sigma((W^j_0)^\top X_{k-1} + u_{k-1}) - \mu_b \int_{\mathcal{W}} \sigma(\sw^\top X_{k-1} + u_{k-1}) \, \lambda(d\sw) \Bigg]^2. \nonumber
\end{align}

The main observation here is that $(u_k)_{k\geq 0}$ is no longer dependent on $B^j_0$ and $W^j_0$, so we could prove a weak law of large numbers for the second and third terms in \eqref{eq:EN2_squared_breakdown}. In fact, for the second term, we see that $B^j_0$ is independent of $\sigma((W^j_0)^\top X_{k-1} + u_{k-1})$, so
\begin{eqnarray}
&\phantom{=}& \E\bigg[\frac{1}{N} \sum_{j=1}^N (B^j_0 - \mu_b) \bracket{\sigma((W^j_0)^\top X_{k-1} + u_{k-1}} \bigg]^2 \nonumber\\
&=& \E\bigg[\frac{1}{N^2} \sum_{j_1, j_2 = 1}^N (B^{j_1}_0 - \mu_b) (B^{j_2}_0 - \mu_b) \nonumber \\
&\phantom{=}&\times \big(\sigma((W^{j_1}_0)^\top X_{k-1} + u_{k-1} \big) \big(\sigma((W^{j_2}_0)^\top X_{k-1} + u_{k-1} \big) \bigg] \nonumber\\
&=& \frac{1}{N^2} \sum_{j_1, j_2 = 1}^N \E[B^{j_1}_0 - \mu_b] \, \E[B^{j_2}_0 - \mu_b]  \\
&\phantom{=}& \times \E\left[\bracket{\sigma((W^{j_1}_0)^\top X_{k-1} + u_{k-1}} \bracket{\sigma((W^{j_2}_0)^\top X_{k-1} + u_{k-1}} \right] \nonumber \\
&=& \frac{1}{N^2} \sum_{j=1}^N \E[B^{j}_0 - \mu_b]^2 \E[\sigma((W^j_0)^\top X_k + u_k)]^2 \nonumber \\
&\leq& \frac{4K_b^2 C_\sigma^2}{N}. \nonumber \label{eq:EN2_second_term}
\end{eqnarray}
For the third term, we note that the random variables $\big(\sigma((W^j_0)^\top X_{k-1} + u_{k-1}) \big)_{j=1}^N$ are independent when conditioned on the inputs $X_{0:k-1} := (X_0, ..., X_{k-1})$. Furthermore, the integral $\int_{\mathcal{W}} \sigma(\sw^\top X_{k-1} + u_{k-1}) \, \lambda(d\sw)$ is the conditional expectation $\E[\sigma((W^j_0)^\top X_{k-1} + u_{k-1}) \mid X_{0:k-1}]$. Therefore,
\begin{align*}
&\phantom{=}\mathbb{E}\sqbracket{\frac{\mu_b}{N} \sum_{j=1}^N \sigma((W^j_0)^\top X_{k-1} + u_{k-1}) - \mu_b \int_{\mathcal{W}} \sigma(\sw^\top X_{k-1} + u_{k-1}) \, \lambda(d\sw)}^2 \\
&= \frac{\mu_b^2}{N^2} \sum_{j_1, j_2=1}^N \E\Bigg[\E\bigg[\bigg[\sigma((W^{j_1}_0)^\top X_{k-1} + u_{k-1}) - \int_{\mathcal{W}} \sigma(\sw^\top X_{k-1} + u_{k-1}) \, \lambda(d\sw) \bigg] \\
&\phantom{=}\times \bigg[\sigma((W^{j_2}_0)^\top X_{k-1} + u_{k-1}) - \int_{\mathcal{W}} \sigma(\sw^\top X_{k-1} + u_{k-1}) \, \lambda(d\sw) \bigg] \, \bigg|\, X_{0:k-1} \bigg] \Bigg] \\
&= \frac{\mu_b^2}{N^2} \sum_{j_1, j_2=1}^N \E\Bigg[\E\bigg[\sigma((W^{j_1}_0)^\top X_{k-1} + u_{k-1}) - \int_{\mathcal{W}} \sigma(\sw^\top X_{k-1} + u_{k-1}) \, \lambda(d\sw) \, \bigg|\, X_{0:k-1} \bigg] \\
&\phantom{=}\times \E\bigg[\sigma((W^{j_2}_0)^\top X_{k-1} + u_{k-1}) - \int_{\mathcal{W}} \sigma(\sw^\top X_{k-1} + u_{k-1}) \, \lambda(d\sw) \,\bigg|\, X_{0:k-1} \bigg] \Bigg] \\
&= \frac{\mu_b^2}{N^2} \sum_{j=1}^N \E\Bigg[\E\bigg[\sigma((W^{j}_0)^\top X_{k-1} + u_{k-1}) - \int_{\mathcal{W}} \sigma(\sw^\top X_{k-1} + u_{k-1}) \, \lambda(d\sw) \, \bigg|\, X_{0:k-1} \bigg]^2 \\
&\leq \frac{\mu_b^2}{N^2} \sum_{j=1}^N \E\Bigg[\E\bigg[\bigg[\sigma((W^{j}_0)^\top X_{k-1} + u_{k-1}) - \int_{\mathcal{W}} \sigma(\sw^\top X_{k-1} + u_{k-1}) \, \lambda(d\sw) \bigg]^2 \, \bigg|\, X_{0:k-1} \bigg] \\
&= \frac{\mu_b^2}{N^2} \sum_{j=1}^N \E\bigg[\sigma((W^{j}_0)^\top X_{k-1} + u_{k-1}) - \int_{\mathcal{W}} \sigma(\sw^\top X_{k-1} + u_{k-1}) \, \lambda(d\sw) \bigg]^2 \\
&\leq  \frac{4K^2_b C^2_\sigma}{N}. \numberthis \label{eq:EN2_third_term}
\end{align*}
It remains to bound the first term, which could be bounded using Cauchy-Schwarz:
\begin{align}
&\phantom{=}\E\Bigg[\frac{1}{N} \sum_{j=1}^N \Big[B^j_0 \Big(\sigma((W^j_0)^\top X_{k-1} + u^N_{k-1}) - \sigma((W^j_0)^\top X_{k-1} + u_{k-1}) \Big)\Big] \Bigg]^2 \nonumber \\
&\leq \frac{1}{N^2} \E\Bigg[\Bigg[\sum_{j=1}^N (B^j_0)^2 \Bigg] \sum_{j=1}^N \Big[\sigma((W^j_0)^\top X_{k-1} + u^N_{k-1}) - \sigma((W^j_0)^\top X_{k-1} + u_{k-1}) \Big]^2 \Bigg] \label{eq:EN2_first_term} \\
&\leq \frac{1}{N^2} \E\sqbracket{N{K_b^2 C_{\sigma'}^2} \bracket{\sum_{j=1}^N \bracket{u^N_{k-1} -u_{k-1})}^2}} \nonumber \\
&\leq {K_b^2 C_{\sigma'}^2} \E(E^{N,2}_{k-1})^2. \nonumber
\end{align}
Combining \eqref{eq:EN2_squared_breakdown} with \eqref{eq:EN2_second_term}, \eqref{eq:EN2_third_term} and \eqref{eq:EN2_first_term} yields:
\begin{equation}
\E\sqbracket{E^{N,2}_k}^2 \leq 2 {K_b^2 C_{\sigma'}^2} \E\bracket{E^{N,2}_{k-1}}^2 + \frac{32 K^2_b C^2_\sigma}{N} \label{eq:evolution_ENk1}
\end{equation}
Assumption \ref{as:ergodicity_conditions} asserts that $2 K_{B}^2 C_{\sigma'}^2 <1$, so by Lemma \ref{lem:recursive_inequality}, for all $k \geq 0$ we have
\begin{equation*}
\E\bracket{E^{N,2}_k}^2 \leq \frac{32 K^2_b C^2_\sigma}{N(1-{2K_b^2 C_{\sigma'}^2})}.
\end{equation*}

\subsection{Dynamics of the memory units at the asymptotic limit}
\label{SS:dynamics_of_memory_units}

Let us first define the sequence $U_k = (X_k, Z_k, \eta_k, u_k)_{k\geq 0}$, which is defined recursively as
\begin{align*}
(X_{k+1}, Z_{k+1}) \,|\, (X_k, Z_k) &\sim P((X_k, Z_k), \,\cdot\,), \\
u_{k+1} &= \mu_b \int_{\mathcal{W}} \sigma(\sw^\top X_k + u_k) \, \lambda(d\sw), \quad u_0 = 0,
\end{align*}
and $(\eta_k)_{k\geq 0}$ is an i.i.d. sequence of random variable independent of $(X_k, Z_k, u_k)_{k\geq 0}$. Note that the $U_k$ is a time-homogeneous Markov chain with transitional probability $Q$, such that for any $A \in \mathcal{B}(\R^{d+d_z}), B, C \in \mathcal{B}(\R)$,
\begin{eqnarray}
&\phantom{=}& Q((x,z,\eta,u), A \times B \times C) \nonumber \\
&=& \mathbb{P}((X_{k+1}, Z_{k+1}) \in A, \, \eta_{k+1} \in B, u_{k+1} \in C \,|\, X_k = x, \, Z_k = z, \,\eta_k = \eta, \, u_k = u) \\
&=& P((x,z), A) \, \mu_\eta(B) \, \mathbb{I}_C\bracket{\mu_b \int_{\mathcal{W}} \sigma\bracket{\sw^\top x + u} \, \lambda(d\sw)}, \nonumber
\end{eqnarray}
where $\mu_\eta$ is the distribution of $\eta_1$. We shall use $U$ for the shorthand of $(x,z,\eta,u)$.

With the above notions, we could show that $Q$ is a contraction in the following sense:
\begin{lemma} \label{lem:perturbation_of_kernel}
For $U := (x,z,\eta,u), \tilde{U} := (\tilde{x}, \tilde{z},\tilde{\eta},\tilde{u}) \in \R^{d+d_z+d_{\eta}+1}$, we have
\begin{equation}
\Wass_2(Q(U,\cdot), Q(\tilde{U},\cdot)) \leq q_0 |U - \tilde{U}|.
\end{equation}
\end{lemma}
\begin{proof}
By Assumption \ref{as:data_generation} and Remark \ref{rmk:optimal_coupling}, there exists a coupling $\gamma_{(x,z), (\tilde{x},\tilde{z})}$ between $P((x,z),\cdot)$ and $P((\tilde{x}, \tilde{z}), \cdot)$ such that
\begin{eqnarray}
&\phantom{=}& \bracket{\Wass_2(P((x,z), \cdot), \, P(\tilde{x}, \tilde{z}), \cdot)}^2 \nonumber \\
&=& \int_{\R^{d+1}\times \R^{d+1}} |(\sx,\sz) - (\tilde{\sx},\tilde{\sz})|^2 \, \gamma_{(x,z),(\tilde{x},\tilde{z})} (d\sx,d\sz,d\tilde{\sx},d\tilde{\sz}) \\
&\leq& L^2 |(x,z) - (\tilde{x}, \tilde{z})|^2. \nonumber
\end{eqnarray}
From this, we construct a new coupling $\tilde{\gamma}_{U,\bar{U}}(\cdot)$ between $Q(U,\cdot)$ and $Q(\bar{U},\cdot)$: for any $A_1, A_2 \in \mathcal{B}(\R^{d+d_z})$, $B_1, B_2 \in \mathcal{B}(\R^{d_y})$ and $C_1, C_2 \in \mathcal{B}(\R^{d_\eta})$:,
\begin{eqnarray}
&\phantom{=}& \tilde{\gamma}_{U,\tilde{U}}(A_1 \times B_1 \times A_2 \times B_2 \times C_1 \times C_2) \nonumber \\
&=& \gamma_{(x,z),(\tilde{x},\tilde{z})}(A_1 \times A_2) \\
&\phantom{=}& \times\mathbb{I}_{B_1}\bracket{\mu_b \int_{\mathcal{W}} \sigma\bracket{\sw^\top x + u} \, \lambda(d\sw)} \mathbb{I}_{B_2}\bracket{\mu_b \int_{\mathcal{W}} \sigma\bracket{\sw^\top \tilde{x} + \tilde{u}} \, \lambda(d\sw)} \nonumber \\
&\phantom{=}& \times (\mathsf{diag}\#\mu_\eta)(C_1 \times C_2), \nonumber
\end{eqnarray}
where $\mathsf{diag} : x \in \R \mapsto (x,y) \in \R^2$ is the diagonal map, and the measure $\mathsf{diag\#\mu_\eta}$ being the push-forward of $\mu_\eta$ by $\mathsf{diag}$, known as the \textit{diagonal coupling}. Using the change of variable formula, we have:
\begin{eqnarray}
&\phantom{=}& \bracket{\Wass_2(Q(U,\cdot), Q(\tilde{U},\cdot))}^2 \nonumber \\
&\leq& \int |(\sx,\sz,\eta,\su)-(\tilde{\sx},\tilde{\sz},\tilde{\eta},\tilde{\su})|^2 \, \gamma_{U,\tilde{U}}(d\sx,d\sz,d\eta,d\su,d\tilde{\sx},d\tilde{\sz},d\eta,d\tilde{\su}) \nonumber \\
&=& \int \sqbracket{|(\sx,\sz)-(\tilde{\sx},\tilde{\sz})|^2 + |\su - \tilde{\su}|^2} \, \tilde{\gamma}_{U,\tilde{U}}(d\sx,d\sz,d\su,d\tilde{\sx},d\sz,d\su) \nonumber \\
&=& \int |(\sx,\sz)-(\tilde{\sx},\tilde{\sz})|^2 \gamma_{(x,z),(\tilde{x},\tilde{z})}(d\sx,d\sz,d\tilde{\sx},d\tilde{\sz}) \nonumber \\
&\phantom{=}& + \abs{\mu_{b}\int_{\mathcal{W}} \sqbracket{\sigma(\sw^\top x + u) - \sigma(\sw^\top \tilde{x} + \tilde{u})} \, \lambda(d\sw)}^2 \\
&\leq& L^2 |(x,z) - (\tilde{x}, \tilde{z})|^2 + \mu_{b}^2 C^2_{\sigma'} \int_{\mathcal{W}} \sqbracket{\sw^\top (x - \tilde{x}) + (u - \tilde{u})}^2 \, \lambda(d\sw) \nonumber \\
&\leq& L^2 |(x,z) - (\tilde{x}, \tilde{z})|^2 + 2\mu_{b}^2 C^2_{\sigma'} \int_{\mathcal{W}} \sqbracket{|\sw|^2 |x - \tilde{x}|^2 + |u - \tilde{u}|^2} \, \lambda(d\sw) \nonumber \\
&\leq& (L^2 + 2 \mu_{b}^2 \mu_{w^{2}}C^2_{\sigma'}) |(x,z) - (\tilde{x}, \tilde{z})|^2 + 2\mu_{b}^2 C^2_{\sigma'}  |u - \tilde{u}|^{2} \, \nonumber \\
&\leq& q_0^2 |U - \tilde{U}|^2. \nonumber
\end{eqnarray}
\end{proof}
Recall that the transition kernel induces an adjoint operator over the space of probability measures $\mathcal{M}(\R^{d+d_z+d_\eta+1})$:
\begin{equation}
    \rho \in \mathcal{M}(\R^{d+d_z+d_\eta+1}) \mapsto Q^\vee \rho, \quad Q^\vee \rho(\cdot) = \int Q((\sx, \sz, \eta, \su), \cdot) \, \rho(d\sx, d\sz, d\eta, d\su).
\end{equation}
Then
\begin{proposition}[Contractivity] \label{prop:conv_of_nu_k}
For all $\rho, \tilde{\rho} \in \mathcal{P}_2(\R^{d+d_z+d_\eta+1})$,
\begin{equation}
    \Wass_2(Q^\vee \rho, Q^\vee \tilde{\rho}) \leq q_0 \Wass_2(\rho, \tilde{\rho}),
\end{equation}
Since $q_0 < 1$ by Assumption \ref{as:ergodicity_conditions}, the operator $Q^\vee$ has a unique contractive fixed point $\nu \in \mathcal{P}_2(\R^{d+d_z+1})$. Such fixed point is the stationary measure of the Markov chain $(U_k)_{k\geq 0}$, and for any $\rho \in \mathcal{P}_2(\R^{d+d_z+d_\eta+1})$,
\begin{equation}
\Wass_2((Q^\vee)^k \rho, \nu) \leq q_0^k \, \Wass_2(\rho, \nu) \overset{k\to\infty}\to 0,
\end{equation}
\end{proposition}

\begin{proof}
The proof of this proposition is very similar to the proof of \cite[Proposition 14.3]{DobrushinRoland2006LoPT}.
\end{proof}

We are now ready to prove the main result.

\begin{proof}[Proof of Theorem \ref{prop:conv_of_mu_k}]
Observe that $H_k = \tilde{f}(U_k)$, where
\begin{equation*}
    \tilde{f}: U:=(x,z,\eta,u) \mapsto (x,z,f(x,z,\eta),u),
\end{equation*}
and that $\tilde{f}$ is Lipschitz. With this, we define $\mu = \tilde{f}\#\nu$, where $\nu$ is given in Proposition \ref{prop:conv_of_nu_k}. Further assume that $\mu_0 = \tilde{f}\#\rho$ for some $\rho \in \mathcal{P}_2(\R^{d+d_z+d_\eta+1})$. By Proposition \ref{prop:conv_of_nu_k} there is a coupling $\gamma_k$ between $(Q^\vee)^k \rho$ and $\nu$ such that
\begin{align*}
    \int_{\X^{2}} |(\sx,\sz,\eta,\su) - (\tilde{\sx},\tilde{\sz},\tilde{\eta},\tilde{\su})|^2 \, \gamma_k(d\sx,d\sz,d\eta,d\su,d\tilde{\sx},d\tilde{\sz},d\eta,d\tilde{\su}) \leq q_0 \Wass_2(\rho,\nu).
\end{align*}
Let $\tilde{\tilde{f}} : (x,z,\eta,u,\tilde{x},\tilde{z},\tilde{\eta},\tilde{u}) \mapsto (x,z,f(x,z,\eta),u,\tilde{x},\tilde{z},f(\tilde{x},\tilde{z},\tilde{\eta}),\tilde{u})$, and $\tilde{\gamma} = \tilde{\tilde{f}} \# \gamma$, then there is a $C > 0$ (from Lipschitzness of $f$) such that
\begin{align*}
&\phantom{=}\int_{\X^{2}} |(\sx,\sz,\sy,\su) - (\tilde{\sx},\tilde{\sz},\tilde{\sy},\tilde{\su})|^2 \, \tilde{\gamma}_k(d\sx,d\sz,d\sy,d\su,d\tilde{\sx},d\tilde{\sz},d\sy,d\tilde{\su}) \\
&\lesssim \int_{\X^{2}} |(\sx,\sz,\eta,\su) - (\tilde{\sx},\tilde{\sz},\tilde{\eta},\tilde{\su})|^2 \, \gamma_k(d\sx,d\sz,d\eta,d\su,d\tilde{\sx},d\tilde{\sz},d\eta,d\tilde{\su}) \\
&\lesssim q_0^{k} \Wass_2(\rho,\nu).
\end{align*}
Note that due to Assumption \ref{as:data_generation}, $H_0$ is bounded and $\E|\eta_k|^2 < +\infty$, so there exists a constant $\sfC > 0$, depending on $C_x, C_\eta$ and $C_\sigma$, such that
\begin{align}
\Wass_2(\rho,\nu) \leq \sfC,\label{eq:apriori_bound_for_initialisation}
\end{align}
therefore
\begin{align*}
\int_{\X^{2}} |(\sx,\sz,\sy,\su) - (\tilde{\sx},\tilde{\sz},\tilde{\sy},\tilde{\su})|^2 \, \tilde{\gamma}_k(d\sx,d\sz,d\sy,d\su,d\tilde{\sx},d\tilde{\sz},d\sy,d\tilde{\su}) \lesssim q_0^{k} \overset{k\to\infty}\to  0.
\end{align*}
\end{proof}

\begin{remark} \label{rmk:support_of_mu}
As we assume in Assumption \ref{as:data_generation} that:
$$|(X_k, Z_k)| \leq C_x, \quad|Y_k| = |f(X_k , Z_k, \eta_k)| \leq C_y,$$
and in Assumption \ref{as:activation_function}, $$|u_k| = |\sigma(\cdot)| \leq C_{\sigma},$$
so the support of $\mu$ must be bounded. In fact, we have
\begin{align*}
\mathsf{supp} \, \mu \subseteq \{(x,z,y,u):|(x,z)| \leq C_x, |y| \leq C_y, |u| \leq C_{\sigma}\}.
\end{align*}
\end{remark}

\section{Proof of weak convergence}
\label{S:weak_convergence}
We shall now prove that the time-rescaled evolution of $(g^N_{\lfloor Nt \rfloor})_{t \in [0,T]}$ converges to the limiting ODE \eqref{eq:KernelLimitEqn}. Here $h$ is restricted to a second-time differentiable function, i.e., is an element in $C^2_b(\R^d)$. Let $\|f\|_{C^2}$ be the $C^2$-norm of $f$ defined by the following:
\begin{equation}
\|f\|_{C^2} = \sup_x \max_{i,j,k} \sqbracket{|f(x)|, \abs{\frac{\partial f(x)}{\partial x^i}}, \abs{\frac{\partial^2 f(x)}{\partial x^j \partial x^k}}}.
\end{equation}
Furthermore, we abuse notation to allow the unimportant constants $0<C, C_T <\infty$ vary from line to line.

The main idea of the proof is to note that for any $k \leq NT$ and $h \in \mathcal{H}$:
\begin{equation}
g^N_k(h) = \sum_{m=0}^{k-1}\triangle g^N_m(h), \quad \triangle g^N_m(h) = g^N_{m+1}(h)-g^N_m(h).
\end{equation}
We shall approximate the increments $\triangle g^N_m(h)$. This will show that $g^N_k(h)$ is a discretisation of the limiting ODE \eqref{eq:KernelLimitEqn}, and therefore the rescaled process $g^N_{\lfloor Nt\rfloor}(h)$ converges (weakly) to the limiting ODE \eqref{eq:KernelLimitEqn}.

The actual convergence of the ODE is broken into a few steps:
\begin{enumerate}
\item In Subsection \ref{SS:well-posedness}, we prove that the limiting ODE is well-posed. The bounds established will be used in the weak convergence analysis, see step 5.
\item In Subsection \ref{SS:proof_of_KernelEvolution1}, we Taylor expand the increments $\triangle g^N_m(h)$ to show that for any $k \leq NT$,
\begin{align*}
\abs{\triangle g^N_m(h) - \delta^{(1)} g^N_m(h)} \leq \frac{C_T}{N^{3-\beta-2\gamma}} \|h\|_{C^2},
\end{align*}
where
\begin{align*}
\delta^{(1)} g^N_m(h)
&= - \frac{\alpha}{N^2} (\psi^N(\hat{Y}^N_m) - Y_m) \\
&\phantom{=}\times \sum_{i=1}^N \sqbracket{\hat{S}^{i,N}_{m+1} h(W_m^i) + (C^{i}_m)^2 \sigma'((W_m^i)^\top X_m + v^N_m) \nabla h(W_m^i)^\top X_m}.
\end{align*}
\item In Subsection \ref{SS:replacing_memory}, we replace the trained parameters with their initialisations. If we define
\begin{align*}
\delta^{(2)} g^N_m(h) &= - \frac{\alpha}{N^2} (\psi^N(g^N_m(h_{m+1})) - Y_m) \\
&\phantom{=}\times \sum_{i=1}^N \sqbracket{\hat{S}^{i,N}_{m+1} h(W_m^i) + (C^{i}_{m})^2 \sigma'((W_m^i)^\top X_m + v^N_m) \nabla h(W_m^i)^\top X_{m}}, \\
\delta^{(3)} g^N_m(h)
&= - \frac{\alpha}{N} (\psi^N(g^N_m(h_{m+1})) - Y_m) \\
&\phantom{=}\times \int \Big[\sigma(\sw^\top X_m + u_m) h(\sw) + \mu_{c^2} \sigma' \bracket{\sw^\top X_m + u_m} \nabla h(\sw)^\top X_m \Big] \, \lambda(d\sw),
\end{align*}
then for any $m \leq NT$,
\begin{equation*}
\E\abs{\delta^{(1)} g^N_m(h) - \delta^{(2)} g^N_m(h)}^2 \vee \E\abs{\delta^{(2)} g^N_m(h) - \delta^{(3)} g^N_m(h)}^2 \leq \frac{C_T \|h\|^2_{C^2}}{N^{4-2\beta-4\gamma}}.
\end{equation*}
\item In Subsection \ref{SS:remove_clip}, we remove the clipping function $\psi^N(\cdot)$ in $\delta^{(3)} g^N_k(h)$ by noticing the initial value $g^N_0(h)$ is $L^2$ integrable for all $h \in C^2_b(\R^d)$.
\item In Subsection \ref{SS:weak_convergence_analysis}, we establish the weak convergence by analysing the corresponding Poisson equation of the Markov chain $(H_k)_{k\geq 0}$ as defined in \cite{meyertweedie}.
\item Finally, in Subsection \ref{SS:ProofMainTheorem}, we complete the proof of the main result of this paper, Theorem \ref{thm:weak_convergence_thm}.
\end{enumerate}

\subsection{Well-posedness of the Limit ODE}
\label{SS:well-posedness}
The limit ODE \eqref{eq:KernelLimitEqn} should be seen as an affine ODE on the dual space $(C^2_b(\R^d))^*$ (which is a Banach space), written as
\begin{equation}
    \frac{d}{dt} g_t = \cA(g_t) + b, \quad g_0 = 0,\nonumber
\end{equation}
where $\cA$ is a linear operator from $(C^2_b(\R^d))^*$ to $(C^2_b(\R^d))^*$ as defined below,
\begin{equation}
\mathcal{A}: g \in (C^2_b(\R^d))^* \mapsto \sqbracket{\mathcal{A}(g): h \mapsto - \alpha \int_{\X} g(\varsigma_{\sx,\su}) \K_{\lambda}(h,\varsigma_{\sx,\su})\, \mu(d\sH)};\nonumber
\end{equation}
and $b$ is the following linear functional:
\begin{equation}
    b: h \mapsto \int_{\X} \sy \K_{\lambda}(h,\varsigma_{\sx,\su}) \, \mu(d\sH)\nonumber
\end{equation}
Therefore the ODE admits a unique solution if $\cA$ is bounded. Indeed, for $h,\varsigma \in C^2_b(\R^d)$,
\begin{equation} \label{eq:norm_for_k_lambda}
|\K_\lambda(h,\varsigma)| \leq \|h\|_\infty + \mu_{c^2} \|\nabla h ^\top \nabla \varsigma \|_\infty \leq (1+\mu_{c^2} d) \|h\|_{C^2} \|\varsigma\|_{C^2} = C \|h\|_{C^2} \|\varsigma \|_{C^2}.
\end{equation}
so
\begin{equation} \label{eq:norm_for_k_lambda_with_varsigma}
|\mathcal{K}_{\lambda}(h,\varsigma_{\sx,\su})| \lesssim  \|\varsigma_{\sx,\su}\|_{C^2} \|h\|_{C^2} \lesssim \max(C_\sigma, C_x C_{\sigma'}, C_x^2 C_{\sigma''}) \|h\|_{C^2} = C\|h\|_{C^2}.
\end{equation}
Therefore
\begin{align*}
|[\mathcal{A}(g)](h)| &\leq \alpha \int_{\X} |g(\varsigma_{\sx,\su})| |\K_{\lambda}(h,\varsigma_{\sx,\su})| \, \mu(d\sH) \\
&\leq \alpha \int_{\X} \|g\|_{(C^2)^*} \|\varsigma_{\sx,\su}\|^2_{C^2_b} \|h\|_{C^2_b}\, \mu(d\sH) \\
&\leq C \|g\|_{(C^2)^*} \|h\|_{C^2}. \numberthis \label{eq:bdd_of_A}
\end{align*}
Therefore
\begin{equation}
\norm{\mathcal{A}(g)}_{(C^2_b(\R^d))^*} \leq C \norm{g}_{(C^2_b(\R^d))^*} \implies \norm{\mathcal{A}} \leq C. \nonumber
\end{equation}
Moreover, $b$ is bounded the following sense:
\begin{equation}
|b(h)| \leq \int_{\X} |\sy| |\K_{\lambda}(h,\varsigma_{\sx,\su})| \, \mu(d\sH) \leq C \norm{h}_{C^2_b} \implies \norm{b}_{(C^2_b(\lambda))^*} \leq C. \label{eq:bdd_of_b}
\end{equation}
Therefore, the following exponential operator is well defined for any $T > 0$:
\begin{equation}
    \exp(t\mathcal{A}) = \sum_{i=0}^\infty \frac{t^i \mathcal{A}^{\circ i}}{i!}, \quad \mathcal{A}^{\circ i} = \underbrace{\cA \circ ... \circ \cA}_{i \text{ times}}, \nonumber
\end{equation}
and therefore
\begin{proposition} \label{prop:well_posedness_ode}
The ODE \eqref{eq:KernelLimitEqn} admits the following classical solution:
\begin{equation} \label{eq:KernelLimitSln}
    g_t = \int_0^t \exp((t-s)\mathcal{A}) b \, ds,
\end{equation}
and when acting on $h \in H^1(\lambda)$ we have
\begin{equation}
    g_t(h) = \int_0^t \exp((t-s)\mathcal{A}) b(h) \, ds. \nonumber
\end{equation}
\end{proposition}
In particular, for any $T>0$, there exists $C_T > 0$ such that for any $t \leq T$
\begin{equation}
\|g_t\|_{(C^2_b(\R^d))^*} \leq T\exp(T\|\cA\|) \|b\|_{(C^2_b(\R^d))^*} \leq C_T, \label{eq:operator_norm_control_of_gt}
\end{equation}
where $\|\cA\|$ is the operator norm of the operator $\cA$, which is proven to be bounded by \eqref{eq:bdd_of_A}. So for all $h$ we have $|g_t(h)| \leq C_T \|h\|_{C^2}$ for any finite $T > 0$.

\begin{proof}
We follow \cite{PazyA1983} for our discussion. Firstly, \cite[Theorems 1.2-3]{PazyA1983} state that $\mathcal{A}$ induces a unique uniformly continuous semigroup $\exp(t\mathcal{A})$. Indeed, consider the sequence of operators from $(C^2_b(\R^d))^*$ to $(C^2_b(\R^d))^*$:
\begin{equation}
S^N(t;\mathcal{A}) = \sum_{i=0}^N \frac{t^i \mathcal{A}^{\circ i}}{i!}\nonumber
\end{equation}
By triangle inequality we have the operator norm control $\|S^N(t;\mathcal{A})\| \leq \exp(t\|\cA\|) < +\infty$, uniform in $N$. Therefore, the partial sums $(S^N(t,\mathcal{A}))_{N\geq 0}$ must converge absolutely. Since the space of operators from $(C^2_b(\R^d))^*$ to $(C^2_b(\R^d))^*$ is Banach, the partial sum $(S^N(t,\mathcal{A}))_{N\geq 0}$ must also converge to an operator in operator norm, for which we define it as $\exp(t\mathcal{A})$.

Since $b \in (C^2_b(\R^d))^*$ is constant in $t$, then by \cite[Corollary 2.5]{PazyA1983} the ODE \ref{eq:KernelLimitEqn} admits a classical solution, given by the formula \eqref{eq:KernelLimitSln} according to \cite[Corollary 2.2]{PazyA1983}.
\end{proof}

\subsection{Pre-Limit Evolution}
\label{SS:proof_of_KernelEvolution1}

Using a Taylor expansion, we derive the following evolution equation for $g^N_m(h)$ for a fixed test function $h$:
\begin{lemma} \label{lem:KernelEvolution1}
For all $h \in C^2_b(\R^d)$, let
\begin{align*}
\triangle g^N_m(h) &= g^N_{m+1}(h) - g^N_m(h), \nonumber \\
\delta^{(1)} g^N_m(h)
&= - \frac{\alpha}{N^2} (\psi^N(\hat{Y}^N_m) - Y_m) \sum_{i=1}^N \sqbracket{\hat{S}^{i,N}_{m+1} h(W_m^i) + (C^{i}_{m})^2 \Delta \hat{S}^{i}_{m+1} \nabla h(W_m^i)^\top X_{m}} \\
&= - \frac{\alpha}{N^2} (\psi^N(\hat{Y}^N_m) - Y_m) \sum_{i=1}^N \big[\sigma((W_m^i)^\top X_m + v^N_k) h(W_m^i) \\
&\phantom{=}+ (C^{i}_{m})^2 \sigma'((W_m^i)^\top X_m + v^N_m) \nabla h(W_m^i)^\top X_{m} \big]. \numberthis \label{eq:KernelEvolution1}
\end{align*}
Then for all $m \leq \floor{NT}$, there are constant(s) $C_T > 0$
\begin{equation}
\abs{\triangle g^N_m(h) - \delta^{(1)} g^N_m(h)} \leq \frac{C_T}{N^{3-\beta-2\gamma}} \|h\|_{C^2}, \label{eq:lemma_61_1}
\end{equation}
and for any $m, \floor{NT}$ and $k \geq 0$,
\begin{align}
\abs{\triangle g^N_m(\varsigma_{X_k, v^N_k}) - \delta^{(1)} g^N_m(\varsigma_{X_k,u^N_k})} &\leq \frac{C_T}{N^{3-\beta-2\gamma}}, \label{eq:lemma_61_2}\\
\abs{\triangle g^N_m(h_{k+1}) - \delta^{(1)} g^N_m(h_{k+1})} &\leq \frac{C_T}{N^{3-\beta-2\gamma}}. \label{eq:lemma_61_3}
\end{align}
\end{lemma}

\begin{proof}
We first establish \eqref{eq:lemma_61_1}. We expand $h(W^i_{m+1}) - h(W^i_m)$ using Taylor series as followed:
\begin{align*}
h(W^i_{m+1}) - h(W^i_m) &= \nabla h(W^{i,*}_m)^\top (W^i_{m+1} - W^i_{m}) \\
&= \nabla h(W^{i}_m)^\top (W^i_{m+1} - W^i_{m}) \\
&\phantom{=}+ (W^i_{m+1} - W^i_{m})^\top \mathsf{Hess}\, h(W^{i,**}_m) (W^i_{m+1} - W^i_{m}), \nonumber
\end{align*}
where $W^{i,*}_m, W^{i,**}_m$ are points in the line segment connecting the points $W^i_m$ and $W^i_{m+1}$. The following remainder terms in the Taylor's expansion is small by Lemma \ref{lem:evolution_of_parameters} (specifically by \eqref{Eq:IncrementDiff}):
\begin{align*}
\abs{(C^i_{m+1}-C^i_m) \nabla h(W^{i,*}_m)^\top (W^i_{m+1} - W^i_m)} &\leq \frac{C_T}{N^{4-2\beta-2\gamma}} \|h\|_{C^2}, \\
\abs{C^i_m (W^i_{m+1}-W^i_m) \mathsf{Hess} \, h(W^{i,*}_m) (W^i_{m+1} - W^i_m)} &\leq {C_T} \abs{W^i_{m+1}-W^i_m}^2 d^2 \|h\|_{C^2} \\
&\leq \frac{C_T}{N^{4-2\beta-2\gamma}} \|h\|_{C^2}.
\end{align*}
Therefore,
\begin{align*}
\triangle g^N_m(h) &= g_{m+1}^N(h) - g_{m}^N(h) \\
&= \frac{1}{N^\beta}\sum_{i=1}^N (C_{m+1}^i h(W^i_{m+1}) - C_m^i h(W_m^i)) \\
&= \frac{1}{N^\beta} \sum_{i=1}^N [(C_{m+1}^i - C_m^i) h(W^i_m) + C^i_m (h(W_{m+1}^i) - h(W_m^i)) \\
&\phantom{=}+ (C_{m+1}^i - C_m^i) (h(W_{m+1}^i) - h(W_m^i))] \\
&= \frac{1}{N^\beta} \sum_{i=1}^N [(C_{m+1}^i - C_m^i) h(W_m^i) + C^i_m \nabla h(W^i_m) (W^i_{m+1} - W^i_k) \\
&\phantom{=}+ C^i_m (W^i_{m+1}-W^i_m) \mathsf{Hess} \, h(W^{i,**}_m) (W^i_{m+1} - W^i_m) \\
&\phantom{=}+ (C_{m+1}^i - C_m^i) \nabla h(W^{i,*}_k)^\top (W^i_{m+1} - W^i_m)]. \nonumber
\end{align*}
Notice that with the specific choice of the learning rate $\alpha^N = \alpha / N^{2-2\beta}$ by (\ref{Eq:LR}) we get after training is taken into account
\begin{align}
&\phantom{=}\frac{1}{N^\beta} \sum_{i=1}^N \Big[(C_{m+1}^i - C_m^i) h(W_m^i) + C_m^i \nabla h(W_m^i)^\top (W_{m+1}^i - W_m^i) \Big] \nonumber \\
&= - \frac{\alpha}{N^2} (\psi^N(\hat{Y}^N_m)-Y_m)  \sum_{i=1}^N  \sqbracket{\hat{S}^{i}_{m+1} h(W_m^i) + C_m^i \Delta \hat{S}^{i,N}_{m+1} \nabla h(W_m^i)^\top X_m} \\
&= \delta^{(1)} g^N_m(h). \nonumber
\end{align}
Therefore, for all $k \leq \floor{NT}$,
\begin{align*}
&\phantom{=}\abs{\triangle g^N_m(h) - \delta^{(1)} g^N_m(h)} \\
&= \Bigg|\frac{1}{N^\beta} \sum_{i=1}^N \Big[C^i_m (W^i_{m+1}-W^i_m) \mathsf{Hess} \, h(W^{i,**}_m) (W^i_{m+1} - W^i_m) \\
&\phantom{=}+ (C_{m+1}^i - C_m^i) \nabla h(W^{i,*}_m)^\top (W^i_{m+1} - W^i_m) \Big] \Bigg| \\
&\leq \frac{1}{N^\beta} \sum_{i=1}^N \Big[\abs{C^i_m (W^i_{m+1}-W^i_m) \mathsf{Hess} \, h(W^{i,**}_m) (W^i_{m+1} - W^i_m)} \\
&\phantom{=}+ \abs{(C_{m+1}^i - C_m^i) \nabla h(W^{i,*}_m)^\top (W^i_{m+1} - W^i_m)} \Big] \\
&\leq \frac{N}{N^\beta} \times \frac{C_T}{N^{4-2\beta-2\gamma}} \|h\|_{C^2} \\
&= \frac{C_T}{N^{3-\beta-2\gamma}} \|h\|_{C^2}. \nonumber
\end{align*}
We note that the above computations are also valid for the case when $h$ is replaced by $\varsigma_{X_k, v^N_k}(\cdot)$ or $h_{k+1}(\cdot)$. For example,
\begin{align*}
&\phantom{=} \triangle g^N_m(\varsigma_{X_k,v^N_k}) \\
&= \frac{1}{N^\beta} \sum_{i=1}^N [C^i_{m+1} \sigma((W^i_{m+1})^\top X_k + v^N_k) - C^i_m\sigma((W^i_m)^\top X_k + v^N_k)] \\
&= \frac{1}{N^\beta} \sum_{i=1}^N \Big[(C^i_{m+1} -C^i_m) \sigma((W^i_{m+1})^\top X_k + v^N_k) \\
&\phantom{=}+ C^i_m \big(\sigma((W^i_{m+1})^\top X_k + v^N_k) - \sigma((W^i_{m})^\top X_k + v^N_k) \big) \\
&\phantom{=}+ (C^i_{m+1} - C^i_m) \big(\sigma((W^i_{m+1})^\top X_k + v^N_k) - \sigma((W^i_{m})^\top X_k + v^N_k) \big) \Big] \\
&= \frac{1}{N^\beta} \sum_{i=1}^N \Big[(C^i_{m+1} -C^i_m) \sigma((W^i_{m+1})^\top X_k + v^N_k) \\
&\phantom{=}+ C^i_m \sigma'((W^i_m)^\top X_k + v^N_k) (W^i_{m+1}-W^i_m)^\top X_k \\
&\phantom{=}+ C^i_m \sigma''((W^{i,**}_m)^\top X_k + v^N_k) |(W^i_{m+1}-W^i_m)^\top X_k|^2 \\
&\phantom{=}+ (C^i_{m+1}-C^i_m) \sigma'((W^{i,*}_m)^\top X_k + v^N_k) (W^i_{m+1}-W^i_m)^\top X_k \Big] \\
&= \delta^{(1)} g^N_m(\varsigma_{X_k,v^N_k}) + \frac{1}{N^\beta} \sum_{i=1}^N \Big[ C^i_m \sigma''((W^{i,**}_m)^\top X_k + v^N_k) |(W^i_{m+1}-W^i_m)^\top X_k|^2 \\
&\phantom{=}+ (C^i_{m+1}-C^i_m) \sigma'((W^{i,*}_m)^\top X_k + v^N_k) (W^i_{m+1}-W^i_m)^\top X_k \Big].
\end{align*}
Therefore, \eqref{eq:lemma_61_1} applies to both the cases when $h$ is replaced by $\varsigma_{X_k,v^N_k}(\cdot)$ or $h_{k+1}$. Note that $\|\varsigma_{X_k,v^N_k}\|_{C^2} \vee \|h_{k+1}\|_{C^2}$ is bounded, so \eqref{eq:lemma_61_2} and \eqref{eq:lemma_61_3} follow.
\end{proof}

As $3-\beta-2\gamma = 2 + 1 - \beta - 2\gamma > 2$, so $\abs{\triangle g^N_m(h) - \delta^{(1)} g^N_m(h)} = o(N^{-2})$ uniformly for $m \leq \floor{NT}$.

\subsection{Replacing the trained memory units}
\label{SS:replacing_memory}

We would like to show that one could study a simpler increment for the evolution of $g^N_k(h)$, in light of the a priori bounds for the increments of parameters (Lemma \ref{lem:evolution_of_parameters}) and our current analysis of the trained hidden memory units in Subsection \ref{SS:dynamics_of_sample_memories}. We begin by showing that $\hat{Y}_k$ could be replaced by $g^N_k(h_{k+1})$:

\begin{lemma} \label{lem:replace_vN_with_h_1}
For $k \leq NT$,
\begin{equation}
\E\abs{\hat{Y}_k - g^N_k(h_{k+1})}^2 \leq \frac{C_T}{N^{2-2\beta-4\gamma}}.
\end{equation}
\end{lemma}

\begin{proof}
Recall that $\hat{Y}_k = g^N_k(\varsigma_{X_k,v^N_k})$, and
\begin{align*}
&\phantom{=} \E\abs{g^N_k(\varsigma_{X_k,v^N_k}) - g^N_k(h_{k+1})}^2 \\
&=\E\abs{\sum_{m=1}^k \triangle g^N_m(\varsigma_{X_k,v^N_k}) - \triangle g^N_m(h_{k+1})}^2 \\
&\leq k\E\sum_{m=1}^k \abs{\triangle g^N_m(\varsigma_{X_k,v^N_k}) - \triangle g^N_m(h_{k+1})}^2 \\
&\leq k\E\sum_{m=1}^k \big| \triangle g^N_m(\varsigma_{X_k,v^N_k}) - \delta^{(1)} g^N_m(\varsigma_{X_k,v^N_k}) \numberthis \label{eq:gNkv_gNkh_breakdown} \\
&\phantom{=}+ \delta^{(1)} g^N_m(\varsigma_{X_k,v^N_k}) - \delta^{(1)} g^N_k(h_{k+1}) + \delta^{(1)} g^N_k(h_{k+1}) - \triangle g^N_k(h_{k+1}) \big|^2 \\
&\leq 3k \sum_{m=1}^k \Big[\E\abs{\triangle g^N_m(v^N_{k+1}) - \delta^{(1)} g^N_m(\varsigma_{X_k,v^N_k})}^2 \\
&\phantom{=}+ \E\abs{\delta^{(1)} g^N_m(\varsigma_{X_k,v^N_k}) - \delta^{(1)} g^N_m(h_{k+1})}^2 + \E\abs{\delta^{(1)} g^N_m(h_{k+1}) - \triangle g^N_m(h_{k+1})}^2 \Big].
\end{align*}
By Lemma \ref{lem:KernelEvolution1},
\begin{align*}
\abs{\triangle g^N_m(\varsigma_{X_k,v^N_k}) - \delta^{(1)} g^N_m(\varsigma_{X_k,v^N_k})}^2 \vee \abs{\triangle g^N_m(h_{k+1}) - \delta^{(1)} g^N_m(h_{k+1})}^2  &\leq \frac{C_T}{N^{6-2\beta-4\gamma}}.
\end{align*}
For the middle term of \eqref{eq:gNkv_gNkh_breakdown}, we have
\begin{align*}
&\phantom{=}\E\abs{\delta^{(1)} g^N_m(\varsigma_{X_k,v^N_k}) - \delta^{(1)} g^N_m(h_{k+1})}^2 \\
&\leq \frac{\alpha^2}{N^4} \E\Bigg[\underbrace{\abs{\psi^N(\hat{Y}^N_k) - Y_k}^2}_{\lesssim  N^{2\gamma}} \bigg[\sum_{i=1}^N \Big[\hat{S}^{i,N}_{m+1} \big[\sigma((W^i_m)^\top X_k + v^N_k) - \sigma((W^i_m)^\top X_k + u_k)\big] \\
&\phantom{=}+ (C^i_m)^2 \sigma'((W^i_m)^\top X_m+v^N_m) \big[\sigma'((W^i_m)^\top X_k + v^N_k) \\
&\phantom{====}- \sigma'((W^i_m)^\top X_k + u_k)\big] X_k^\top X_m \Big]\bigg]^2\Bigg] \\
&\overset{\mathsf{(CS)}}\lesssim \frac{1}{N^{3-2\gamma}} \E\Bigg[\sum_{i=1}^N \bigg[\underbrace{\Big[\hat{S}^{i,N}_{m+1}\Big]^2}_{\leq C_\sigma} \big[\sigma((W^i_m)^\top X_k + v^N_k) - \sigma((W^i_m)^\top X_k + u_k)\big]^2 \\
&\phantom{=}+ \underbrace{(C^i_m)^4}_{\leq C_T} \underbrace{(\sigma'((W^i_m)^\top X_m + v^N_m))^2}_{\leq C^2_{\sigma'}}  \big[\sigma'((W^i_m)^\top X_k + v^N_k) \\
&\phantom{====}- \sigma'((W^i_m)^\top X_k + u_k)\big]^2 \underbrace{|X_k^\top X_m|^2}_{\leq C_x} \bigg]\Bigg] \\
&\leq \frac{C_T}{N^{2-2\gamma}} \E|v^N_k - u_k|^2 \\
&\leq \frac{C_T}{N^{2-2\gamma}} \Big[\E|v^N_k - u^N_k|^2 + \E|u^N_k - u_k|^2 \Big] \\
&\overset{(*)}\leq \frac{C_T}{N^{2-2\gamma}} \sqbracket{\frac{C_T}{N^{2-2\beta-2\gamma}} + \frac{1}{N}},
\end{align*}
where $(*)$ is from lemma \ref{lem:diff_vNk_hk} and the finite constant $C_{T}$ may change from line to line.

Since $\beta \in (1/2,1)$ and $\gamma \in (0,(1-\beta)/2)$, we have
$$0 < (1-\beta) + (1-\beta-2\gamma) = 2-2\beta-2\gamma < 1-2\gamma < 1,$$
so $1/N = o(N^{-(2-2\beta-2\gamma)})$. Therefore,
\begin{align}
\E\abs{g^N_k(\varsigma_{X_k,v^N_k}) - g^N_k(h_{k+1})}^2
&\leq 3k \sum_{m=1}^k \Big[2 \times \frac{C_T}{N^{6-2\beta-4\gamma}} + \frac{C_T}{N^{2-2\gamma}} \frac{C_T}{N^{2-2\beta-2\gamma}}\Big] \nonumber \\
&\leq \frac{C_T}{N^{2-2\beta-4\gamma}}.
\end{align}
\end{proof}

\begin{lemma}
Define
\begin{align}
\delta^{(2)} g^N_m(h) &= - \frac{\alpha}{N^2} (\psi^N(g^N_m(h_{m+1})) - Y_m) \nonumber \\
&\phantom{=}\times \sum_{i=1}^N \sqbracket{\hat{S}^{i,N}_{m+1} h(W_m^i) + (C^{i}_m)^2 \Delta \hat{S}^{i}_{m+1} \nabla h(W_m^i)^\top X_m},
\end{align}
then for all $h \in C^2_b(\R^d)$, the following holds for all $m \leq \floor{NT} - 1$:
\begin{equation}
\E\abs{\delta^{(1)} g^N_m(h) - \delta^{(2)} g^N_m(h)}^2 \leq \frac{C_T  \|h\|^2_{C^2}}{N^{4-2\beta-4\gamma}}.
\end{equation}
In particular, for any $m \leq \floor{NT} - 1$ and $k \geq 0$,
\begin{equation}
\E\abs{\delta^{(1)} g^N_m(h_{k+1}) - \delta^{(2)} g^N_m(h_{k+1})}^2 \leq \frac{C_T}{N^{4-2\beta-4\gamma}}. \label{eq:lemma_63_0}
\end{equation}
\end{lemma}

\begin{proof}
We first note the uniform bound in $i$ for $m \leq \floor{NT}$:  thanks to the boundedness of the test function $h$, boundedness of activation function (see Assumption \ref{as:activation_function} and equations \eqref{eq:application_activation_1}-\eqref{eq:application_activation_2}) and the boundedness of $C^i_k$ (see Assumption \ref{as:ergodicity_conditions} and Lemma \ref{lem:evolution_of_parameters}), one has
\begin{align}
\abs{\hat{S}^{i,N}_{m+1} h(W^i_m) + (C^i_m)^2 \Delta \hat{S}^{i,N}_{m+1} \nabla h(W^i_m)^\top X_m} &\lesssim \|h\|_{\infty} + C_T d^2 \max_i \left\| \frac{\partial h}{\partial w^i} \right\|_\infty \nonumber \\
&\leq C_T \|h\|_{C^2}, \label{eq:lemma_63_1}
\end{align}
By the boundedness of derivative of $\psi^N$ and that $\hat{Y}^N_k = g^N_k(\varsigma_{X_k,v^N_k})$,
\begin{align*}
\abs{\delta^{(1)} g^N_m(h) - \delta^{(2)} g^N_m(h)}
&\leq \frac{C_T \|h\|^2_{C^2}}{N}\abs{\psi^N(\hat{Y}^N_m) - \psi^N(g^N_m(h_{m+1}))} \\
&\leq \frac{C_T \|h\|^2_{C^2}}{N} \abs{g^N_m(\varsigma_{X_m,v^N_m}) - g^N_m(h_{m+1})}. \numberthis
\label{eq:lemma_63_2}
\end{align*}
Thus, by Lemma \ref{lem:replace_vN_with_h_1},
\begin{equation}
\E\abs{\delta^{(1)} g^N_m(h) - \delta^{(2)} g^N_m(h)}^2 \leq \frac{C_T \|h\|^2_{C^2}}{N^2} \E\sqbracket{g^N_m(\varsigma_{X_m,v^N_m}) - g^N_m(h_{m+1})}^2 \leq \frac{C_T \|h\|^2_{C^2}}{N^{4-2\beta-4\gamma}} \label{eq:lemma_63_3}
\end{equation}
Note that \eqref{eq:lemma_63_1} and \eqref{eq:lemma_63_2} remains valid when $h$ is replaced with $h_{k+1}$, and $\|h_{k+1}\| \leq C_\sigma$, so \eqref{eq:lemma_63_2} holds, which yields \eqref{eq:lemma_63_0}.
\end{proof}
For the next step, we revisit the formula for $\delta^{(2)}g^N_m(h)$:
\begin{align*}
\delta^{(2)} g^N_m(h) &= - \frac{\alpha}{N^2} (\psi^N(g^N_m(h_{m+1})) - Y_m) \\
&\phantom{=}\times \sum_{i=1}^N \sqbracket{\hat{S}^{i,N}_{m+1} h(W_k^i) + (C^{i}_{m})^2 \Delta \hat{S}^{i}_{m+1} \nabla h(W_m^i)^\top X_m} \\
&= - \frac{\alpha}{N^2} (\psi^N(g^N_k(h_{m+1})) - Y_m) \sum_{i=1}^N \Big[\sigma((W^i_m)^\top X_m + v^N_m) h(W_m^i) \\
&\phantom{=}+ (C^{i}_m)^2 \sigma'\bracket{(W^i_m)^\top X_m + v^N_m} \nabla h(W_m^i)^\top X_m \Big].
\end{align*}
In light of the bounds of the increments of parameters (Lemma \ref{lem:evolution_of_parameters}) and the trained memory units (Subsection \ref{SS:dynamics_of_sample_memories}), we would expect that one could study the new increments with the average of empirical distribution of updated parameters $\lambda^N_k$ replaced by an average with respect to the initial parameter distribution $\lambda$. This is formalised by the following lemma:
\begin{lemma}
Define
\begin{align*}
\delta^{(3)} g^N_m(h)
&= - \frac{\alpha}{N} (\psi^N(g^N_m(h_{m+1})) - Y_m) \\
&\phantom{=}\times \int_{\mathcal{W}} \Big[\sigma(\sw^\top X_m + u_m) h(\sw) + \mu_{c^2} \sigma' \bracket{\sw^\top X_m + u_m} \nabla h(\sw)^\top X_m \Big] \, \lambda(d\sw)
\end{align*}
then for all $h \in C^2_b(\R^d)$, the following holds for all $m \leq \floor{NT}-1$
\begin{equation}
    \E\abs{\delta^{(2)} g^N_m(h) - \delta^{(3)} g^N_m(h)}^2 \leq \frac{C_T \|h\|^2_{C^2}}{N^{4-2\beta-4\gamma}}.
\end{equation}
Moreover, for all $m \leq \floor{NT}-1$ and $k \geq 0$,
\begin{equation}
\E\abs{\delta^{(2)} g^N_m(h_{k+1}) - \delta^{(3)} g^N_m(h_{k+1})}^2 \leq \frac{C_T}{N^{4-2\beta-4\gamma}}.
\end{equation}
\end{lemma}

\begin{proof}
We break down the fluctuation term into different components
\begin{equation}
\delta^{(2)} g^N_m(h) - \delta^{(3)} g^N_m(h) = - \frac{\alpha}{N}(\psi^N(g^N_m(h_{k+1})) - Y_m) \sum_{l=1}^6 M^{l,N}_m(h),\nonumber
\end{equation}
where
\begin{align*}
M^{1,N}_m(h) &= \frac{1}{N} \sum_{i=1}^N [\sigma((W^i_m)^\top X_m + v^N_m) h(W^i_m) - \sigma((W^i_0)^\top X_m + v^N_m) h(W^i_0)], \\
M^{2,N}_m(h) &= \frac{1}{N} \sum_{i=1}^N [(C^i_k)^2 \sigma'((W^i_m)^\top X_m + v^N_m) \nabla h(W^i_m)^\top X_m \\
&\phantom{=}- (C^i_0)^2 \sigma'((W^i_0)^\top X_m + v^N_m) \nabla h(W^i_0)^\top X_m], \\
M^{3,N}_m(h) &= \frac{1}{N} \sum_{i=1}^N [\sigma((W^i_0)^\top X_m + v^N_m) - \sigma((W^i_0)^\top X_m + u_m)] h(W^i_0), \\
M^{4,N}_m(h) &= \frac{1}{N} \sum_{i=1}^N (C^i_0)^2 (\sigma'((W^i_0)^\top X_m + v^N_m) - \sigma'((W^i_0)^\top X_m + u_m) \nabla h(W^i_0)^\top X_m, \\
M^{5,N}_m(h) &= \frac{1}{N} \sum_{i=1}^N \sigma((W^i_0)^\top X_m + u_m) h(W^i_0) - \int_{\mathcal{W}} \sigma(\sw^\top X_m + u_m) h(\sw) \, \lambda(d\sw),\\
M^{6,N}_m(h) &= \frac{1}{N} \sum_{i=1}^N (C^i_0)^2 \sigma'\bracket{(W^i_0)^\top X_m + u_m} \nabla h(W^i_0)^\top X_m \\
&\phantom{=}- \mu_{c^2} \int_{\mathcal{W}} \sigma'\bracket{\sw^\top X_m + u_k} \nabla h(\sw)^\top X_m \, \lambda(d\sw),
\end{align*}
then
\begin{align*}
\E\abs{\delta^{(2)} g^N_k(h) - \delta^{(3)} g^N_k(h)}^2
&\leq \frac{\alpha^2}{N^2}\E\sqbracket{\abs{\psi^N(g^N_k(h_{k+1}))-Y_k}^2 \bracket{\sum_{\bullet=1}^6 M^{\bullet,N}_k}^2} \\
&\lesssim \frac{1}{N^{2-2\gamma}} \bracket{\sum_{\bullet=1}^6 \E\sqbracket{M^{\bullet,N}_k}^2}.
\end{align*}

For the first term,
\begin{align*}
|M^{1,N}_m(h)| &\leq \frac{1}{N} \sum_{i=1}^N |\sigma((W^i_m)^\top X_m + v^N_m) h(W^i_m) - \sigma((W^i_0)^\top X_m + v^N_m) h(W^i_0)| \\
&\leq \frac{1}{N} \sum_{i=1}^N [|\sigma((W^i_m)^\top X_m + v^N_m)| |h(W^i_m) - h(W^i_0))| \\
&\phantom{=}+ |\sigma((W^i_m)^\top X_m + v^N_m) - \sigma((W^i_0)^\top X_m + v^N_m)||h(W^i_0)|] \\
&\leq \frac{1}{N} \sum_{i=1}^N \left(d \left[\max_i \left\| \frac{\partial h}{\partial w^i}\right\|_\infty\right] |W^i_m - W^i_0| + C_\sigma \|h\|_\infty |(W^i_m - W^i_0)^\top X_k|\right) \\
&\lesssim \frac{\|h\|_{C^2}}{N} \sum_{i=1}^N \|W^i_m - W^i_0\| \overset{\eqref{eq:evolution_of_parameters}}\leq \frac{C_T \|h\|_{C^2}}{N^{1-\beta-\gamma}} \numberthis
\end{align*}
Notice that the above computations are valid when $h$ is replaced with $h_{k+1}$, so we have
\begin{align*}
|M^{1,N}_k(h_{k+1})| \leq \frac{C_T}{N^{1-\beta-\gamma}}.
\end{align*}

For the second term,
\begin{align}
\abs{M^{2,N}_m(h)}
&\leq \frac{1}{N} \sum_{i=1}^N \Big|(C^i_m)^2 \sigma'((W^i_m)^\top X_m + v^N_m) \nabla h(W^i_m)^\top X_m \nonumber \\
&\phantom{=}- (C^i_0)^2 \sigma'((W^i_0)^\top X_m + v^N_m)  \nabla h(W^i_0)^\top X_m \Big| \nonumber \\
&\leq \frac{1}{N} \sum_{i=1}^N \bigg[|(C^i_m)^2 - (C^i_0)^2| |\sigma'((W^i_m)^\top X_m + v^N_m)| \, \underbrace{|\nabla h(W^i_m)^\top X_m)\big|}_{\leq d \|h\|_{C^2}} \nonumber \\
&\phantom{=}+ \underbrace{(C^i_0)^2}_{\leq K_c^2} | \sigma'((W^i_m)^\top X_m + v^N_m) - \sigma'((W^i_0)^\top X_m + v^N_m)| |\nabla h(W^i_m)^\top X_m| \\
&\phantom{=}+ \underbrace{(C^i_0)^2}_{\leq K_c^2} |\sigma'((W^i_0)^\top X_m + v^N_m)| |(\nabla h(W^i_m) - \nabla h(W^i_0))^\top X_m| \bigg] \nonumber \\
&\leq \frac{1}{N} \sum_{i=1}^N \Big[C_\sigma d^2 \|h\|_{C^2} \abs{C^i_m - C^i_0} \underbrace{\abs{C^i_m + C^i_0}}_{\lesssim 1+T} + C_\sigma d^2 |W^i_m - W^i_0| |X_m|^2 \|h\|_{C^2} \nonumber \\
&\phantom{=}+ C_\sigma \underbrace{|\nabla h(W^i_m) - \nabla h(W^i_0)|}_{\leq d |W^i_m - W^i_0| \|h\|_{C^2}} \Big] \nonumber \\
&\leq \frac{C_T \|h\|_{C^2}}{N} \sum_{i=1}^N \sqbracket{|C^i_m - C^i_0| + |W^i_m - W^i_0|} \nonumber \\
&\overset{\eqref{eq:evolution_of_parameters}}\leq \frac{C_T \|h\|_{C^2}}{N^{1-\beta-\gamma}}. \nonumber
\end{align}
The above computations are valid when $h$ is replaced with $h_{k+1}$, so we have
\begin{align*}
|M^{2,N}_m(h_{k+1})| \leq \frac{C_T}{N^{1-\beta-\gamma}}.
\end{align*}

For the third term,
\begin{align}
|M^{3,N}_m(h)| \leq |v^N_m - u_m| \|h\|_\infty, \label{eq:6_10_M3p}
\end{align}
so by Lemma \ref{lem:diff_vNk_hk}, we have
\begin{align*}
\E|M^{3,N}_m(h)|^2
&\leq \|h\|^2_{C^2} \E|v^N_m-u_m|^2 \\
&\leq 2\|h\|^2_{C^2} [\E|v^N_m-u^N_m|^2 + \E|u^N_m-u_m|^2] \\
&\lesssim 2\|h\|^2_{C^2} \left[\frac{C_T}{N^{2-2\beta-2\gamma}} + \frac{1}{N} \right]. \numberthis \label{eq:6_10_M3}
\end{align*}
For the fourth term,
\begin{align*}
|M^{4,N}_m(h)|
&\leq \frac{1}{N} \sum_{i=1}^N (C^i_0)^2 |\sigma'((W^i_0)^\top X_m + v^N_m) - \sigma'((W^i_0)^\top X_m + u_m)| |\nabla h(W^i_0)^\top X_m| \\
&\leq dC_{\sigma^{''}} |v^N_m- h_m| \|h\|_{C^2}, \numberthis \label{eq:6_10_M4p}
\end{align*}
so, similar as \eqref{eq:6_10_M3}, we have
\begin{align}
\E|M^{4,N}_m(h)|^2
\leq 2 C^2_{\sigma^{''}} d^2 \|h\|^2_{C^2} \left[\frac{C_T}{N^{2-2\beta-2\gamma}} + \frac{1}{N} \right]. \label{eq:6_10_M4}
\end{align}
Note that \eqref{eq:6_10_M3p} and \eqref{eq:6_10_M4p} are valid when $h$ is replaced with $h_{k+1}$. In particular, we have
\begin{align*}
|M^{3,N}_m(h_{k+1})| \vee |M^{4,N}_m(h_{k+1})| \leq \max(1,dC_\sigma^{''}) |v^N_m - h_m|,
\end{align*}
so
\begin{align*}
\E|M^{3,N}_m(h_{k+1})|^2 \vee \E|M^{4,N}_m(h_{k+1})^2| \leq \frac{C_T}{N^{2-2\beta-2\gamma}} + \frac{1}{N}.
\end{align*}

The remaining two terms could be easily bounded using similar arguments in establishing equation \eqref{eq:evolution_ENk1}. For the fifth term, we note that $\sigma((W^i_0)^\top X_m + u_m) h(W^i_0)$ are i.i.d. when conditioned on $X_{0:m}$, with
\begin{align*}
\E[\sigma((W^i_0)^\top X_m + u_m) h(W^i_0) \mid X_{0:m}] = \int_{\mathcal{W}} \sigma(\sw^\top X_m + u_m) h(\sw) \, \lambda(d\sw),
\end{align*}
so
\begin{align*}
&\phantom{=}\E|M^{5,N}_m(h)|^2 \\
&= \E\Bigg[\E\bigg[\bigg[\frac{1}{N} \sum_{i=1}^N \sigma((W^i_0)^\top X_m + u_m) h(W^i_0) - \int_{\mathcal{W}} \sigma(\sw^\top X_m + u_m) h(\sw) \, \lambda(d\sw) \bigg]^2 \bigg| X_{0:m} \bigg] \Bigg] \\
&= \E\Bigg[\E\bigg[\frac{1}{N^2} \sum_{i_1, i_2=1}^N \bigg[\sigma((W^{i_1}_0)^\top X_m + u_m) h(W^{i_1}_0) - \int_{\mathcal{W}} \sigma(\sw^\top X_m + u_m) h(\sw) \, \lambda(d\sw)\bigg] \\
&\phantom{=}\times \bigg[\sigma((W^{i_2}_0)^\top X_m + u_m) h(W^{i_2}_0) - \int_{\mathcal{W}} \sigma(\sw^\top X_m + u_m) h(\sw) \, \lambda(d\sw)\bigg] \bigg| X_{0:m} \bigg] \Bigg] \\
&= \frac{1}{N^2} \sum_{i_1, i_2=1}^N \E\Bigg[ \E\bigg[\sigma((W^{i_1}_0)^\top X_m + u_m) h(W^{i_1}_0) - \int_{\mathcal{W}} \sigma(\sw^\top X_m + u_m) h(\sw) \, \lambda(d\sw) \bigg| X_{0:m} \bigg] \\
&\phantom{=}\times \E\bigg[\sigma((W^{i_2}_0)^\top X_m + u_m) h(W^{i_2}_0) - \int_{\mathcal{W}} \sigma(\sw^\top X_m + u_m) h(\sw) \, \lambda(d\sw) \bigg| X_{0:m} \bigg] \Bigg] \\
&= \frac{1}{N^2} \sum_{i=1}^N \E\Bigg[  \E\bigg[\sigma((W^{i}_0)^\top X_m + u_m) h(W^{i}_0) - \int_{\mathcal{W}} \sigma(\sw^\top X_m + u_m) h(\sw) \, \lambda(d\sw) \bigg| X_{0:m} \bigg] \Bigg]^2 \\
&= \frac{1}{N^2} \sum_{i=1}^N \E\Bigg[\E[(\sigma((W^i_0)^\top X_m + u_m))^2  (h(W^i_0))^2 | X_{0:m}] \\
&\phantom{=}- \bigg[\int_{\mathcal{W}} \sigma(\sw^\top X_m + u_m) h(\sw) \, \lambda(d\sw) \bigg]^2 \Bigg] \\
&\leq \frac{2}{N} \|h\|^2_{C^2}. \numberthis \label{eq:6_10_M5}
\end{align*}
As $(C^i_0)^2 \sigma'((W^i_0)^\top X_m + u_m) \nabla h(W^i_0)^\top X_m \,|\, X_{0:m}$ is conditionally mutually i.i.d., we can apply the exact same argument above to prove the following estimate for the sixth term:
\begin{align}
\E\abs{M^{6,N}_m(h)}^2 &\lesssim \frac{\|h\|^2_{C^2}}{N}. \label{eq:6_10_M6}
\end{align}
Again, the computations for \eqref{eq:6_10_M5} and \eqref{eq:6_10_M6} could be generalised when $h$ is replaced with $h_{k+1}$ - for example, it is true that $\sigma((W^i_0)^\top X_m + u_m) h_{k+1}(W^i_0) = \sigma((W^i_0)^\top X_m + u_m) \sigma((W^i_0)^\top X_k + u_k)$ are i.i.d. when conditioned on $X_{0:\max(k,m)}$, with
\begin{align*}
\E[\sigma((W^i_0)^\top X_m + u_m) h_{k+1}(W^i_0) \mid X_{0:\max(k,m)}] = \int_{\mathcal{W}} \sigma(\sw^\top X_m + u_m) h_{k+1}(\sw) \, \lambda(d\sw),
\end{align*}
so \eqref{eq:6_10_M5} remains valid if we condition on $X_{0:\max(k,m)}$ instead of $X_{0:m}$. Overall, we will get
\begin{align}
\E\abs{M^{5,N}_m(h_{k+1})}^2 \vee \E\abs{M^{6,N}_m(h_{k+1})}^2 &\lesssim \frac{1}{N}. \label{eq:6_10_M6evolution}
\end{align}

Finally, since $\beta > 1/2$ and $\gamma > 0$, we have $2-2\beta-2\gamma < 1 - 2\gamma < 1$, so $C_T N^{-1} = o(N^{-(2-2\beta-2\gamma)})$. Therefore, one could sum up all fluctuation terms to have
\begin{equation}
\E\abs{\delta^{(2)} g^N_k(h) - \delta^{(3)} g^N_k(h)}^2 \leq \frac{C}{N^{2-2\gamma}} \frac{C_T \|h\|^2_{C^2}}{N^{2-2\beta-2\gamma}} = \frac{C_T \|h\|^2_{C^2}}{N^{4-2\beta-4\gamma}},
\end{equation}
and
\begin{equation}
\E\abs{\delta^{(2)} g^N_k(h_{k+1}) - \delta^{(3)} g^N_k(h_{k+1})}^2 \lesssim \frac{C_T}{N^{4-2\beta-4\gamma}},
\end{equation}
\end{proof}

Let us take a step back and review what we have done so far. We have proven that:
\begin{itemize}
\item for a fixed $h \in C^2_b(\R^d)$, the increments $\triangle g^N_m(h)$ could be approximated by $\delta^{(l)} g^N_m(h)$ for $l=1, 2, 3$, and
\item the increments $\triangle g^N_m(h_k)$ could be (uniformly) approximated by $\delta^{(l)} g^N_m(h_k)$ for $l=1, 2, 3$.
\end{itemize}
These estimates could be combined to obtain an estimate of $g^N_t(h)$. In particular, for $t \leq T$,
\begin{align*}
&\phantom{=} \E\abs{g^N_t(h) - g^N_0(h) - \sum_{k=0}^{\floor{Nt} - 1} \delta^{(3)} g^N_k(h)}^2 \\
&=\E \abs{\sum_{m=0}^{\floor{Nt}-1} (\triangle g^N_m(h) - \delta^{(3)} g^N_m(h))}^2 \\
&\leq Nt \sum_{m=0}^{\floor{Nt}-1} \E\abs{\triangle g^N_m(h) - \delta^{(3)} g^N_m(h))}^2 \numberthis \label{eq:KernelEvolution3} \\
&\leq 3Nt \sum_{m=0}^{\floor{Nt}-1}\E \bigg[\abs{\triangle g^N_m(h) - \delta^{(1)} g^N_m(h))}^2 \\
&\phantom{=}+ \abs{\delta^{(1)} g^N_m(h) - \delta^{(2)} g^N_m(h))}^2 + \abs{\delta^{(2)} g^N_m(h) - \delta^{(3)} g^N_m(h))}^2 \bigg] \\
&\leq (NT)^2 \times \frac{C_T \|h\|^2_{C^2}}{N^{4-2\beta-4\gamma}} = \frac{C_T \|h\|^2_{C^2}}{N^{2-2\beta-4\gamma}}.
\end{align*}
We have been careful with our proofs to ensure that \eqref{eq:KernelEvolution3} remains true even when $h$ is replaced with $h_{k+1}$. With the extra effort, the following estimates holds true for $k\geq 0$ uniformly:
\begin{equation}
\E\abs{g^N_t(h_{k+1}) - g^N_0(h_{k+1}) - \sum_{k=0}^{\floor{Nt} - 1} \delta^{(3)} g^N_k(h_{k+1})}^2 \lesssim \frac{C_T}{N^{2-2\beta-4\gamma}}. \label{eq:KernelEvolution3A}
\end{equation}

Before we move on, let us simplify the expression of $\delta^{(3)} g^N_m(h)$. Recall $h_{m+1}(w) = \sigma(w^\top X_m + u_m)$, so that $\nabla h_{m+1}(w) = \sigma'(w^\top X_m + u_m) X_m$. Define the kernel $\K_\lambda(\cdot,\cdot): \cH \times \cH \to \R$ for all $x$:
\begin{equation}
\K_{\lambda}(h, \varsigma) = \int_{\mathcal{W}} [h(\sw) \varsigma(\sw) + \mu_{c^2} \nabla h(\sw)^\top \nabla\varsigma(\sw)] \, \lambda(d\sw),
\end{equation}
then
\begin{align*}
\delta^{(3)} g^N_m(h)
&= - \frac{\alpha}{N} (\psi^N(g^N_m(h_{m+1})) - Y_m) \\
&\phantom{=} \times \int_{\mathcal{W}} \Big[\sigma(\sw^\top X_m + u_m) h(\sw) + \mu_{c^2} \sigma' \bracket{\sw^\top X_m + u_m} \nabla h(\sw)^\top X_m \Big] \, \lambda(d\sw) \\
&= - \frac{\alpha}{N} (\psi^N(g^N_m(h_{m+1})) - Y_m) \\
&\phantom{=} \times \int_{\mathcal{W}} \Big[h_{k+1}(\sw) h(\sw) + \mu_{c^2} \nabla h_{k+1}(\sw)^\top \nabla h(\sw) \Big] \, \lambda(d\sw) \\
&= - \frac{\alpha}{N} (\psi^N(g^N_m(h_{m+1})) - Y_m) \, \mathcal{K}_\lambda(h,h_{m+1}). \numberthis \label{eq:6_32}
\end{align*}
Recall we have the following bound from \eqref{eq:norm_for_k_lambda} for the outputs of $K_\lambda$:
\begin{equation}
\abs{\mathcal{K}_{\lambda}(h,\varsigma)} \leq C \|h\|_{C^2} \|\varsigma \|_{C^2}. \label{eq:6_33}
\end{equation}

\subsection{Removal of the clipping function}
\label{SS:remove_clip}
The next step for us would be to remove the clipping function in $\delta^{(3)} g^N_m(h)$, and set the initial value of the output evolution to zero, without introducing too much error.

\begin{lemma}[Uniform integrability of $g^N_k(h_{k+1})$]
For $k \leq \floor{NT} - 1$, we have
$$\E\sqbracket{g^N_k(h_{k+1})}^2 \leq C_T.$$
\end{lemma}

\begin{proof}
Using the mutual independence of all random variables $(C^i_0)_{i=1}^N, (W^i_0)_{i=1}^N$ and $X_{0:k}$, and the fact that $C^i_0$ have zero means, we could show that for all $k \geq 0$
\begin{align*}
\E|g^N_0(h_{k+1})|^2
&= \E\sqbracket{\E\sqbracket{\bracket{\frac{1}{N^\beta} \sum_{i=1}^N C^i_0\sigma((W^i_0)^\top X_k + u_k)}^2 \, \bigg|\, X_{0:k}}} \\
&= \frac{1}{N^{2\beta}} \sum_{i_1, i_2=1}^N \E\sqbracket{\E\sqbracket{C^{i_1}_0 C^{i_2}_0 \sigma((W^{i_1}_0)^\top X_k + u_k) \sigma((W^{i_2}_0)^\top X_k + u_k) \, \bigg|\, X_{0:k}}} \\
&= \frac{1}{N^{2\beta}} \sum_{i_1, i_2=1}^N \bigg[\E[C^{i_1}_0 C^{i_2}_0] \\
&\phantom{=}\times \E\sqbracket{\E\sqbracket{\sigma((W^{i_1}_0)^\top X_k + u_k) \mid X_{0:k}} \E\sqbracket{\sigma((W^{i_2}_0)^\top X_k + u_k) \mid X_{0:k}}} \bigg] \\
&= \frac{1}{N^{2\beta}} \sum_{i=1}^N \E[(C^i_0)^2] \E\sqbracket{\E\sqbracket{\bracket{\sigma((W^i_0)^\top X_k + u_k)}^2 \mid X_{0:k}}^2 } \\
&\leq \frac{\mu_{c^2} C_\sigma^2}{N^{2\beta - 1}} \nonumber
\end{align*}
Therefore, by \eqref{eq:KernelEvolution3A} (with constant $C_T > 0$ defined in the inequality) we have
\begin{align*}
\E|g^N_k(h_{k+1})|^2
&= \E\bigg|g^N_0(h_{k+1}) + \sum_{m=0}^{k-1} \delta^{(3)} g^N_m(h_{k+1}) \\
&\phantom{=}+ g^N_k(h_{k+1}) - g^N_0(h_{k+1}) - \sum_{m=0}^{k-1} \delta^{(3)} g^N_m(h_{k+1})\bigg|^2 \\
&\leq 3\E|g^N_0(h_{k+1})|^2 + 3\E\abs{\sum_{m=0}^{k-1} \delta^{(3)} g^N_m(h_{k+1})}^2 \\
&\phantom{=}+ 3\E\abs{g^N_k(h_{k+1}) - g^N_0(h_{k+1}) - \sum_{m=0}^{k-1} \delta^{(3)} g^N_m(h_{k+1})}^2 \\
&\leq 3k \sum_{m=0}^{k-1} \E\sqbracket{\delta^{(3)} g^N_m(h_{k+1})}^2 + \frac{3\mu_{c^2} C^2_\sigma}{N^{2\beta-1}} + \frac{3C_T}{N^{2-2\beta-4\gamma}}\\
&\leq \frac{3\alpha^2 T}{N} \sum_{m=0}^{k-1} \E[\abs{\psi^N(g^N_m(h_{m+1})) - Y_m}^2 \abs{\mathcal{K}_\lambda(h_{k+1}, h_{m+1})}^2] \\
&\phantom{=}+ \frac{3\mu_{c^2} C^2_\sigma}{N^{2\beta-1}} + \frac{3C_T}{N^{2-2\beta-4\gamma}} \\
&\leq \frac{3CT}{N} \sum_{m=0}^{k-1} 2 \sqbracket{\E\abs{\psi^N(g^N_m(h_{m+1}))}^2 + \E\abs{Y_m}^2} + \frac{3\mu_{c^2} C^2_\sigma}{N^{2\beta-1}} + \frac{3C_T}{N^{2-2\beta-4\gamma}} \\
&\leq \frac{6CT}{N} \sum_{m=0}^{k-1}  \sqbracket{\E\abs{g^N_m(h_{m+1})}^2} + 2C^2_y C T + \frac{3\mu_{c^2} C^2_\sigma}{N^{2\beta-1}} + \frac{3C_T}{N^{2-2\beta-4\gamma}} \\
&\leq \frac{6CT}{N} \sum_{m=0}^{k-1} \E\abs{g^N_m(h_{m+1})}^2 + C_T.\nonumber
\end{align*}
We may then use discrete Gr\"onwall's inequality to conclude that
\begin{equation}
\E|g^N_k(h_{k+1})|^2 \leq C_T \exp\left(\frac{kTC_T}{N} \right) \leq C_T, \nonumber
\end{equation}
for a possibly different constant $C_{T}<\infty$.
\end{proof}

Because the expectation $\E|g^N_k(h_{k+1})|^2$ is uniformly bounded for $m \leq \floor{NT}$, we may invoke Markov's inequality to prove the following:
\begin{lemma}[Removal of clipping function in $\delta^{(3)}$]
Define
\begin{align}
    \delta^{(4)} g^N_m(h)
    = -\frac{\alpha}{N} (g^N_m(h_{m+1}) - Y_m) \K_{\lambda}(h, h_{m+1}),\nonumber
\end{align}
then for all $h \in C^2_b(\R^d)$, the following holds for all $m \leq \floor{NT}$
\begin{equation}
    \E\abs{\delta^{(3)} g^N_m(h) - \delta^{(4)} g^N_m(h)} \leq \frac{C_T}{N^{1+\gamma}} \|h\|_{C^2}.\nonumber
\end{equation}
\end{lemma}

\begin{proof}
Note that
\begin{align*}
    \abs{\delta^{(3)} g^N_m(h) - \delta^{(4)} g^N_m(h)}
    &= \frac{\alpha C}{N} |\psi^N(g^N_m(h_{m+1})) - g^N_k(h_{m+1})| |\mathcal{K}_{\lambda}(h, h_k)| \\
    &\lesssim \frac{\|h\|_{C^2}}{N} |\psi^N(g^N_m(h_{k+1})) - g^N_k(h_{m+1})|
\end{align*}
Recall that when $|g^N_m(h_{m+1})| \leq N^\gamma$ then $\psi^N(g^N_m(h_{m+1})) = g^N_m(h_{m+1})$, so
\begin{align*}
    \psi^N(g^N_m(h_{m+1})) - g^N_m(h_{m+1})
    &= (\psi^N(g^N_m(h_{m+1})) - g^N_m(h_{m+1})) \\
    &\phantom{=}\times \bracket{\mathbf{1}_{\set{|g^N_m(h_{m+1})| > N^\gamma}} + \mathbf{1}_{\set{|g^N_m(h_{m+1})| \leq N^\gamma}}} \\
    &= (\psi^N(g^N_m(h_{m+1})) - g^N_m(h_{m+1})) \mathbf{1}_{\set{|g^N_m(h_{m+1})| > N^\gamma}}.\nonumber
\end{align*}
Combining with $|\psi^N(x)| \leq |x|$, we have
\begin{eqnarray*}
    &\phantom{=}& \E|\delta^{(3)} g^N_m(h) - \delta^{(4)} g^N_m(h)| \\
    &\lesssim& \frac{\|h\|_{C^2}}{N} \E\sqbracket{|\psi^N(g^N_m(h_{m+1})) - g^N_m(h_{m+1})| \mathbf{1}_{\set{|g^N_m(h_{m+1})| > N^\gamma}}} \\
    &\leq& \frac{\|h\|_{C^2}}{N} \E\sqbracket{(\abs{\psi^N(g^N_m(h_{m+1}))} + \abs{g^N_k(h_{m+1})})  \mathbf{1}_{\set{|g^N_m(h_{m+1})| > N^\gamma}}} \\
    &\lesssim& \frac{2\|h\|_{C^2}}{N} \E\sqbracket{\abs{g^N_m(h_{m+1})}  \mathbf{1}_{\set{|g^N_m(h_{m+1})| > N^\gamma}}} \\
    &\leq& \frac{2\|h\|_{C^2}}{N} \sqrt{\E\abs{g^N_m(h_{m+1})}^2} \sqrt{\E\sqbracket{\mathbf{1}_{\set{|g^N_m(h_{m+1})| > N^\gamma}}}} \\
    &\leq& \frac{2\|h\|_{C^2}}{N} \sqrt{\frac{\sqbracket{\E\abs{g^N_m(h_{m+1})}^2}^2}{N^{2\gamma}}} \\
    &\leq& \frac{C_T}{N^{1+\gamma}} \|h\|_{C^2}.
\end{eqnarray*}
\end{proof}

Arguing as before, we have
\begin{eqnarray}
&\phantom{=}& \E\abs{g^N_t(h) - g^N_0(h) - \sum_{m=0}^{\floor{Nt} - 1} \delta^{(4)} g^N_m(h)} \nonumber \\
&\leq& \E\abs{g^N_t(h) - g^N_0(h) - \sum_{k=0}^{\floor{Nt}-1} \delta^{(3)} g^N_k(h)} + \sum_{m=0}^{\floor{Nt}-1} \E\abs{\delta^{(3)} g^N_k(h) - \delta^{(4)} g^N_k(h)} \\
&\lesssim& \sqbracket{\E\abs{g^N_t(h) - g^N_0(h) - \sum_{k=0}^{\floor{Nt}-1} \delta^{(3)} g^N_k(h)}^2}^{1/2} + \frac{TC_T N}{N^{\gamma}} \|h\|_{C^2} \nonumber \\
&\lesssim& \frac{C_T}{N^{(1-\beta-2\gamma) \wedge \gamma}} \|h\|_{C^2}, \nonumber
\label{eq:KernelEvolution4}
\end{eqnarray}

We now record a simple observation on the boundedness of the second moment of $g^N_0(h)$.
\begin{lemma} \label{lem:initial_boundedness}
For all $h \in C^2_b$,
\begin{equation}
    \E\sqbracket{g^N_0(h)}^2 \leq \frac{\mu_{c^2}}{N^{2\beta-1}} \|h\|^{2}_{C^2} \nonumber
\end{equation}
\end{lemma}

\begin{proof}
Using the independence and identical distribution of $(C^i_0,W^i_0)$, for $i=1,\cdots,N$ and the fact that  $C^i_0$ and $W^i_0$ are independent from each other with $C^i_0$ having zero mean, we have
\begin{equation}
    \E\sqbracket{g^N_0(h)}^2 = \frac{1}{N^{2\beta}} \E\sqbracket{\sum_{i=1}^N C^i_0 h(W^i_0)}^2 = \frac{\E\sqbracket{C^1_0 h(W^1_0)}^{2}}{N^{2\beta - 1}} \leq \frac{\mu_{c^2}}{N^{2\beta-1}} \|h\|^{2}_{C^2}\nonumber
\end{equation}
\end{proof}

This allows us to consider the new evolution equation that $\varphi^N_k(h)$ that satisfies
\begin{align}
    \varphi^N_k(h) &= \sum_{m=0}^{k-1} \triangle \varphi^N_m(h), \quad \varphi^N_m(h) = 0 \\
    \triangle \varphi^N_m(h) &= - \frac{\alpha}{N} (\varphi^N_m(h_{m+1}) - Y_m) \mathcal{K}_{\lambda}(h, h_{m+1}). \notag \\
    &= - \frac{\alpha}{N} (\varphi^N_m(\varsigma_{X_m,u_m} - Y_m) \mathcal{K}_{\lambda}(h, \varsigma_{X_m,u_m}). \nonumber
\end{align}

\begin{lemma}[$\varphi^N_k(h)$ approximates $g^N_k(h)$] \label{lem:D8}
For all $k \leq \floor{NT}$ and $h \in \cH$,
\begin{equation} \label{lem:phi^N}
\E\abs{g^N_k(h) - \varphi^N_k(h)} \lesssim \frac{C_T}{N^{(1-\beta-2\gamma) \wedge \gamma \wedge (\beta-1/2)}} \|h\|_{C^2}.
\end{equation}
\end{lemma}

\begin{proof}
In light of Lemma \ref{lem:initial_boundedness} and equation \eqref{eq:KernelEvolution4} (with $\sfC > 0$ as defined for \eqref{eq:KernelEvolution4}), we know that
\begin{align*}
&\phantom{=} \E\abs{g^N_k(h) - \varphi^N_k(h)} \\
&\leq \E\abs{g^N_k(h) - g^N_0(h) - \sum_{m=0}^{k-1} \delta^{(4)}g^N_m(h) + g^N_0(h) + \sum_{m=0}^{k-1} \delta^{(4)}g^N_m(h) - \varphi^N_k(h)} \\
&\leq \E\abs{g^N_k(h) - g^N_0(h) - \sum_{m=0}^{k-1} \delta^{(4)}g^N_m(h)} + \E|g^N_0(h)| + \E\abs{\sum_{m=0}^{k-1} (\delta^{(4)} g^N_m(h) - \triangle \varphi^N_m(h))} \numberthis \label{eq:initial_boundedness_eq1} \\
&\leq \sum_{m=0}^{k-1} \E\abs{\delta^{(4)} g^N_m(h) - \triangle \varphi^N_m(h)} + \frac{\sqrt{\mu_{c^2}} \|h\|_{C^2}}{N^{\beta-1/2}} + \frac{C_T}{N^{(1-\beta-2\gamma) \wedge \gamma}} \|h\|_{C^2} \\
&\leq \frac{\alpha C\|h\|_{C^2}}{N} \sum_{m=0}^{k-1} \E\abs{g^N_m(h_{m+1}) - \varphi^N_m(h_{m+1})} + \frac{\sqrt{\mu_{c^2}} \|h\|_{C^2}}{N^{\beta-1/2}} + \frac{C_T}{N^{(1-\beta-2\gamma) \wedge \gamma}} \|h\|_{C^2}.
\end{align*}
The above inequality is also true when replacing $h$ with $h_{k+1}$. In particular, as $\|h_{k+1}\|_{C^2} \leq C$, so
\begin{align*}
\E\abs{g^N_k(h_{k+1}) - \varphi^N_k(h_{k+1})} &\leq \frac{C}{N} \sum_{m=0}^{k-1} \E|g^N_m(h_{m+1}) - \varphi^N_m(h_{m+1})| \\
&\phantom{=}+ \frac{K^2_b}{N^{\beta-1/2}} + \frac{C_T}{N^{(1-\beta-2\gamma)\wedge \gamma}}; \numberthis \label{eq:initial_boundedness_eq2}
\end{align*}
so by discrete Gronwall's inequality,
\begin{align*}
\E\abs{g^N_k(h_{k+1}) - \varphi^N_k(h_{k+1})} &\leq \frac{C_T k/N}{N^{(1-\beta-2\gamma) \wedge \gamma \wedge (\beta-1/2)}} \\
&\leq \frac{C_T}{N^{(1-\beta-2\gamma) \wedge \gamma \wedge (\beta-1/2)}}. \numberthis \label{eq:initial_boundedness_eq3}
\end{align*}
We can then plug in \eqref{eq:initial_boundedness_eq3} into \eqref{eq:initial_boundedness_eq1} to obtain our desired result.
\end{proof}

With the absence of initialisation in the sequence $\varphi^N_m(h)$, one can prove that $\varphi^N_m(h)$ is \textit{surely} bounded for any $h \in C^2_b(\R^d)$.

\begin{lemma} \label{lem:lipschitzness_of_G}
For $k \leq \floor{NT}$ and $h \in C^2_b$, there is $C_T > 0$ such that:
\begin{itemize}
\item $|\varphi^N_k(h)| \leq C_T \|h\|_{C^2}$ surely, and
\item $\varphi^N_k(h)$ is globally Lipschitz over $C^2_b(\R^d)$, with respect to $\|\cdot\|_{C^2}$, with the Lipschitz constant $\leq C_T$.
\end{itemize}
Furthermore, consider the map for $\sH = (\sx, \sz, \sy, \su) \in \R^{d+d_z+d_y+1}$:
\begin{align*}
G^{N,h}_m: \sH \in \R^{d+d_z+d_y+1} &\mapsto - \alpha(\varphi^N_m(\varsigma_{\sx,\su}) - \sy) \K_{\lambda}(h,\varsigma_{\sx,\su}),\nonumber
\end{align*}
so that
\begin{equation} \label{eq:def_of_G_2}
N \times \triangle \varphi^N_m(h) = G^{N,h}_m(X_m, Z_m, Y_m, u_m).
\end{equation}
Then for fixed $h \in C^2_b$ and $m \leq \floor{NT}$,
\begin{itemize}
\item $\displaystyle{\int |G^{N,h}_m| \, d\mu \leq C_T \|h\|_{C^2}}$, where $\mu$ is the unique invariant measure of chain $(H_k)$ as defined in Theorem \ref{prop:conv_of_mu_k}, and
\item $G^{N,h}_m$ is $\|h\|_{C^2}$-Lipschitz over the support of $\mu$, with the Lipschitz constant $\leq C_T$.
\end{itemize}
\end{lemma}

\begin{proof}
From the evolution equation, one sees that
\begin{equation}
\abs{\varphi^N_k(h)} \leq \frac{C}{N} \|h\|_{C^2}\sum_{m=0}^{k-1} (\abs{\varphi^N_m(h_{m+1})} + C_y), \label{eq:sure_boundedness_1}
\end{equation}
The above equation is also true when $h$ is replaced by $h_{k+1}$ - in fact, the following holds:
\begin{equation*}
\abs{\varphi^N_k(h_{k+1})} \leq \frac{C}{N} \sum_{m=0}^{k-1} (|\varphi^N_m(h_{m+1})| + C_y),
\end{equation*}
so by discrete Gronwall's inequality one has
\begin{equation}
\abs{\varphi^N_m(h_{m+1})} \lesssim \exp(Cm/N) \leq C_T. \label{eq:sure_boundedness_evolving}
\end{equation}
Substituting this into \eqref{eq:sure_boundedness_1} yields our estimate $|\varphi^N_m(h)| \lesssim C_T \|h\|_{C^2}$. The global Lipschitzness of $\varphi^N_m(\cdot)$ follows from its linearity:
\begin{align}
|\varphi^N_k(h) - \varphi^N_k(\tilde{h})| = |\varphi^N_k(h -\tilde{h})| \leq C_T \|h - \tilde{h} \|_{C^2}. \label{eq:Lipschitzness_phiNm}
\end{align}

Next, we wish to prove the sure boundedness of $\displaystyle{\int |G^{N,h}_m(\sH)| \, \mu(d\sH)}$. In fact, we have
\begin{align*}
\int_{\X} |G^{N,h}_m(\sH)| \, \mu(d\sH) &= \int_{\X} \alpha |\varphi^N_m(\varsigma_{\sx,\su})- \sy| |\K_\lambda(h,\varsigma_{\sx,\su})| \, \mu(d\sH) \leq C_T \|h\|_{C^2},
\end{align*}
as $\mu$ has bounded support.

Finally, we show the local Lipschitzness of $G^{N,h}_m$. We shall focus on the case where $|\sx| \vee |\sx'| \leq C_x$ and $|\sy| \vee |\sy'|\leq C_y$ according to Assumption \ref{as:data_generation} on the input and output sequences. Therefore, for any $\sH = (\sx,\sz,\sy,\su)$ and $\tilde{\sH} = (\tilde{\sx},\tilde{\sz},\tilde{\sy},\tilde{\su})$:
\begin{align*}
|G^{N,h}_m(\sH) - G^{N,h}_m(\tilde{\sH})| &\leq \alpha |\varphi^N_m(\varsigma_{\sx,\su} - \varsigma_{\tilde{\sx},\tilde{\su}}) + \sy - \tilde{\sy}| |\K_{\lambda}(h,\varsigma_{\tilde{\sx},\tilde{\su}})| \\
&\phantom{=}+ \alpha |\varphi^N_m(\varsigma_{\sx,\su}) - \sy||\K_{\lambda}(h,\varsigma_{\tilde{\sx},\tilde{\su}} - \varsigma_{\sx,\su})| \\
&\leq C_T \|h\|_{C^2} [|\sx-\tilde{\sx}| + |\su-\tilde{\su}| + |\sy-\tilde{\sy}|],
\end{align*}
which completes our proof. In particular, the Lipschitz constant for $G^{N,h}_0$ is independent from $T$ since $\varphi^N_0(\cdot)$ is set to zero.
\end{proof}

\begin{remark}
We note that the random variables $\varphi^N_m(h_{m+1}) = \varphi^N_m(\varsigma_{X_m,u_m})$ is $(H_{0:k})$-measurable whenever $m \leq k$. This could be shown by strong induction: for the base case when $k=0$, $\varphi^N_0(h_1) = 0$ is clearly $H_0$-measurable.

Assume, for the purpose of induction, there is a $k'$ such that we know that for all $k \leq k'$, $\varphi^N_k(h_{k+1})$ is $(H_{0:k})$-measurable (hence also $(H_{0:k'+1})$-measurable). Then
\begin{equation*}
\varphi^N_{k'+1}(h_{k'+2}) = -\frac{\alpha}{N} \sum_{m=0}^{k'} (\underbrace{\varphi^N_m(h_{m+1})}_{(H_{0:k'})\text{-measurable}}-Y_m) \underbrace{\K_\lambda(h_{k'+2}, h_{m+1})}_{(H_{0:k'+1})\text{-measurable}}
\end{equation*}
so $\varphi^N_{k'+1}(h_{k'+2})$ is indeed $(H_{0:k'+1})$-measurable as well. With this in mind, we can also show that for any $h \in C^2_b(\R^d)$, both $\varphi^N_k(h)$ and $G^{N,h}_k$ are both $(H_{0:{k-1}})$-measurable:
\begin{equation*}
\varphi^N_k(h) = -\frac{\alpha}{N} \sum_{m=0}^{k-1} (\underbrace{\varphi^N_m(h_{m+1})}_{(H_{0:k-1})\text{-measurable}} - Y_m) \underbrace{\K_\lambda(h,h_{m+1})}_{(H_{0:k-1})\text{-measurable}}.
\end{equation*}
\end{remark}

\subsection{Weak convergence analysis for $g^N_k(h)$}
\label{SS:weak_convergence_analysis}

We study the difference between the random evolution
\begin{equation*}
\varphi^N_k(h) = \sum_{m=0}^{k-1} \triangle \varphi^N_m(h), \quad \varphi^N_0(h) \equiv 0,
\end{equation*}
and the evolution
\begin{equation*}
    \tilde{\varphi}^N_k(h) = \sum_{m=0}^{k-1} \delta \varphi^N_m(h), \quad \tilde{\varphi}^N_0(h) \equiv 0,
\end{equation*}
where
\begin{equation*}
\delta \varphi^N_k(h) = - \frac{\alpha}{N} \int_{\X} (\varphi^N_k(\varsigma_{\sx,\su}) - \sy) \K_{\lambda}(h,\varsigma_{\sx,\su}) \, \mu(d\sH).
\end{equation*}

We shall first construct the associated Poisson equation of the chain of RNN memory layers \cite{meyertweedie}. We write the transition kernel of the Markov chain $(H_k)_{k\geq 1}$ as $P$, and recall that the Markov chain admits a limiting invariant measure $\mu$.

\begin{definition}
Let $G: \R^{d+d_z+d_y+1} \to \R$ be a measurable function with $\int |G| \, d\mu < +\infty$, then the Poisson equation is a functional equation on $\hat{G}$:
\begin{equation} \label{eq:Poisson_equation}
\hat{G}(H) - \int_{\X} \hat{G}(\bar{H}) \, Q(H, d\bar{H}) = G(H) - \int_{\X} G(\bar{H}) \, \mu(d\bar{H}),
\end{equation}
where $H=(x,z,y,u)$, $\bar{H}=(\bar{x},\bar{z},\bar{y},\bar{u})$, and $Q$ is the transitional kernel of the Markov Chain $(H_k)$.
\end{definition}

In our case, the Poisson equation is used to replace the integral with respect to the invariant measure $\mu$ with an integral with respect to the transition kernel $Q$. We note that the following expansion converges for any $H = (x,z,y,u) \in \mathcal{X}$ if $G$ is globally $C$-Lipschitz on the support of $\mu$:
\begin{equation} \label{eq:proposed_solution_of_Poisson_equation}
\hat{G}(H) = \sum_{k=0}^\infty \bracket{\int_{\X} G(\bar{H}) \, Q^k(H,d\bar{H}) - \int_{\X} G(\bar{H}) \, \mu(d\bar{H})},
\end{equation}
where $Q^k = \underbrace{P \circ ... \circ P}_{k \text{ times}}$. This is because
\begin{equation} \label{eq:sln_to_Poisson_eqn_bdd_1}
\abs{\int_{\X} G(\bar{H}) Q^k(H,d\bar{H}) - \int_{\X} G(\bar{H}) \mu(d\bar{H})} \lesssim \Wass_2(Q^k(H,\cdot), \mu) \lesssim q_0^k.
\end{equation}
As a result, for all $H$
\begin{equation} \label{eq:sln_to_Poisson_eqn_bdd_2}
|\hat{G}(H)| = \abs{\sum_{k=0}^\infty \bracket{\int_{\X} G(\bar{H}) \, P^k(H,d\bar{H}) - \int_{\X} G(\bar{H}) \, \mu(d\bar{H})}} \lesssim \sum_{k=0}^\infty q_0^k = \frac{1}{1-q_0} < +\infty.
\end{equation}

\begin{lemma} \label{lem:sln_of_Poisson_equation_is_C-Lipschitz}
Let $G$ be a $L_G$-Lipschitz function on the support of $\mu$, and define $\hat{G}$ as in \eqref{eq:proposed_solution_of_Poisson_equation}. Then $\hat{G}$ is the solution to the Poisson equation \eqref{eq:Poisson_equation}. Moreover $\hat{G}$ is Lipschitz on the support of $\mu$, with the Lipschitz constant $\lesssim L_G/(1-q_0)$.
\end{lemma}

\begin{proof}
We observe that
\begin{equation}
\int_{\X} \hat{G}(\bar{H}) \, Q(H, d\bar{H}) = \int_{\X} \sqbracket{\sum_{k=0}^\infty \bracket{\int_{\X} G(\sH) \, Q^k(\bar{H}, d\sH) - \int_{\X} G(\sH) \, \mu(d\sH)}}  \, Q(H, d\bar{H}). \nonumber
\end{equation}
We can exchange the sum and the integral as the integrand is summable (hence integrable). Noting that $\mu$ is invariant, we have
\begin{equation}
\int_{\X} \hat{G}(\bar{H}) \, Q(H, d\bar{H}) = \sum_{k=0}^\infty \bracket{\int_{\X} G(\bar{H}) Q^{k+1}(H,d\bar{H}) - \int_{\X} G(\bar{H}) \mu(d\bar{H})}.\nonumber
\end{equation}
Subtracting this from $\hat{G}(H)$ yields
\begin{align*}
\hat{G}(H) - \int_{\X} \hat{G}(\bar{H}) \, Q(H,d\bar{H})
&= \int_{\X} \hat{G}(\bar{H}) \, Q^0(H, d\bar{H}) - \int_{\X} G(\bar{H}) \, \mu(d\bar{H}) \\
&= G(H) - \int_{\X} G(\bar{H}) \, \mu(d\bar{H})\nonumber
\end{align*}
as desired. We further note that if $H'=(x',z',y',h')$, then
\begin{align*}
    |\hat{G}(H) - \hat{G}(H')| &\leq \sum_{k=0}^\infty \abs{\int_{\X} G(\bar{H}) \, Q^k(H,d\bar{H}) - \int_{\X} G(\bar{H}) \, Q^k(H',d\bar{H})} \\
    &\leq L_G \sum_{k=0}^\infty \Wass_2(Q^k(H,\cdot), Q^k(H',\cdot)) \\
    &\lesssim L_G \|H - H'\|_\X \sum_{k=0}^\infty q_0^k =\frac{L_G}{1-q_0} \|H - H'\|_\X, \nonumber
\end{align*}
completing the proof.
\end{proof}

We now consider the Poisson equations with $G(\sH) = G^{N,h}_m(\sH)$ as defined in Lemma \ref{lem:lipschitzness_of_G}. These Poisson equations admit a solution for each $m \leq \floor{NT}-1$ using the expansion in \eqref{eq:proposed_solution_of_Poisson_equation}, for which we will call them ${\hat{G}^{N,h}_m}$. In summary we have
\begin{equation} \label{eq:Poisson_equation_special}
    \hat{G}^{N,h}_m(\sH) - \int_{\X} \hat{G}^{N,h}_m(\bar{\sH}) \, P(\sH, d\bar{\sH}) = G^{N,h}_m(\sH) - \int_{\X} G^{N,h}_m(\bar{\sH}) \, \mu(d\bar{\sH}).
\end{equation}
It is important to note from the above analysis that, since there is a constant $\sfC > 0$ such that $G^{N,h}_m(\sH)$ is $C_T \|h\|_{C^2}$-Lipschitz (Lemma \ref{lem:lipschitzness_of_G}) for $m \leq \floor{NT}$, the above analysis shows that the following holds:
\begin{corollary} \label{cor:regularity_of_hat_Gmh}
\begin{equation}
\sup_{\sH \in \mathsf{supp}\, \mu} |\hat{G}^{N,h}_m(\sH)| \lesssim \frac{C_T}{1-q_0} \|h\|_{C^2},\nonumber
\end{equation}
and that there is a constant $\sfC > 0$ such that $\hat{G}^{N,h}_m(\cdot)$ is $\displaystyle{\frac{C_T}{1-q_0} \|h\|_{C^2}}$-Lipschitz over the support of $\mu$.
\end{corollary}

With this, we could study the difference between $\varphi^N_k(h)$ and $\tilde{\varphi}^N_k(h)$ for any $h \in C^2_b(\R^d)$ and $m \leq \floor{NT}$. Recalling that $H_k = (X_k,Z_k,Y_k,u_k)$ and letting $\bar{\sH} = (\bar{\sx},\bar{\sz},\bar{\sy},\bar{\su})$, we have
\begin{align*}
\varphi^N_k(h) - \tilde{\varphi}^N_k(h) &= \sum_{m=0}^{k-1} (\triangle \varphi^N_m(h) - \delta \varphi^N_m(h)) \\
&= \frac{1}{N} \sum_{m=0}^{k-1} \bracket{G^{N,h}_m(H_m) - \int_{\X} G^{N,h}_m(\bar{\sH}) \, \mu(d\bar{\sH})} \\
&= \frac{1}{N} \sum_{m=0}^{k-1} \bracket{\hat{G}^{N,h}_m(H_m) - \int_{\X} \hat{G}^{N,h}_m(\bar{\sH}) \, Q(H_m, d\bar{\sH})} \\
&= \fe^{N,1}_k(h) + \fe^{N,2}_k(h) + R^{N,1}(h), \numberthis \label{eq:DiffE1}
\end{align*}
where
\begin{align}
\fe^{N,1}_k(h) &= \frac{1}{N} \sum_{m=1}^{k-1} \bracket{ \hat{G}^{N,h}_m(H_m) - \int_{\X} \hat{G}^{N,h}_m(\bar{\sH}) \, Q(H_{m-1}, d\bar{\sH})} \nonumber\\
\fe^{N,2}_k(h) &= \frac{1}{N} \sum_{m=1}^{k-1} \bracket{ \int_{\X} \hat{G}^{N,h}_m(\bar{\sH}) \, Q(H_{m-1}, d\bar{\sH}) - \int_{\X} \hat{G}^{N,h}_m(\bar{\sH}) \, Q(H_m, d\bar{\sH})} \nonumber\\
R^{N,1}(h) &= \frac{1}{N} \bracket{\hat{G}^{N,h}_0(H_0) - \int_{\X} \hat{G}^{N,h}_0(\bar{\sH}) \, Q(H_0, d\bar{\sH})}\nonumber
\end{align}
Corollary \ref{cor:regularity_of_hat_Gmh} immediately leads to
\begin{equation}
|R^{N,1}(h)| \lesssim \frac{C_T}{N}\|h\|_{C^2} \overset{N\to\infty}\to 0.\nonumber
\end{equation}
The second term on the RHS of \eqref{eq:DiffE1}, i.e. $\fe^{N,2}_k$, could be analysed by observing that
\begin{align*}
\fe^{N,2}_k(h) &= \frac{\alpha}{N} \sum_{m=1}^{k-1} \bigg(\int_{\X} \hat{G}^{N,h}_m(\bar{H}) \, Q(H_{m-1}, d\bar{H})- \int_{\X} \hat{G}^{N,h}_m(\bar{H}) \, Q(H_m, d\bar{H})\bigg) \\
&= \frac{\alpha}{N} \bigg(\int_{\X} \hat{G}^{N,h}_1(\bar{H}) Q(H_0, d\bar{H}) + \sum_{m=2}^{k-1} \int_{\X} \hat{G}^{N,h}_m(\bar{H}) \, Q(H_{m-1}, d\bar{H})  \\
&\phantom{=}- \sum_{m=1}^{k-1}\int_{\X} \hat{G}^{N,h}_m(\bar{H}) \, Q(H_m, d\bar{H}) \bigg) \\
&= \frac{\alpha}{N} \bigg(\int_{\X} \hat{G}^{N,h}_1(\bar{H}) Q(H_0, d\bar{H}) + \sum_{m=1}^{k-2} \int_{\X} \hat{G}^{N,h}_{m+1}(\bar{H}) \, Q(H_{m}, d\bar{H}) \\
&- \sum_{m=1}^{k-1}\int_{\X} \hat{G}^{N,h}_m(\bar{H}) \, Q(H_m, d\bar{H}) \bigg)\\
&= \frac{\alpha}{N} \sum_{m=1}^{k-2} \int (\hat{G}^{N,h}_{m+1}(\bar{H}) - \hat{G}^{N,h}_m(\bar{H})) \, Q(H_{m}, d\bar{H}) + R^{N,2}_k(h), \nonumber
\end{align*}
where
\begin{equation}
R^{N,2}_k(h) = \frac{\alpha}{N} \bigg(\int_{\X} \hat{G}^{N,h}_1(\bar{H}) \, P(H_0, d\bar{H}) - \int_{\X} \hat{G}^{N,h}_{k-1}(\bar{H}) \, P(H_{k-1}, d\bar{H}) \bigg) \nonumber
\end{equation}
By Corollary \ref{cor:regularity_of_hat_Gmh}, both $\hat{G}^{N,h}_1(\cdot)$ and $\hat{G}^{N,h}_{k-1}(\cdot)$ are bounded by $C_T \|h\|_{C^2}$ for some constant $\sfC > 0$. Therefore $R^{N,2}_k(h)$ satisfies the following trivial bound:
\begin{align*}
|R^{N,2}_k(h)| &\leq \frac{\alpha}{N} \sqbracket{\int |\hat{G}^{N,h}_1(\bar{H})| \, P(H_0, d\bar{H}) + \int |\hat{G}^{N,h}_{k-1}(\bar{H})| \, P(H_{k-1}, d\bar{H})} \leq \frac{C_T}{N}\|h\|_{C^2}.
\end{align*}
The first term of $\fe^{N,2}_k(h)$, is bound by the following lemma.
\begin{lemma}
For $t \leq T$, there are $C_T > 0$ such that for all $m \leq \floor{NT}$,
\begin{equation}
\sup_{\sH \in \mathsf{supp} \, \mu} |\Delta \hat{G}^{N,h}_m(\sH)| \leq \frac{C_T}{N} \|h\|_{C^2},\nonumber
\end{equation}
where $\Delta \hat{G}^{N,h}_{m}(\sH) = \hat{G}^{N,h}_{m+1}(\sH) - \hat{G}^{N,h}_{m}(\sH)$.
\end{lemma}

\begin{proof}
Define $\Delta G^{N,h}_m(\sH) = G^{N,h}_{m+1}(\sH) - G^{N,h}_m(\sH)$, then we see that $\Delta \hat{G}^{N,h}_m(\sH)$ is a solution to the Poisson equation for $\Delta G^{N,h}_m(\sH)$, i.e.
\begin{equation}
\Delta \hat{G}^{N,h}_m(\sH) - \int_{\X} \Delta\hat{G}^{N,h}_m(\bar{H}) \, P(\sH, d\bar{H}) = \Delta G^{N,h}_m(\sH) - \int_{\X} \Delta G^{N,h}_m(\bar{\sH}) \, \mu(d\bar{\sH}),\nonumber
\end{equation}
As seen in \eqref{eq:sln_to_Poisson_eqn_bdd_1}-\eqref{eq:sln_to_Poisson_eqn_bdd_2}, the boundedness of $\Delta \hat{G}^{N,h}_m(\sH)$ depends on the Lipschitz constants of $\Delta G^{N,h}_m(\sH)$. The proof is complete if we prove that $\Delta G^{N,h}_m(\sH)$ is $C_T \|h\|_{C^2}/N$-Lipschitz. To begin, we note that
\begin{align*}
\Delta G^{N,h}_m(\sH) &= G^{N,h}_{m+1}(\sH) - G^{N,h}_m(\sH) \\
&= \alpha (\varphi^N_{m+1}(\varsigma_{\sx,\su}) - \varphi^N_m(\varsigma_{\sx,\su}) \K_\lambda(h,\varsigma_{\sx,\su}) \\
&\overset{\eqref{eq:def_of_G_2}}= \frac{\alpha}{N} G^{N, \varsigma_{\sx,\su}}_m(H_m) \K_\lambda(h,\varsigma_{\sx,\su}), \nonumber
\end{align*}
where $H_m = (X_m, Z_m, Y_m, u_m)$. Since we know a priori that $|(X_m, Z_m)| \leq C_x, |Y_m| \leq C_y$ and $|u_m| \leq C_\sigma$, we have the following control (recalling that $h_{m+1}(\cdot) = \varsigma_{X_m,u_m}$):
\begin{align*}
|\Delta G^{N,h}_m(\sH) - \Delta G^{N,h}_m(\tilde{\sH})|
&= \frac{\alpha}{N}[|K_\lambda(h,\varsigma_{\sx,\su})||G^{N,\varsigma_{\sx,\su}}_m(H_m)-G^{N,\varsigma_{\sx,\su}}_m(H_m)| \\
&\phantom{=}+ |G^{N,\varsigma_{\tilde{\sx}, \tilde{\su}}}_m(H_m)| |\K_{\lambda}(h,\varsigma_{\sx,\su} - \varsigma_{\tilde{\sx},\tilde{\su}}))|] \\
&\leq \frac{\alpha}{N}[C \|h\|_{C^2} |\varphi^N_m(h_{m+1}) - Y_m| |\K_\lambda(\varsigma_{\sx,\su} - \varsigma_{\tilde{\sx},\tilde{\su}}, h_{m+1})| \\
&\phantom{=}+ \alpha\|h\|_{C^2} |\varphi^N_m(h_{m+1}) - Y_m| |\K_\lambda(\varsigma_{\sx,\su}, h_{m+1})|] \\
&\leq \frac{C_T}{N} \|h\|_{C^2}.
\end{align*}
As $\Delta G^{N,h}_m(\sH)$ is Lipschitz when $m \leq \floor{NT}$ with constant $\lesssim C_T \|h\|_{C^2} / N$, Lemma \ref{lem:sln_of_Poisson_equation_is_C-Lipschitz} asserts that $\displaystyle{|\Delta \hat{G}^{N,h}_m(\sH)| \leq \frac{C_T}{N} \|h\|_{C^2}}$ whenever $\sH$ lies in the support of $\mu$.
\end{proof}

The above lemmas lead to the control
\begin{equation}
|\fe^{N,2}_k(h)| \leq \frac{C_T}{N} \|h\|_{C^2}.\nonumber
\end{equation}

Finally, let us analyse the term $\fe^{N,1}_k(h)$. For $k\geq 1$, define the term
\begin{equation}
e_m^N(h) = \hat{G}^{N,h}_m(H_m) - \int_{\X} \hat{G}^{N,h}_m(\bar{\sH}) \, P(H_{m-1}, d\bar{\sH}), \quad k \geq 1, \nonumber
\end{equation}
so that
$$\fe^{N,1}_k(h) = \frac{\alpha}{N} \sum_{m=1}^{k-1} e_m^N(h).$$
Observe that $e^N_m(h)$ is $(H_{0:m-1})$-measurable for any $h \in C^2_b(\R^d)$, where $H_{0:m-1} = (H_0,H_1,...,H_{m-1})$.
Therefore, we have
\begin{equation*}
\E[e_k^N(h)|H_{0:k-1}] = \E[\hat{G}^{k,h}(H_k) \,|\, H_{0:k-1} ] - \int_{\X} \hat{G}^{k,h}(\bar{\sH}) \, P(H_{k-1}, d\bar{\sH}) = 0.
\end{equation*}
If we define $H_{k:m}$ being the sequence $(H_k,...,H_m)$, then for $m > k$ we can prove recursively
\begin{align*}
\E[e_k^N(h) e_m^N(h) \,|\, H_{0:k-1}]
&= \E[\E[e_k^N(h) e_m^N(h) \,|\, H_{0:m-1} ] \,|\, H_{0:k-1}] \\
&= \E[e_m^N(h) \E[e_k^N(h) \,|\, H_{0:m-1}] \,|\, H_{0:k-1} ] = 0.\nonumber
\end{align*}
Finally, notice that $\hat{G}^{k,h}(\cdot)$ is uniformly bounded (see Corollary \ref{cor:regularity_of_hat_Gmh}), we have
\begin{align*}
    \E\sqbracket{(e^N_k(h))^2} &\leq \E\sqbracket{\hat{G}^{N,h}_k(H_k) - \int_{\X} \hat{G}^{N,h}_k(\bar{\sH}) \, P(H_{k-1}, d\bar{\sH})}^2 \\
    &\leq 2\E\sqbracket{\hat{G}^{N,h}_k(H_k)}^2 + 2\E\sqbracket{\int_{\X} \hat{G}^{N,h}_k(\bar{\sH}) \, P(H_{k-1}, d\bar{\sH})}^2 \leq C_T \|h\|^2_{C^2}. \nonumber
\end{align*}
Therefore
\begin{equation}
\mathbb{E}[\fe^{N,1}(h)]^2 = \frac{\alpha^2}{N^2} \sum_{j,k=1}^{\floor{Nt}-1} \mathbb{E}[e_j^N(h) e_k^N(h)] = \frac{\alpha^2}{N^2} \sum_{k=1}^{\floor{Nt}-1} \mathbb{E}[(e_k^N(h))^2] \leq \frac{C_T}{N} \|h\|^2_{C^2}.\nonumber
\end{equation}

By combining all of the bounds, we have
\begin{lemma} For $k \leq \floor{NT}$,
\begin{eqnarray}
\E[\varphi^N_k(h) -  \tilde{\varphi}^N_k(h)]^2 \leq \frac{C_T}{N} \|h\|^2_{C^2}.\nonumber
\end{eqnarray}
\end{lemma}

\begin{proof}
We collect the above estimates to show that
\begin{align*}
\mathbb{E}[\varphi^N_k(h) -  \tilde{\varphi}^N_k(h)]^2
&\leq 3\mathbb{E} \sqbracket{\fe^{N,1}_k(h)}^2 + 3\mathbb{E} \sqbracket{\fe^{N,2}_k(h)}^2 + 3\mathbb{E} \sqbracket{R^{N,1}(h)}^2 \\
&\lesssim \frac{C_T}{N} \|h\|^2_{C^2}.\nonumber
\end{align*}
\end{proof}

Let us then define a new process $\phi^N_k(h)$ that satisfies the recursion
\begin{align}
\phi^N_k(h) &= \sum_{m=0}^{k-1} \delta \phi^N_m(h), \quad \phi^N_0(h) = 0, \nonumber\\
    \delta \phi^N_m(h) &= - \frac{\alpha}{N} \int_{\X} (\phi^N_m(\varsigma_{\sx,\su}) - \sy) \K_{\lambda}(h,\varsigma_{\sx,\su}) \, \mu(d\sH),\nonumber
\end{align}

\begin{remark} \label{rmk:well_posedness_phi}
As a sanity check, let us show that $\phi^N_k(h)$ is bounded whenever $h \in C^2_b$ and $k \leq NT$. We shall define the quantity
$$\Phi^N_k = \sup_{h \in C^2_b} \frac{|\phi^N_k(h)|}{\|h\|_{C^2}},$$
and inductively prove that $\Phi^N_k \leq C_T$ whenever $k \leq NT$. Clearly $\Phi^N_0 = 0$ as $\phi^N_0(h) \equiv 0$. Furthermore, for any $k \geq 0$, we have
\begin{align*}
|\phi^N_k(h)| &\leq \frac{\alpha}{N} \sum_{m=0}^{k-1} \int_{\X} |\phi^N_m(\varsigma_{\sx,\su}) - \sy)| |\K_\lambda(h,\varsigma_{\sx,\su})| \, \mu(d\sH) \\
&\leq \frac{\alpha}{N} \sum_{m=0}^{k-1} [\Phi^N_m \|\varsigma_{\sx,\su}\|_{C^2} + C_y] C \|h\|_{C^2} \|\varsigma_{\sx,\su}\|_{C^2} \, \\
&\leq \left[\frac{C}{N} \sum_{m=0}^{k-1} \Phi^N_m + C \right] \|h\|_{C^2}.
\end{align*}
Therefore
\begin{align*}
\Phi^N_k \leq C + \frac{C}{N} \sum_{m=0}^{k-1} \Phi^N_m,
\end{align*}
so by discrete Gr\"onwall inequality, we have
\begin{align*}
\Phi^N_k \lesssim \exp\left(\frac{C k}{N}\right) \lesssim  \exp(CT).
\end{align*}
By definition, for any $h \in C^2_b$,
\begin{equation}
|\phi^N_k(h)| \leq C_T \|h\|_{C^2}.
\end{equation}
\end{remark}

\begin{lemma} \label{lem:D16}
For all $k \leq \floor{NT}$ and $h \in C^2_b$, we have
\begin{equation}
\E|\varphi^N_k(h) - \phi^N_k(h)|^2 \leq \frac{C_T}{N} \|h\|^2_{C^2}.\nonumber
\end{equation}
\end{lemma}

\begin{proof}
Let $\Gamma^N_k(h) = \varphi^N_k(h) - \phi^N_k(h)$. We note that $\Gamma^N_k(\cdot)$ satisfies the following recursion:
\begin{align*}
\Gamma^N_m(h) &= \varphi^N_m(h) - \tilde{\varphi}^N_m(h) + \tilde{\varphi}^N_m(h) - \phi^N_m(h) \\
&= \varphi^N_m(h) - \tilde{\varphi}^N_m(h) - \frac{\alpha}{N} \sum_{k=0}^{m-1} \int_{\X} \Gamma^N_k(\varsigma_{\sx,\su}) \, \K_{\lambda}(h, \varsigma_{\sx,\su}) \, \mu(d\sH). \nonumber
\end{align*}
So, there exists a constant $\sfC > 0$ such that
\begin{align*}
\E[\Gamma^N_k(h)]^2
&\leq 2\E[\varphi^N_k(h) - \tilde{\varphi}^N_k(h)]^2 + \frac{2\alpha^2}{N^2} \E\left[\sum_{m=0}^{k-1} \int_{\X} \Gamma^N_m(\varsigma_{\sx,\su}) \, \K_{\lambda}(h, \varsigma_{\sx,\su}) \, \mu(d\sH) \right]^2 \\
&\leq \frac{C_T}{N}  \|h\|_{C^2}^2 + \frac{2\alpha^2}{N} \sum_{m=0}^{k-1} \E\left[\int_{\X} \big(\Gamma^N_m( \varsigma_{\sx,\su}) \K_{\lambda}(h, \varsigma_{\sx,\su}) \big)^2 \,\mu(d\sH) \right] \\
&\overset{\text{(Tonelli)}}= \frac{C_T}{N} \|h\|_{C^2}^2 + \frac{2\alpha^2}{N} \sum_{m=0}^{k-1} \int_{\X} \E\sqbracket{\Gamma^N_m(\varsigma_{\sx,\su}) \K_{\lambda}(h,\varsigma_{\sx,\su})}^2 \,\mu(d\sH) \\
&\leq \frac{C_T}{N} \|h\|_{C^2}^2 + \frac{2\alpha^2 C}{N} \sum_{m=0}^{k-1} \int_{\X} \E\sqbracket{(\Gamma^N_m(\varsigma_{\sx,\su}))^2 \|h\|_{C^2}^2 \|\varsigma_{\sx,\su}\|_{C^2}}^2 \,\mu(d\sH) \\
&\leq \frac{C_T}{N} \|h\|_{C^2}^2 + \frac{C \|h\|_{C^2}^2}{N} \sum_{m=0}^{k-1} \int_{\X} \E[\Gamma^N_m(\varsigma_{\sx,\su})]^2 \,\mu(d\sH). \numberthis \label{eq:lemma_6_16_pre_Gronwall}
\end{align*}
Defining further $\displaystyle{\tilde{\Gamma}^N_k = \int_{\X} \E[\Gamma^N_k(\varsigma_{\tilde{\sx},\tilde{\su}})]^2} \, \mu(d\tilde{\sH})$, where $\tilde{H} = (\tilde{\sx}, \tilde{\sz}, \tilde{\sy}, \tilde{\su})$, then
\begin{equation*}
\tilde{\Gamma}^N_m \leq \frac{C_T}{N} + \frac{C_T}{N} \sum_{m=0}^{k-1} \tilde{\Gamma}^N_k.
\end{equation*}
Therefore, by discrete Gronwall inequality, we have for all $k \leq \lfloor NT \rfloor$,
\begin{equation*}
\tilde{\Gamma}^N_k \lesssim \frac{C_T}{N} \exp\left(C \frac{k}{N} \right) \lesssim \frac{C_T}{N}.
\end{equation*}
Substituting this into \eqref{eq:lemma_6_16_pre_Gronwall} yields the desired result:
\begin{align*}
\E[\Gamma^N_k(h)]^2
&\lesssim \frac{C_T}{N} \|h\|_{C^2}^2.
\end{align*}
\end{proof}

Finally, we denote $\tilde{\phi}^N_t(h) := \phi^N_{\floor{Nt}}(h)$ as the time-rescaled evolution of $\phi^N_m(h)$, and recall that
\begin{equation}
g_t(h) = - \alpha \int_0^t \sqbracket{\int_{\X} (g_s( \varsigma_{\sx,\su}) - \sy)) \K_{\lambda}(h, \varsigma_{\sx,\su}) \, \mu(d\sH)} \, ds, \quad g_0(h) = 0.\nonumber
\end{equation}

\begin{lemma} \label{lem:D17}
For $h \in C^2_b$,
\begin{equation}
\sup_{t\in[0,T]} |\tilde{\phi}^N_t(h) - g_t(h)| \leq \frac{C_T}{N} \|h\|_{C^2}.\nonumber
\end{equation}
\end{lemma}

\begin{proof}
Let
\begin{equation*}
\Upsilon^N_t(h) = \phi^N_t(h) - g_t(h), \quad \tilde{\Upsilon}^N_t = \|\Upsilon^N_t\|_{(C^2)^*} = \sup_{h \in C^2_b} \frac{|\Upsilon^N_t(h)|}{\|h\|_{C^2}}.
\end{equation*}
We note that $\tilde{\Upsilon}^N_t$ is well defined for any $t \leq T$. In fact,
\begin{equation*}
\tilde{\Upsilon}^N_t \leq \Phi^N_{\lfloor Nt \rfloor} + \|g_t\|_{(C^2)^*} \leq C_T,
\end{equation*}
guaranteed by Remark \ref{rmk:well_posedness_phi} and Proposition \ref{prop:well_posedness_ode}.

Then, for $k \leq \floor{NT}$,
\begin{align*}
\Upsilon^N_{(k+1)/N}(h)
&= \Upsilon^N_{k/N}(h) + \tilde{\phi}^N_{(k+1)/N}(h) - \tilde{\phi}^N_{k/N}(h) - g_{(k+1)/N}(h) + g_{k/N}(h) \\
&= \Upsilon^N_{k/N}(h) - \frac{\alpha}{N} \int_{\X} (\phi^N_k(\varsigma_{\sx,\su}) - \sy) \K_{\lambda}(h,\varsigma_{\sx,\su}) \, \mu(d\sH) \\
&\phantom{=}+ \alpha \int_{k/N}^{(k+1)/N} \int_{\X} (g_s(\varsigma_{\sx,\su})-\sy) \K_{\lambda}(h,\varsigma_{\sx,\su})) \,  \mu(d\sH) \, ds \\
&= \Upsilon^N_{k/N}(h) - \frac{\alpha}{N} \int_{\X} (\phi^N_k(\varsigma_{\sx,\su}) - \sy) \K_{\lambda}(h,\varsigma_{\sx,\su}) \,  \mu(d\sH) \\
&\phantom{=}+ \alpha \int_{\X} \left[ \left(\int_{k/N}^{(k+1)/N} g_s(\varsigma_{\sx,\su}) \, ds \right) -\sy \right] \K_{\lambda}(h,\varsigma_{\sx,\su}) \,  \mu(d\sH) \\
&= \Upsilon^N_{k/N}(h) - \frac{\alpha}{N} \int_{\X} \bigg[\phi^N_k(\varsigma_{\sx,\su}) - N\int_{k/N}^{(k+1)/N} g_s(\varsigma_{\sx,\su}) \, ds \bigg] \K_{\lambda}(h,\varsigma_{\sx,\su}) \,  \mu(d\sH) \\
&= \Upsilon^N_{k/N}(h) - \underbrace{\frac{\alpha}{N} \int_{\X} \Upsilon^N_{k/N}(\varsigma_{\sx,\su}) \K_{\lambda}(h,\varsigma_{\sx,\su}) \,  \mu(d\sH)}_{:= \mathsf{(I)}} \\
&\phantom{=}- \underbrace{\frac{\alpha}{N} \int_{\X} \bigg[g_{k/N}(\varsigma_{\sx,\su}) - N\int_{k/N}^{(k+1)/N} g_s(\varsigma_{\sx,\su}) \, ds \bigg] \K_{\lambda}(h,\varsigma_{\sx,\su}) \, \mu(d\sH)}_{:= \mathsf{(II)}}
\end{align*}
We shall now bound $\mathsf{(I)}$:
\begin{align*}
|\mathsf{(I)}| &\leq \frac{\alpha}{N} \int_{\X} |\Upsilon^N_{k/N}(\varsigma_{\sx,\su})| |\K_{\lambda}(h,\varsigma_{\sx,\su})| \, \mu(d\sH) \\
&\leq \frac{C}{N} \int_{\X} \tilde{\Upsilon}^N_{k/N} \|h\|_{C^2} \|\varsigma_{\sx,\su}\|_{C^2}^2 \, \mu(d\sH) \\
&\leq \frac{C}{N} \tilde{\Upsilon}^N_{k/N} \|h\|_{C^2}.
\end{align*}
As for $\mathsf{(II)}$, the well-posedness of $g_t$ by Proposition \ref{prop:well_posedness_ode} implies that for all $s \in [k/N, (k+1)/N)$ and $h \in C^2_b$,
\begin{equation*}
|g_s(h) - g_{k/N}(h)| \leq \int_{m/N}^s |[\cA(g_\tau)](h) + b(h)| \, d\tau \lesssim (s-k/N) C_T \|h\|_{C^2} \leq \frac{C_T}{N} \|h\|_{C^2}.
\end{equation*}
Therefore,
\begin{align*}
\left| g_{k/N}(h) - N \int_{k/N}^{(k+1)/N} g_s(h) \, ds \right|
&= N\left| \int_{k/N}^{(k+1)/N} (g_{k/N}(h) - g_s(h)) \, ds \right| \\
&\leq N \int_{k/N}^{(k+1)/N} |g_s(h) - g_{k/N}(h)| \, ds \\
&\lesssim \frac{C_T}{N} \|h\|_{C^2}.
\end{align*}
This enables us to bound $\mathsf{(II)}$:
\begin{align*}
|\mathsf{(II)}| &\lesssim \frac{\alpha}{N} \int_{\X} \left[\frac{C_T}{N} \|\varsigma_{\sx,\su}\|_{C^2} \right] \|h\|_{C^2} \|\varsigma_{\sx,\su}\|_{C^2} \, \mu(d\sH) \\
&\leq \frac{C_T}{N^2} \|h\|_{C^2}.
\end{align*}
Therefore
\begin{align*}
|\Upsilon^N_{(k+1)/N}(h)|
&\leq \left[\left(1+ \frac{C}{N} \right) \tilde{\Upsilon}^N_{k/N} + \frac{C_T}{N^2} \right] \|h\|_{C^2}
\end{align*}
and that
\begin{equation*}
\tilde{\Upsilon}^N_{(k+1)/N}(h)
\leq \left(1+ \frac{C}{N} \right) \tilde{\Upsilon}^N_{k/N} + \frac{C_T}{N^2}.
\end{equation*}
By Lemma \eqref{lem:recursive_inequality}, for all $k \leq \floor{NT}$ we have
\begin{equation}
\tilde{\Upsilon}^N_{k/N}
\lesssim \left(1 + \frac{C}{N} \right)^{NT} \frac{C_T}{N} \leq \frac{C_T}{N}.\nonumber
\end{equation}
Finally, for all $t \leq T$, define $m = \floor{t/N}$, then
\begin{align*}
|\tilde{\phi}^N_t(h) - g_t(h)| &\leq  \underbrace{|\tilde{\phi}^N_t(h) - \tilde{\phi}^N_{m/N}(h)|}_{=0} + |\tilde{\Upsilon}^N_{m/N}(h)| + |g_{m/N}(h) - g_t(h)| \\
&\leq \frac{C_T}{N}\|h\|_{C^2}. \nonumber
\end{align*}
\end{proof}

\subsection{The final steps} \label{SS:ProofMainTheorem}
Now, we are in position to prove the main convergence result of this paper.

\begin{proof} (Proof of Theorem \ref{thm:weak_convergence_thm})
As a summary, we collect
\begin{itemize}
\item Lemma \ref{lem:D8}:
\begin{equation*}
\sup_{k \leq NT} \E\abs{g^N_k(h) - \varphi^N_k(h)} \leq \frac{C_T}{N^{(1-\beta-2\gamma) \wedge \gamma \wedge (\beta-1/2)}} \|h\|_{C^2}.
\end{equation*}
\item Lemma \ref{lem:D16}:
\begin{align*}
\sup_{k \leq NT} \E[\varphi^N_k(h) - \phi^N_k(h)]^2 &\leq \frac{C_T}{N} \|h\|^2_{C^2} \\
\implies \sup_{h\in \cH} \E|\varphi^N_m(h) - \phi^N_m(h)| &\leq \frac{C_T}{N^{1/2}} \|h\|_{C^2}.
\end{align*}
\item Lemma \ref{lem:D17}:
\begin{equation*}
\sup_{t\in[0,T]} |\phi^N_t(h) - g_t(h)| \leq \frac{C_T}{N}  \|h\|_{C^2}.
\end{equation*}
\end{itemize}
Adding all of the error terms yields:
\begin{align*}
&\phantom{=}\sup_{t \in [0,T]} \E|g^N_t(h) - g_t(h)| \\
&\leq \sup_{t \in [0,T]} [\E|g^N_t(h) - \varphi^N_t(h)| + \E|\varphi^N_t(h) - \phi^N_t(h)|] + \sup_{t \in [0,T]} \E|\phi^N_t(h) - g_t(h)| \\
&\leq \sup_{k \leq NT} [\E|g^N_k(h) - \varphi^N_k(h)| + \E|\varphi^N_k(h) - \phi^N_k(h)|] + \sup_{t \in [0,T]} \E|\phi^N_t(h) - g_t(h)| \\
&\leq \frac{C_T}{N^\epsilon} \|h\|_{C^2} \nonumber
\end{align*}
where $\epsilon = (1-\beta-2\gamma) \wedge \gamma \wedge (\beta-1/2) > 0$, noting that $\beta - 1/2 < 1/2$.
\end{proof}


\begin{appendix}
\section{Recursive Inequality}
\label{S:recursive_inequality}

Many proofs of the technical lemmas involves the study of a sequence $(a_k)_{k\geq 0}$ that satisfies the following recursive inequality:
\begin{equation}
    a_k \leq M_1 a_{k-1} + M_2,\nonumber
\end{equation}
where $M_1, M_2 \geq 0$. By recursion, we can prove that
\begin{lemma} \label{lem:recursive_inequality}
For $M_1 < 1$, we have
\begin{equation}
    a_k \leq M_1^k a_0 + \frac{1-M_1^k}{1-M_1}M_2 \leq M_1^k a_0 + \frac{1}{1-M_1}M_2,\nonumber
\end{equation}
and for $M_1 > 1$, we have
\begin{equation}
    a_k \leq M_1^k a_0 + \frac{M_1^k-1}{M_1-1}M_2 \leq M_1^k \bracket{a_0 + \frac{M_2}{M_1 - 1}}.\nonumber
\end{equation}
\end{lemma}

\section{Construction of a clipping function}
\label{S:construction_of_clipping_function}

We consider the function
\begin{equation}
    f(x) = \begin{cases}
        \exp(-1/x) & x > 0 \\
        0 & x \leq 0,
    \end{cases}\nonumber
\end{equation}
which is known to be infinitely smooth (i.e. in $C^\infty(\R)$). Therefore, the function
\begin{equation}
    g(x) = \frac{f(x)}{f(x) + f(1-x)}\nonumber
\end{equation}
is also infinitely smooth. In particular we have
\begin{equation}
    g(x) = \begin{cases}
        = 0 & x < 0 \\
        \in [0,1] & x \in [0,1] \\
        = 1 & x > 1
    \end{cases}.\nonumber
\end{equation}
Therefore for any $a < b$, the infinitely smooth function
\begin{equation}
    g_{a,b}(x) = g\bracket{\frac{x-a}{b-a}} = \begin{cases}
        = 0 & x < a \\
        \in [0,1] & x \in [a,b] \\
        = 1 & x > b
    \end{cases},\nonumber
\end{equation}
and that the infinitely smooth function
\begin{equation}
    \rho_N(x) = g_{-2N^\gamma,-N^\gamma}(-x) g_{-2N^\gamma,-N^\gamma}(x) = \begin{cases}
        = 1 & |x| \leq N^\gamma \\
        \in [0,1] & N^\gamma < |x| \leq 2N^\gamma \\
        = 0 & |x| > 2N^\gamma
    \end{cases}.\nonumber
\end{equation}
\begin{figure}[h!]
    \centering
    \includegraphics[width=\textwidth]{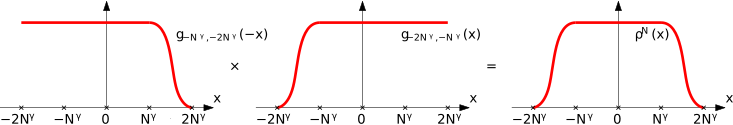}
    \caption{$\rho_N(x) = g_{-2N^\gamma,-N^\gamma}(-x) \times g_{-2N^\gamma,-N^\gamma}(x)$}
    \label{fig:bump_function_construction}
\end{figure}
Finally, the function
\begin{equation}
    \psi_N(x) = \int_0^x \rho_N(y) \, dy\nonumber
\end{equation}
satisfies all the requirement for being a smooth clipping function in Definition \ref{def:smooth_clipping_function}: (2) follows by direct computation, and (3) is true by definition (that $\frac{d}{dx} \psi_N(x) = \rho_N(x)$). Finally, (1) follows by Fundamental theorem of calculus. By symmetry we could prove only for the case when $x > 0$, for which
\begin{equation}
    |\psi_N(x)| \leq \int_0^{2N^\gamma} |\rho_N(y)| \, dy \leq 2N^\gamma.\nonumber
\end{equation}

\section{A-priori controls over the increments of parameters}
\label{S:increments_of_parameters}

Readers are reminded that $x \lesssim y$ means there exists $C>0$ such that $x \leq Cy$. In particular, the constant $C$ does not depend on $N$ and $T$.

\subsection{Proof of Lemma \ref{lem:evolution_of_parameters}}
\label{SS:proof_of_evolution_of_parameters_with_clipping}

Assumption \ref{as:activation_function} on the activation functions $\sigma$ implies for all $k\in\mathbb{N}$
\begin{align}
    \abs{\hat{S}^{i,N}_{k+1}} &= \sigma\bracket{(W^i_k)^\top X_k + \frac{1}{N} \sum_{j=1}^N B^j_k \hat{S}^{j,N}_k} \leq C_\sigma, \label{eq:application_activation_1} \\
    \abs{\Delta\hat{S}^{i,N}_{k+1}} &= \sigma'\bracket{(W^i_k)^\top X_k + \frac{1}{N} \sum_{j=1}^N B^j_k \hat{S}^{j,N}_k} \leq  C_{\sigma'}, \label{eq:application_activation_2}
\end{align}
where $(\hat{S}^{i,N}_k)_{k\geq 0}$ and $(\Delta \hat{S}^{i,N}_k)_{k\geq 1}$ are as specified in Algorithm \ref{alg:onlineSGDwithBPTT}. Therefore, we have
\begin{equation}
    \abs{C^i_{k+1} - C^i_k} \leq \frac{\alpha}{N^{2-\beta}} \abs{\hat{S}^{i,N}_{k+1}} \bracket{|\psi^N(\hat{Y}^N_k)| + \abs{Y_k}} \leq \frac{\alpha C_\sigma}{N^{2-\beta}}(2N^\gamma + C_y) \lesssim \frac{1}{N^{2-\beta-\gamma}}, \label{Eq:IncrementDiff}
\end{equation}
so by a telescopic sum argument we have for all $k\leq \floor{NT}$
\begin{equation*}
    \abs{C^i_k - C^i_0} \leq \sum_{j=0}^{k-1} \abs{C^i_{j+1}-C^i_j} \lesssim \frac{NT}{N^{2-\beta-\gamma}} = \frac{T}{N^{1-\beta-\gamma}},
\end{equation*}
and
\begin{equation*}
    |C^i_k| \leq |C^i_0| + |C^i_k - C^i_0| \lesssim 1+T.
\end{equation*}
We may further conclude that
\begin{align*}
    \|W^i_{k+1} - W^i_k\|
    &\leq \frac{\alpha}{N^{2-\beta}} |C^i_k| \big[|\psi^N(\hat{Y}^N_k)| + |Y_k| \big] |\Delta \hat{S}^{i,N}_{k+1}| \|X_k\| \\
    &\leq \frac{C_T}{N^{2-\beta}} \\
    &\leq\frac{C_T}{N^{2-\beta-\gamma}}, \nonumber \\
    |B^i_{k+1} - B^i_k| &\leq \frac{\alpha}{N^{3-\beta}} \bracket{|\psi^N(\hat{Y}^N_k)| + |Y_k|} \sum_{\ell=1}^N |C^\ell_k| |\hat{S}^{\ell,N}_k| |\Delta \hat{S}^{\ell,N}_{k+1}| \\
    &\leq \frac{C_T}{N^{2-\beta}} (2N^\gamma + C_y)\\
    &\leq \frac{C_T}{N^{2-\beta-\gamma}}. \nonumber
\end{align*}
Therefore
\begin{equation*}
    |C^i_{k+1} - C^i_k| \vee \|W^i_{k+1} - W^i_0 \| \vee |B^i_{k+1} - B^i_k| \lesssim \frac{1+T}{N^{2-\beta-\gamma}},
\end{equation*}
and by following the telescoping sum argument as above we conclude, as desired,
\begin{equation*}
    |C^i_k - C^i_0| + \norm{W^i_{k} - W^i_0} + |B^i_k - B^i_0| \lesssim \frac{(1+T)^2}{N^{1-\beta-\gamma}}. \nonumber
\end{equation*}

\subsection{Proof of Remark \ref{rmk:less_sharp_evolution_of_parameters}}
\label{SS:proof_of_evolution_of_parameters}

The proof of remark relies on the observation that
\begin{equation} \label{eq:diff_abs_C}
\abs{C^i_{k+1}} - \abs{C^i_k} \leq \abs{C^i_{k+1} - C^i_k} \leq \frac{\alpha}{N^{2-\beta}} |\hat{S}^{i,N}_{k+1}| \bracket{|\hat{Y}^N_k| + \abs{Y_k}} \leq \frac{\alpha C_\sigma C_y}{N^{2-\beta}} + \frac{\alpha C^2_\sigma}{N^2} \sum_{i=1}^N |C^i_k|,
\end{equation}
so by letting $\bar{C}_k := \frac{\sum_{i=1}^N |C^i_k|}{N}$, one has
\begin{equation}
    \bar{C}_{k+1} - \bar{C}_k \leq \frac{\alpha C_\sigma C_y}{N^{2-\beta}} + \frac{\alpha C^2_\sigma}{N} \bar{C}_k \implies \bar{C}_{k+1} \leq \bracket{1 + \frac{\alpha C^2_\sigma}{N}} \bar{C}_k + \frac{\alpha C_\sigma C_y}{N^{2-\beta}},\nonumber
\end{equation}
and by Lemma \ref{lem:recursive_inequality}, for all $k \leq \floor{NT}$, one has
\begin{equation}
    \bar{C}_k \leq \bracket{1 + \frac{\alpha C^2_\sigma}{N}}^{\floor{TN}} \bracket{1 + \frac{\alpha C_\sigma C_y / N^{2-\beta}}{\alpha C^2_\sigma /N}} \lesssim \exp(\alpha C^2_\sigma T).\nonumber
\end{equation}
Substituting this back to Equation \eqref{eq:diff_abs_C} yields
\begin{equation}
    \abs{C^i_{k+1} - C^i_k} \leq \frac{\alpha C_\sigma C_y}{N^{2-\beta}} + \frac{\alpha C^2_\sigma \bar{C}_k}{N} \lesssim \frac{C_{T}}{N},\nonumber
\end{equation}
noting that $\beta < 1$ and $C_{T}<\infty$ is constant that depends on $\alpha,C_{\sigma},C_{y}$ and on $T$. Therefore, by a telescoping sum argument we have for all $k \leq \lfloor NT \rfloor$,
$$|C^i_k| = |C^i_0|+\sum_{m=0}^{k-1}|C^i_{m+1} - C^i_{m}| \lesssim 1+NTe^{\alpha C_\sigma^2 T} / N = C_{T},$$
for a different constant $C_{T}<\infty$. This enables us to show further that
\begin{align*}
\norm{W^i_{k+1} - W^i_k} &\leq \frac{\alpha}{N^{2-\beta}} |C^i_k| \sqbracket{|\hat{Y}^N_k| + |Y_k|} |\Delta \hat{S}^{i,N}_{k+1}| \\
&\leq \frac{\alpha C_\sigma |C^i_k|}{N^{2-\beta}} \sqbracket{|Y_k| + \frac{1}{N^\beta} \sum_{\ell=1}^N |C^\ell_k|} \\
&\leq \frac{\alpha C_\sigma C|C^i_k|}{N^{2-\beta}} + \frac{\alpha C_\sigma}{N} \sum_{\ell=1}^N |C^i_k| |C^\ell_k|\\
&\leq \frac{C_{T}}{N}, \nonumber \\
\abs{B^i_{k+1} - B^i_k}
&\leq \frac{\alpha}{N^{3-\beta}} \bracket{|\hat{Y}^N_k| + |Y_k|} \sum_{\ell=1}^N |C^\ell_k| |\Delta \hat{S}^{\ell,N}_{k+1}| \\
&\leq \frac{C_{T}}{N^{2-\beta}} \left(C_y + \frac{1}{N^\beta} \sum_{\ell=1}^N |C^\ell_k||\hat{S}^{\ell,N}_{k+1}| \right) \\
&\leq \frac{C_{T}}{N}.\nonumber
\end{align*}

Overall, for all $k \leq \lfloor NT \rfloor$
\begin{equation*}
|C^i_{k+1} - C^i_k| + \|W^i_{k+1} - W^i_k \| + |B^i_{k+1} - B^i_k| \lesssim \frac{C_T}{N},
\end{equation*}
and therefore
\begin{equation*}
|C^i_k - C^i_0| + \|W^i_k - W^i_0 \| + |B^i_k - B^i_0| \leq C_T.
\end{equation*}

Note that the constant $C_{T}$ depends on $T,\sigma,\alpha$ and may change from line to line.

\section{Details on numerical simulations}
\label{S:numerics}
To illustrate the above results, we simulated 100 paths of the untrained hidden memory units, denoted as $h^{N,\mathsf{path}}_k(W^i) = S^{i,N}_k(X^{\mathsf{path}};\theta)$, each based on the common parameters $\theta = (C^i,W^i,B^i)$ and an independent instance of the input sequence $X^{\mathsf{path}}$ for $\mathsf{path} = 1, 2, ..., 100$. The actual input sequence is simulated from the following 2D random dynamical systems:
\begin{align*}
\begin{bmatrix} X_{k+1} \\ Z_{k+1} \end{bmatrix} &= \frac{1}{2} \tanh\bracket{P \begin{bmatrix} 1 & 0 \\ 0 & 1/2 \end{bmatrix} P^{-1} \begin{bmatrix} X_{k} \\ Z_{k} \end{bmatrix}}, \\
P &= \begin{bmatrix} \sqrt{3}/2 & 1/2 \\ -1/2 & \sqrt{3}/2 \end{bmatrix}, \; X_0, Z_0 \overset{\text{iid}}\sim \mathsf{Uniform}[0,1].
\end{align*}

We choose the standard sigmoid function as our activation function, and $(B^i, W^i)$ are simulated so that $B^i \overset{\text{iid}}\sim \mathsf{Uniform}[0,1]$ and $W^i \overset{\text{iid}}\sim \mathsf{Normal}[0,1]$.\footnote{The setting does not fully satisfy our strong assumptions, but it gives a good illustration of the general mean-field behaviour.}

The mean-field behaviour could be studied by looking at the empirical distributions of the untrained hidden units for each set (represented by the gray lines in figure \ref{fig:2}):
\begin{equation*}
\nu^{N,\mathsf{path}}_k = \frac{1}{N} \sum_{i=1}^N \delta_{h^{N,\mathsf{path}}_k(W^i)},
\end{equation*}
as well as the overall distributions of the hidden units from all sets as a Monte-carlo estimate of the expectation of the distribution of $\varsigma_{X_k,u^N_k}(W^i)$ (represented by the
red lines in Figure \ref{fig:2}):
\begin{equation*}
\nu^{N,\mathsf{overall}}_k = \frac{1}{100} \sum_{\mathsf{path}=1}^{100} \delta_{h^{N,\mathsf{path}}_k(W^i)}.
\end{equation*}
The convergence of distribution of $\varsigma_{X_k,u^N_k}(W^i)$ as $N\to\infty$ is illustrated by noting that the above plots are similar to each other for large $N$. Figure \ref{fig:3} in the introduction is plotted by stacking all overall empirical distributions $\nu^{N,\mathsf{overall}}_k$ on the same plot for easier comparison.
\begin{figure}[t]
    \centering
    \includegraphics[width=.9\textwidth]{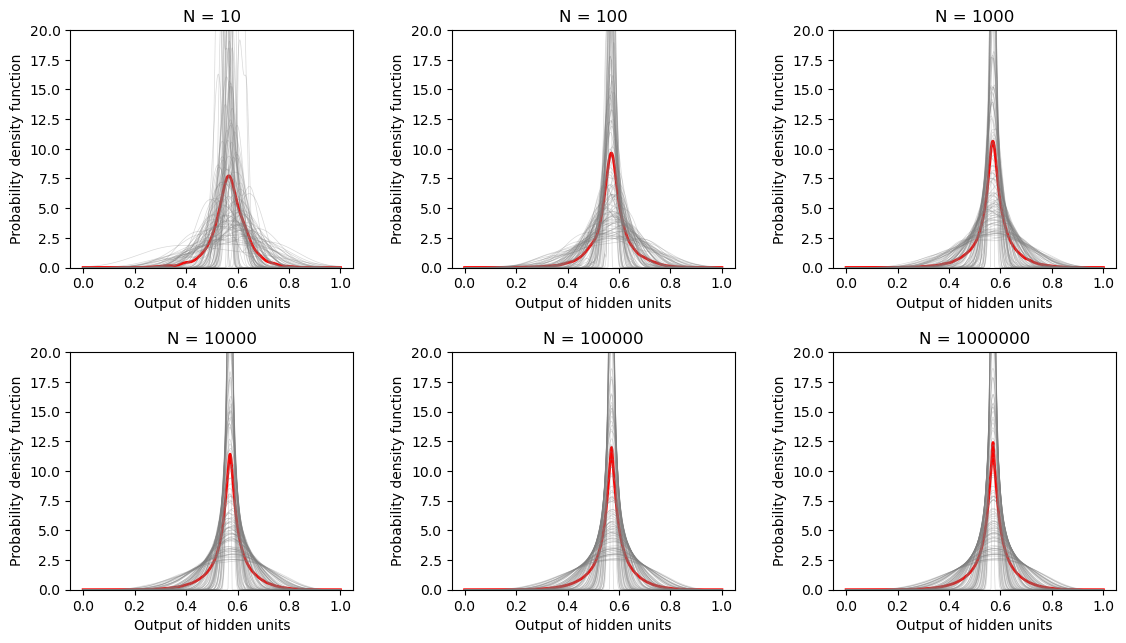}
    \caption{Empirical distributions of the untrained hidden memory units $\varsigma_{X_k,u^N_k}(W^i)$ for varying $N$ and large time step $k \approx 50000$. The grey lines represents the empirical distribution for a \textit{single} set of the untrained hidden memory units $\nu^{N,\mathsf{path}}_k$, and the red line represents the empirical distribution of all untrained hidden memory units from \textit{all} sets $\nu^{N,\mathsf{overall}}_k$.}
    \label{fig:2}
\end{figure}

The intermediate proposition provides a solid foundation for our numerical experiment. Showing the ergodicity of $\varsigma_{X_k,u^N_k}(W^i) = S^{i,N}(X;\theta)$ is easier than showing the ergodicity of the function $h_k(\cdot)$ as the former does not involve the exact computation of $\la b' h_k(w'), \lambda \ra$. The above proposition implies that for sufficiently large $N$, the untrained memory units $\varsigma_{X_k,u^N_k}(W^i) = S^{i,N}(X;\theta)$ exhibit similar ergodic behaviour due to the ergodicity of $h_k$. To illustrate this, we compute the time average of the empirical first and second moments of the sample hidden memory units $(S^{i,N}_k(X;\theta))_{k\geq 0}$ for each sets, defined as
\begin{equation*}
\mathsf{timeAvg}^{N,p,\mathsf{path}}_{T} = \frac{1}{NT} \sum_{k=0}^T \sum_{i=1}^N \sqbracket{h^{N,\mathsf{path}}_k(W^i)}^p, \quad p = 1,2, \quad \mathsf{path}=1,...,100,
\end{equation*}
as well as the overall time average of the empirical first and second moments:
\begin{equation*}
\mathsf{timeAvg}^{N,p,\mathsf{overall}}_{T} = \frac{1}{100}  \sum_{\mathsf{path}=1}^{100} \mathsf{timeAvg}^{N,p,\mathsf{path}}_{T}.
\end{equation*}
They are both plotted in Figure \ref{fig:1}, with the red line being $\mathsf{timeAvg}^{N,p,\mathsf{overall}}_{T}$, and all $\mathsf{timeAvg}^{N,p,\mathsf{path}}_{T}$ lies within the gray line.
band. The shrinking of the gray band towards the red line is another evidence of mean-field behaviour, and the convergence of $\mathsf{timeAvg}^{N,p,\mathsf{overall}}_{T}$ demonstrates the ergodicity of the untrained hidden memory units. Figure \ref{fig:1_truncated} in the introduction is the enlarged version of \ref{fig:1} for $N = 10^6$.

\begin{figure}
\centering
\includegraphics[width=.6\textwidth]{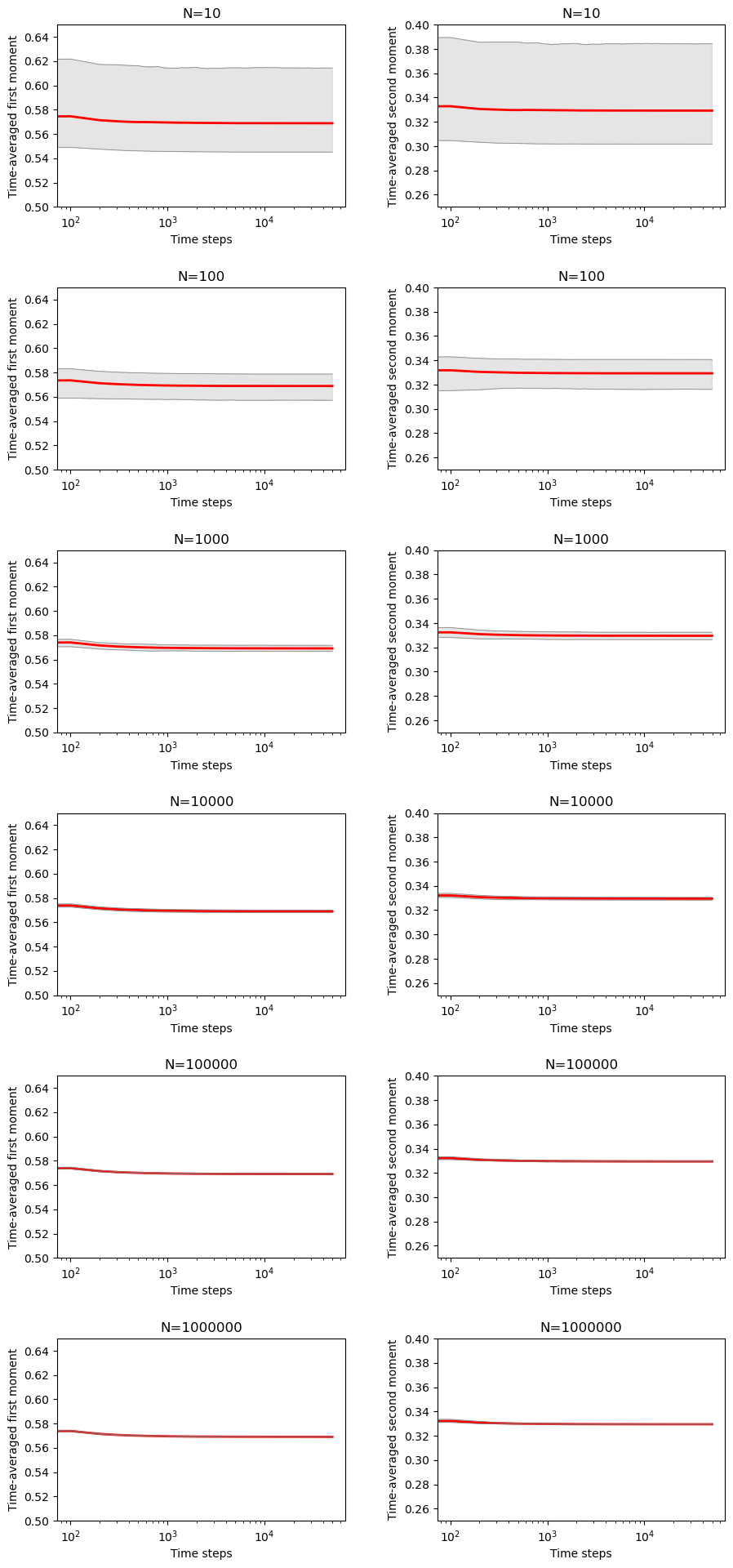}
\caption{The plot of time averages for $N = 10^k, k = 2,3,4,5,6$ and $p = 1,2$. The actual realizations of $\mathsf{timeAvg}^{N,p,\mathsf{path}}_T$ lie in the grey band, and the red line is the overall time average $\mathsf{timeAvg}^{N,p,\mathsf{overall}}_T$ as the Monte-Carlo estimate of $\E[\mathsf{timeAvg}^{N,p}_T]$. The desired converging behaviour is only exhibited when $N$ is sufficiently large.}
\label{fig:1}
\end{figure}

Please refer to \url{https://github.com/Samuel-CHLam/rnn_ergodicity} for the code of the simulation.
\end{appendix}

\section*{Acknowledgment}
The authors would like to acknowledge the use of the University of Oxford Advanced Research Computing (ARC) facility for completing the numerical simulations. We would also like to thank the anonymous reviewers  for their constructive criticism that greatly improved the article.

\bibliographystyle{vancouver}
\bibliography{export}
\end{document}